\documentclass[10pt,journal]{IEEEtran}
\usepackage{amsfonts}
\usepackage[nocompress]{cite}
\usepackage{cite}
\usepackage{eurosym}
\usepackage{setspace}
\usepackage{tikz}
\usepackage{tkz-graph}
\usepackage{cite}
\usepackage[hidelinks]{hyperref}
\usepackage{caption}
\usepackage{subcaption}
\usepackage{algorithm}
\usepackage{algpseudocode}
\usepackage{array}
\usepackage{mdwmath}
\usepackage{mdwtab}
\usepackage{mathtools}
\usepackage{amsmath}
\usepackage{amsthm}
\usepackage{amssymb}
\usepackage{array}
\usepackage{epstopdf}
\usepackage[utf8]{inputenc}
\usepackage[english]{babel}

\ifCLASSOPTIONcompsoc
\else
\fi
\ifCLASSINFOpdf
\else
\fi
\usetikzlibrary{shapes,arrows,fit,calc,positioning,automata}

\newcommand{\X}{\mathbf{X}}
\newcommand{\Y}{\mathbf{Y}}
\newcommand{\I}{\mathbf{I}}
\newcommand{\x}{\mathbf{x}}
\newcommand{\mub}{\mathbf{\mu}}
\newcommand{\Lambdab}{\mathbf{\Lambda}}

\newcommand{\M}{\mathbf{M}}
\newcommand{\m}{\mathbf{m}}

\theoremstyle{plain}
\newtheorem{theorem}{Theorem}
\newtheorem{prop}{Proposition}
\newtheorem{corollary}{Corollary}
\newtheorem{lemma}{Lemma}
\theoremstyle{definition}
\newtheorem{definition}{Definition}

\begin{document}

\title{Optimal Bayesian Transfer Learning}

\author{Alireza Karbalayghareh, 
\IEEEmembership{Student Member, IEEE,} Xiaoning Qian, \IEEEmembership{Senior
Member, IEEE,} and~Edward~R.~Dougherty, \IEEEmembership{Fellow, IEEE}\thanks{
The authors are with the Department of Electrical and Computer Engineering,
Texas A\&M University, College Station, TX, USA.}}
\maketitle

\IEEEtitleabstractindextext{
\begin{abstract}
Transfer learning has recently attracted significant research attention, as it simultaneously learns from different source domains, which have plenty of labeled data, and transfers the relevant knowledge to the target domain with limited labeled data to improve the prediction performance. We propose a Bayesian transfer learning framework, in the homogeneous transfer learning scenario, where the source and target domains are related through the joint prior density of the model parameters. The modeling of joint prior densities enables better understanding of the ``transferability" between domains. We define a joint Wishart distribution for the precision matrices of the Gaussian feature-label distributions in the source and target domains to act like a bridge that transfers the useful information of the source domain to help classification in the target domain by improving the target posteriors. Using several theorems in multivariate statistics, the posteriors and posterior predictive densities are derived in closed forms with hypergeometric functions of matrix argument, leading to our novel closed-form and fast Optimal Bayesian Transfer Learning (OBTL) classifier.  Experimental results on both synthetic and real-world benchmark data confirm the superb performance of the OBTL compared to the other state-of-the-art transfer learning and domain adaptation methods.

\end{abstract}

}

\IEEEdisplaynontitleabstractindextext
\IEEEpeerreviewmaketitle

\section{Introduction}
\label{sec1}
A basic assumption of traditional machine learning is that data in the
training and test sets are independently sampled in one domain with the
identical underlying distribution. However, with the growing amount of
heterogeneity in modern data, the assumption of having only one domain may
not be reasonable. Transfer learning (TL) is a learning strategy that
enables us to learn from a source domain with plenty of labeled data as well
as a target domain with no or very few labeled data in order to design a
better classifier in the target domain than the ones trained by target-only
data for its generalization performance. This can reduce the effort of
collecting labeled data for the target domain, which might be very costly,
if not impossible. Due to its importance, there has been ongoing research on
the topic of transfer learning and many surveys in the recent years covering
transfer learning and domain adaptation methods from different perspectives 
\cite{survey2010,survey2017deep, survey2015, survey2016, survey2017}.

If we train a model in one domain and directly apply it in another, the
trained model may not generalize well, but if the domains are related,
appropriate transfer learning and domain adaptation methods can borrow
information from all the data across the domains to develop better
generalizable models in the target domain. Transfer learning in medical
genomics is desirable, since the number of labeled data samples is often
very limited due to the difficulty of having disease samples and the
prohibitive costs of human clinical trials. However, it is relatively easier
to obtain gene-expression data for cell lines or other model species like
mice or dogs. If these different life systems share the same underlying
disease cellular mechanisms, we may utilize data in cell lines or model
species as our source domain to develop transfer learning methods for more
accurate human disease prognosis in the target domain \cite%
{zou2015transfer,ganchev2011transfer}.

\subsection{Related Works}

Domain adaptation (DA) is a specific case of transfer learning where the
source and target domains have the same classes or categories \cite%
{survey2017deep, survey2015, survey2017}. DA methods either adapt the model
learned in the source domain to be applied in the target domain or adapt the
source data so that the distribution can be close to the one of the target
data. Depending on the availability of labeled target data, the DA methods
are categorized as unsupervised and semi-supervised algorithms. Unsupervised
DA problems applies to the cases where there are no labeled target data and
the algorithm uses only unlabeled data in the target domain along with
source labeled data \cite{gong2012geodesic}. Semi-supervised DA methods use
both the unlabeled and a few labeled target data to learn a classifier in
the target domain with the help of source labeled data \cite{HFA2012,
hoffman2013,hoffman2014, CDLS2016}.

Depending on whether the source and target domains have the same feature
space with the same feature dimension, there are homogeneous and
heterogeneous DA methods. The first direction in homogeneous DA is instance
re-weighting, for which the most popular measure to re-weight the data is
Maximum Mean Discrepancy (MMD) \cite{MMD} between the two domains. Transfer
Adaptive Boosting (TrAdaBoost) \cite{dai2007boosting} is another method that
adaptively sets the weights for the source and target samples during each
iteration based on the relevance of source and target data to help train the
target classifier. Another direction is model or parameter adaptation. There
are several efforts to adapt the SVM classifier designed in the source
domain for the target domain, for example, based on residual error \cite%
{duan2009domain,bruzzone2010domain}. Feature augmentation methods, such as
Geodesic Flow Sampling (GFS) and Geodesic Flow Kernel (GFK) \cite%
{gong2012geodesic}, derive intermediate subspaces using Geodesic flows,
which interpolate between the source and target domains. Finding an
invariant latent domain in which the distance between the empirical
distributions of the source and target data is minimized is another
direction to tackle the problem of domain adaptation, such as Invariant
Latent Space (ILS) in \cite{ILS2017}. Authors in \cite{ILS2017} proposed to
learn an invariant latent Hilbert space to address both the unsupervised and
semi-supervised DA problems, where a notion of domain variance is
simultaneously minimized while maximizing a measure of discriminatory power
using Riemannian optimization techniques. Max-Margin Domain Transform (MMDT) 
\cite{hoffman2013} is a semi-supervised feature transformation DA method
which uses a cost function based on the misclassification loss and jointly
optimizes both the transformation and classifier parameters. Another
domain-invariant representation method \cite{OT} matches the distributions
in the source and target domains via a regularized optimal transportation
model. Heterogeneous Feature Augmentation (HFA) \cite{HFA2012} is a
heterogeneous DA method which typically embeds the source and target data
into a common latent space prior to data augmentation.

Domain adaption has been recently studied in deep learning frameworks like
deep adaptation network (DAN) \cite{long2015learning}, residual transfer
networks (RTN) \cite{long2016unsupervised}, and models based on generative
adversarial networks (GAN) such as domain adversarial neural network (DaNN) 
\cite{ganin2016domain} and coupled GAN (CoGAN) \cite{liu2016coupled}.
Although deep DA methods have shown promising results, they require a fairly
large amount of labeled data.

\subsection{Main Contributions}

This paper treats homogeneous transfer learning and domain adaptation from
Bayesian perspectives, a key aim being better theoretical understanding when
data in the source domain are \textquotedblleft transferrable" to help
learning in the target domain. When learning complex systems with limited
data, Bayesian learning can integrate prior knowledge to compensate for the
generalization performance loss due to the lack of data. Rooted in
Optimal Bayesian Classifiers (OBC) \cite{Lori1,Lori2}, which gives the classifiers having
Bayesian minimum mean squared error (MMSE) over uncertainty classes of
feature-label distributions, we propose a Bayesian transfer learning
framework and the corresponding Optimal Bayesian Transfer Learning (OBTL)
classifier to formulate the OBC in the target domain by taking advantage of
both the available data and the joint prior knowledge in source and target
domains. In this Bayesian learning framework, transfer learning from the
source to target domain is through a joint prior probability density
function for the model parameters of the feature-label distributions of the
two domains. By explicitly modeling the dependency of the model parameters
of the feature-label distribution, the posterior of the target model
parameters can be updated via the joint prior probability distribution
function in conjunction with the source and target data. Based on that, we
derive the \textit{effective} class-conditional densities of the target
domain, by which the OBTL classifier is constructed.

Our problem definition is the same as the aforementioned domain adaptation
methods, where there are plenty of labeled source data and few labeled
target data. The source and target data follow different multivariate
Gaussian distributions with arbitrary mean vectors and precision (inverse of
covariance) matrices. For the OBTL, we define a joint Gaussian-Wishart prior
distribution, where the two precision matrices in the two domains are
jointly connected. This joint prior distribution for the two precision
matrices of the two domains acts like a bridge through which the useful
knowledge of the source domain can be transferred to the target domain,
making the posterior of the target parameters tighter with less uncertainty.

With such a Bayesian transfer learning framework and several theorems from
multivariate statistics, we define an appropriate joint prior for the
precision matrices using hypergeometric functions of matrix argument, whose
marginal distributions are Wishart as well. The corresponding closed-form
posterior distributions for the target model parameters are derived by
integrating out all the source model parameters. Having closed-form
posteriors facilitates closed-form effective class-conditional densities.
Hence, the OBTL classifier can be derived based on the corresponding
hypergeometric functions and does not need iterative and costly techniques
like MCMC sampling. Although the OBTL classifier has a closed form,
computing these hypergeometric functions involves the computation of series
of zonal polynomials, which is time-consuming and not scalable to high
dimension. To resolve this issue, we use the Laplace approximations of these
functions, which preserves the good prediction performance of the OBTL while
making it efficient and scalable. The performance of the OBTL is tested on
both synthetic data and real-world benchmark image datasets to show its
superior performance over state-of-the-art domain adaption methods.

The paper is organized as follows. Section \ref{sec2} introduces the
Bayesian transfer learning framework. Section \ref{sec3} derives the
closed-form posteriors of target parameters, via which Section \ref{sec4}
obtains the effective class-conditional densities in the target domain.
Section \ref{sec5} derives the OBTL classifier, and Section \ref{sec6}
presents the OBC in the target domain and shows that the OBTL classifier
converts to the target-only OBC when there is no interaction between the
domains. Section \ref{sec7} presents
experimental results using both synthetic and real-world benchmark data.
Section \ref{sec8} concludes the paper. Appendix \ref{appendix:hypergeometric} states some useful 
theorems for the generalized hypergeometric functions of matrix argument. Appendices \ref{appendix:posterior} and \ref{appendix:effective} provide the proofs of our main  theorems. Finally, Appendix \ref{appendix:Laplace} presents the Laplace approximation of Gauss hypergeometric functions of matrix argument.

\section{Bayesian Transfer Learning Framework}
\label{sec2}

We consider a supervised transfer learning problem in which there are $L$
common classes (labels) in each domain. Let $\mathcal{D}_{s}$ and $\mathcal{D%
}_{t}$ denote the labeled datasets of the source and target domains with the
sizes of $N_{s}$ and $N_{t}$, respectively, where $N_{t}\ll N_{s}$. Let $%
\mathcal{D}_{s}^{l}=\left\{ \mathbf{x}_{s,1}^{l},\mathbf{x}_{s,2}^{l},\cdots
,\mathbf{x}_{s,n_{s}^{l}}^{l}\right\} $, $l\in \{1,\cdots ,L\}$, where  $%
n_{s}^{l}$ denotes the size of data in the source domain for the label $l$. Similarly, let $%
\mathcal{D}_{t}^{l}=\left\{ \mathbf{x}_{t,1}^{l},\mathbf{x}_{t,2}^{l},\cdots
,\mathbf{x}_{t,n_{t}^{l}}^{l}\right\} $, $l\in \{1,\cdots ,L\}$, where  $%
n_{t}^{l}$ denotes the size of data in the target domain for the label $l$. There is no intersection
between $\mathcal{D}_t^i$ and $\mathcal{D}_t^j$ and also between $\mathcal{D}_s^i$ and $\mathcal{D}_s^j$ for 
any $i,j\in \{1,\cdots,L\}$. Obviously, we have $%
\mathcal{D}_{s}=\cup _{l=1}^{L}\mathcal{D}_{s}^{l}$, $\mathcal{D}_{t}=\cup
_{l=1}^{L}\mathcal{D}_{t}^{l}$, $N_{s}=\sum_{l=1}^{L}n_{s}^{l}$, and $%
N_{t}=\sum_{l=1}^{L}n_{t}^{l}$. Since we consider the homogeneous transfer
learning scenario, where the feature spaces are the same in both the source
and target domains, $\mathbf{x}_{s}^{l}$ and $\mathbf{x}_{t}^{l}$ are $%
d\times 1$ vectors for $d$ features of the source and target domains,
respectively.

Letting $\mathbf{x}^{l}=\left[ {\mathbf{x}_{t}^{l^{\prime }}},{\mathbf{x}%
_{s}^{l^{\prime }}}\right] ^{\prime }$ be a $2d\times 1$ augmented feature
vector, $\mathbf{A}^{^{\prime }}$ denoting the transpose of matrix $\mathbf{A%
}$, a general joint sampling model would take the Gaussian form 
\begin{equation}
\mathbf{x}^{l}\sim \mathcal{N}\left( \mathbf{\mu }^{l},\left( \mathbf{%
\Lambda }^{l}\right) ^{-1}\right) ,~~~l\in \{1,\cdots ,L\},  
\label{x}
\end{equation}%
with 
\begin{equation}
\mathbf{\mu }^{l}=%
\begin{bmatrix}
\mathbf{\mu }_{t}^{l} \\ 
\mathbf{\mu }_{s}^{l}%
\end{bmatrix}%
,~~~~\mathbf{\Lambda }^{l}=%
\begin{bmatrix}
\mathbf{\Lambda }_{t}^{l} & \mathbf{\Lambda }_{ts}^{l} \\ 
{\mathbf{\Lambda }_{ts}^{l}}^{^{\prime }} & \mathbf{\Lambda }_{s}^{l}%
\end{bmatrix},
  \label{mu_lambda}
\end{equation}
where $\mathbf{\mu }^{l}$ is the $2d\times 1$ mean vector, and $\mathbf{%
\Lambda }^{l}$ is the $2d\times 2d$ precision matrix. In this model, $%
\mathbf{\Lambda }_{t}^{l}$ and $\mathbf{\Lambda }_{s}^{l}$ account for the
interactions of features within the source and target domains, respectively,
and $\mathbf{\Lambda }_{ts}^{l}$ accounts for the interactions of the
features across the source and target domains, for any class $l\in
\{1,\cdots ,L\}$. In this Gaussian setting, it is common to use a
Wishart distribution as a prior for the precision matrix $\mathbf{\Lambda }^{l}$, since it is a conjugate prior.

In transfer learning, it is not realistic to assume joint sampling of the source and target domains. Therefore we cannot use the general joint sampling model.
Instead, we assume that there are two datasets separately sampled from the
source and target domains. Thus, we define a joint prior
distribution for $\mathbf{\Lambda }_{s}^{l}$ and $\mathbf{\Lambda }_{t}^{l}$
by marginalizing out the term $\mathbf{\Lambda }_{ts}^{l}$. This joint prior
distribution of the parameters of the source and target domains accounts for
the dependency (or ``relatedness") between the domains.

Given this adjustment to account for transfer learning, we utilize a Gaussian
model for the feature-label distribution in each domain: 
\begin{equation}
\mathbf{x}_{z}^{l}\sim \mathcal{N}\left( \mathbf{\mu }_{z}^{l},{\left( 
\mathbf{\Lambda }_{z}^{l}\right) }^{-1}\right) ,~~~l\in \{1,\cdots ,L\},
\label{x_s_x_t}
\end{equation}%
where subscript $z\in \{s,t\}$ denotes the source $s$ or target $t$ domain, 
$\mathbf{\mu }_{s}^{l}$ and $\mathbf{\mu }_{t}^{l}$ are $d\times 1$ mean
vectors in the source and target domains for label $l$, respectively, $%
\mathbf{\Lambda }_{s}^{l}$ and $\mathbf{\Lambda }_{t}^{l}$ are the $d\times
d $ precision matrices in the source and target domains for label $l$,
respectively, and a joint Gaussian-Wishart distribution is employed as a prior for
mean and precision matrices of the Gaussian models. Under these assumptions,
the joint prior distribution for $\mathbf{\mu }_{s}^{l}$, $\mathbf{\mu }%
_{t}^{l}$, $\mathbf{\Lambda }_{s}^{l}$, and $\mathbf{\Lambda }_{s}^{l}$
takes the form
\begin{equation}
\label{general_joint_prior}
p\left( \mathbf{\mu }_{s}^{l},\mathbf{\mu }_{t}^{l},\mathbf{\Lambda }%
_{s}^{l},\mathbf{\Lambda }_{t}^{l}\right) =p\left( \mathbf{\mu }_{s}^{l},%
\mathbf{\mu }_{t}^{l}|\mathbf{\Lambda }_{s}^{l},\mathbf{\Lambda }%
_{t}^{l}\right) p\left( \mathbf{\Lambda }_{s}^{l},\mathbf{\Lambda }%
_{t}^{l}\right) .
\end{equation}%
To facilitate conjugate priors, we assume that, for any class $l\in
\{1,\cdots ,L\}$, $\mathbf{\mu }_{s}^{l}$ and $\mathbf{\mu }_{t}^{l}$ are
conditionally independent given $\mathbf{\Lambda }_{s}^{l}$ and $\mathbf{%
\Lambda }_{t}^{l}$, so that 
\begin{equation}
p\left( \mathbf{\mu }_{s}^{l},\mathbf{\mu }_{t}^{l},\mathbf{\Lambda }%
_{s}^{l},\mathbf{\Lambda }_{t}^{l}\right) =p\left( \mathbf{\mu }_{s}^{l}|%
\mathbf{\Lambda }_{s}^{l}\right) p\left( \mathbf{\mu }_{t}^{l}|\mathbf{%
\Lambda }_{t}^{l}\right) p\left( \mathbf{\Lambda }_{s}^{l},\mathbf{\Lambda }%
_{t}^{l}\right) ,  \label{p_mu}
\end{equation}%
and that both $p\left( \mathbf{\mu }_{s}^{l}|\mathbf{\Lambda }%
_{s}^{l}\right) $ and $p\left( \mathbf{\mu }_{t}^{l}|\mathbf{\Lambda }%
_{t}^{l}\right) $ are Gaussian, 
\begin{equation}
\mathbf{\mu }_{z}^{l}|\mathbf{\Lambda }_{z}^{l}\sim \mathcal{N}\left( 
\mathbf{m}_{z}^{l},\left( \kappa _{z}^{l}\mathbf{\Lambda }_{z}^{l}\right)
^{-1}\right) ,  \label{mu_s_mu_t}
\end{equation}%
where $\mathbf{m}_{z}^{l}$ is the $d\times 1$ mean vector of $\mathbf{\mu }%
_{z}^{l}$, and $\kappa _{z}^{l}$ is a positive scalar hyperparameter. We
need to define a joint distribution for $\mathbf{\Lambda }_{s}^{l}$ and $%
\mathbf{\Lambda }_{t}^{l}$. In the case of a prior for either $\mathbf{%
\Lambda }_{s}^{l}$ or $\mathbf{\Lambda }_{t}^{l}$, we use a Wishart
distribution as the conjugate prior. Here we desire a joint distribution for 
$\mathbf{\Lambda }_{s}^{l}$ and $\mathbf{\Lambda }_{t}^{l}$, whose marginal
distributions for both $\mathbf{\Lambda }_{s}^{l}$ and $\mathbf{\Lambda }%
_{t}^{l}$ are Wishart.

We present some definitions and theorems that will be used in
deriving the OBTL classifier.

\begin{definition}
\label{definition1} A random $d\times d$ symmetric positive-definite matrix $%
\mathbf{\Lambda }$ has a nonsingular Wishart distribution with $\nu $
degrees of freedom, $W_{d}(\mathbf{M},\nu )$, if $\nu \geq d$ and $\mathbf{M}
$ is a $d\times d$ positive-definite matrix ($\mathbf{M}>0$) and the density
is 
\begin{equation}
p(\mathbf{\Lambda })=\left[ 2^{\frac{\nu d}{2}}\Gamma _{d}\left( \frac{\nu }{%
2}\right) |\mathbf{M}|^{\frac{\nu }{2}}\right] ^{-1}|\mathbf{\Lambda }|^{%
\frac{\nu -d-1}{2}}\mathrm{etr}\left( -\frac{1}{2}\mathbf{M}^{-1}\mathbf{%
\Lambda }\right) ,  \label{wishart}
\end{equation}%
where $|\mathbf{A}|$ is the determinant of $\mathbf{A}$, $\mathrm{etr}(%
\mathbf{A})=\exp \left( \mathrm{tr}(\mathbf{A})\right) $ and $\Gamma
_{d}(\alpha )$ is the multivariate gamma function given by 
\begin{equation}
\Gamma _{d}(\alpha )=\pi ^{\frac{d(d-1)}{4}}\prod_{i=1}^{d}\Gamma \left(
\alpha -\frac{i-1}{2}\right) .  \label{Gamma_multi}
\end{equation}
\end{definition}

\begin{prop}
\label{proposition1}  \cite{muirhead}: If $\mathbf{\Lambda} \sim W_d(\mathbf{M}%
,\nu)$, and $\mathbf{A}$ is an $r\times d$ matrix of rank $r$, where $r \le d
$, then $\mathbf{A} \mathbf{\Lambda} \mathbf{A}^{^{\prime }} \sim W_r(%
\mathbf{A}\mathbf{M}\mathbf{A}^{^{\prime }},\nu)$.
\end{prop}

\begin{corollary}
\label{corollary1} If $\mathbf{\Lambda} \sim W_d(\mathbf{M},\nu) $ and $%
\mathbf{\Lambda} = 
\begin{psmallmatrix} \Lambdab_{11} & \Lambdab_{12}
\\ \Lambdab_{12}^{'} & \Lambdab_{22} \end{psmallmatrix} $, where $\mathbf{%
\Lambda}_{11}$ and $\mathbf{\Lambda}_{22}$ are $d_1\times d_1$ and $d_2
\times d_2$ submatrices, respectively, and if $\mathbf{M} = 
\begin{psmallmatrix} \M_{11} & \M_{12} \\ \M_{12}^{'} & \M_{22}
\end{psmallmatrix} $ is the corresponding partition of $\mathbf{M}$ with $%
\mathbf{M}_{11}$ and $\mathbf{M}_{22}$ being two $d_1 \times d_1$ and $d_2
\times d_2$ submatrices, respectively, then $\mathbf{\Lambda}_{11} \sim
W_{d_1}(\mathbf{M}_{11},\nu)$ and $\mathbf{\Lambda}_{22} \sim W_{d_2}(%
\mathbf{M}_{22},\nu)$.
\end{corollary}

Using Corollary \ref{corollary1}, we can ensure that using the Wishart
distribution for the precision matrix $\mathbf{\Lambda}^l$ (\ref{mu_lambda})
of the joint model in (\ref{x}) will lead to the Wishart marginal
distributions for $\mathbf{\Lambda}_s^l$ and $\mathbf{\Lambda}_t^l$ in the
source and target domains separately, which is a desired property. Now we
introduce a theorem, proposed in \cite{joint_wishart}, which gives the form
of the joint distribution of the two submatrices of a partitioned Wishart
matrix.

\begin{theorem}
\label{theorem1}  \cite{joint_wishart}: Let $\mathbf{\Lambda} = 
\begin{psmallmatrix} \Lambdab_{11} & \Lambdab_{12} \\ \Lambdab_{12}^{'} &
\Lambdab_{22} \end{psmallmatrix} $ be a $(d_1+d_2) \times (d_1+d_2)$
partitioned Wishart random matrix, where the diagonal partitions are of
sizes $d_1 \times d_1$ and $d_2 \times d_2$, respectively. The Wishart
distribution of $\mathbf{\Lambda}$ has $\nu \ge d_1 +d_2$ degrees of freedom
and positive-definite scale matrix $\mathbf{M}=%
\begin{psmallmatrix}
\mathbf{M}_{11} & \mathbf{M}_{12} \\ \mathbf{M}_{12}^{'} &
\mathbf{M}_{22}\end{psmallmatrix} $ partitioned in the same way as $\mathbf{%
\Lambda}$. The joint distribution of the two diagonal partitions $\mathbf{%
\Lambda}_{11}$ and $\mathbf{\Lambda}_{22}$ have the density function given
by 
\begin{equation}
\begin{aligned} & p(\mathbf{\Lambda}_{11},\mathbf{\Lambda}_{22}) = \\ & K ~
\mathrm{etr}\left(-\frac{1}{2} \left(\mathbf{M}_{11}^{-1} +
\mathbf{F}^{'}\mathbf{C}_2 \mathbf{F}\right)\mathbf{\Lambda}_{11}\right)
\mathrm{etr}\left(-\frac{1}{2} \mathbf{C}_2^{-1}
\mathbf{\Lambda}_{22}\right) \\ & \times
|\mathbf{\Lambda}_{11}|^{\frac{\nu-d_2-1}{2}} ~
|\mathbf{\Lambda}_{22}|^{\frac{\nu-d_1-1}{2}} ~ ~_0F_1\left(\frac{\nu}{2};
\frac{1}{4}\mathbf{G} \right), \end{aligned}
\end{equation}
where $\mathbf{C}_2=\mathbf{M}_{22} - \mathbf{M}_{12}^{^{\prime }}\mathbf{M}%
_{11}^{-1}\mathbf{M}_{12}$, $\mathbf{F}=\mathbf{C}_2^{-1} \mathbf{M}%
_{12}^{^{\prime }} \mathbf{M}_{11}^{-1}$, $\mathbf{G}=\mathbf{\Lambda}_{22}^{%
\frac{1}{2}} \mathbf{F} \mathbf{\Lambda}_{11} \mathbf{F}^{^{\prime }}\mathbf{%
\Lambda}_{22}^{\frac{1}{2}}$, $K^{-1} = 2^{\frac{(d_1+d_2)\nu}{2}}
\Gamma_{d_1}\left(\frac{\nu}{2}\right) \Gamma_{d_2}\left(\frac{\nu}{2}%
\right) |\mathbf{M}|^{\frac{\nu}{2}}$, and$~_0F_1$ is the generalized
matrix-variate hypergeometric function.
\end{theorem}

\begin{definition}
\label{definition2}  \cite{nagar2017properties}: The generalized
hypergeometric function of one matrix argument is defined by 
\begin{eqnarray}  \label{hypergeo}
~_pF_q(a_1,\cdots,a_p;b_1,\cdots,b_q;\mathbf{X}) \hspace{2cm}  \notag \\
= \sum_{k=0}^{\infty} \sum_{\kappa \vdash k} \frac{(a_1)_\kappa \cdots
(a_p)_{\kappa}}{(b_1)_{\kappa} \cdots (b_q)_\kappa} \frac{C_\kappa(\mathbf{X}%
)}{k!},
\end{eqnarray}
where $a_i$, $i=1,\cdots,p$, and $b_j$, $j=1,\cdots,q$, are arbitrary
complex (real in our case) numbers, $C_\kappa(\mathbf{X})$ is the zonal
polynomial of $d\times d$ symmetric matrix $\mathbf{X}$ corresponding to the
ordered partition $\kappa=(k_1,\cdots,k_d)$, $k_1 \ge \cdots \ge k_d \ge 0$, 
$k_1+\cdots k_d=k$ and $\sum_{\kappa\vdash k}$ denotes summation over all
partitions $\kappa$ of $k$. The generalized hypergeometric coefficient $%
(a)_\kappa$ is defined by 
\begin{equation}
(a)_\kappa = \prod_{i=1}^d \left(a - \frac{i-1}{2}\right)_{k_i},
\end{equation}
where $(a)_r=a(a+1)\cdots (a+r-1)$, $r=1,2,\cdots$, with $(a)_0=1$.
\end{definition}

Conditions for convergence of the series in (\ref{hypergeo}) are available
in the literature \cite{constantine1963}. From (\ref{hypergeo}) it follows 
\vspace{-.2cm}
\begin{equation}
\begin{aligned} &~_0F_0(\mathbf{X}) = \sum_{k=0}^{\infty} \sum_{\kappa
\vdash k} \frac{C_\kappa(\mathbf{X})}{k!} = \sum_{k=0}^{\infty}
\frac{(\mathrm{tr}(\mathbf{X}))^k}{k!} = \mathrm{etr}(\mathbf{X}), \\
&~_1F_0(a;\mathbf{X}) = \sum_{k=0}^{\infty} \sum_{\kappa \vdash k}
\frac{(a)_\kappa C_\kappa(\mathbf{X})}{k!} = |\mathbf{I}_m -
\mathbf{X}|^{-a}, ~~ ||\mathbf{X}|| <1, \\ &~_0F_1(b;\mathbf{X}) =
\sum_{k=0}^{\infty} \sum_{\kappa \vdash k} \frac{
C_\kappa(\mathbf{X})}{(b)_\kappa k!}, \\ &~_1F_1(a;b;\mathbf{X}) =
\sum_{k=0}^{\infty} \sum_{\kappa \vdash k}
\frac{(a)_\kappa}{(b)_\kappa}\frac{ C_\kappa(\mathbf{X})}{k!}, \\
&~_2F_1(a,b;c;\mathbf{X}) = \sum_{k=0}^{\infty} \sum_{\kappa \vdash k}
\frac{(a)_\kappa (b)_\kappa}{(c)_\kappa}\frac{ C_\kappa(\mathbf{X})}{k!},
~~~||\mathbf{X}|| <1, \end{aligned}  \label{Gauss}
\end{equation}
where $||\mathbf{X}|| <1$ means that the maximum of the absolute values of
the eigenvalues of $\mathbf{X}$ is less than $1$. $_1F_1(a;b;\mathbf{X})$
and $_2F_1(a,b;c;\mathbf{X}) $ are respectively called Confluent and Gauss
hypergeometric functions of matrix argument. See Appendix \ref{appendix:hypergeometric} for some useful theorems on
zonal polynomials and generalized hypergeometric functions of matrix arguments. We use those 
theorems to derive the posterior densities and posterior predictive densities of the target parameters in closed forms 
in terms of Confluent and Gauss hypergeometric functions of matrix argument in Sections \ref{sec3} and \ref{sec4}, respectively.

Now, using Theorem \ref{theorem1}, we define the joint prior distribution, $%
p(\mathbf{\Lambda }_{s}^{l},\mathbf{\Lambda }_{t}^{l})$ in (\ref{p_mu}), of
the precision matrices of the source and target domains for class $l\in
\{1,\cdots ,L\}$ as follows: 
\vspace{-.2cm}
\begin{equation}
\begin{aligned} &p(\mathbf{\Lambda}_{t}^l,\mathbf{\Lambda}_{s}^l) = K^l
\mathrm{etr}\left(-\frac{1}{2} \left({\left(\mathbf{M}_{t}^l\right)}^{-1} +
{\mathbf{F}^l}^{'}\mathbf{C}^l
\mathbf{F}^l\right)\mathbf{\Lambda}_{t}^l\right) \\ & \times
\mathrm{etr}\left(-\frac{1}{2} {\left(\mathbf{C}^l\right)}^{-1}
\mathbf{\Lambda}_{s}^l\right) \\ & \times
\left|\mathbf{\Lambda}_{t}^l\right|^{\frac{\nu^l-d-1}{2}}
\left|\mathbf{\Lambda}_{s}^l\right|^{\frac{\nu^l-d-1}{2}}
~_0F_1\left(\frac{\nu^l}{2}; \frac{1}{4}\mathbf{G}^l \right), \end{aligned}
\label{joint}
\end{equation}%
where $\mathbf{M}=%
\begin{psmallmatrix}
\M_{t}^{l} & \M_{ts}^{l} \\ {\M_{ts}^{l}}^{'} & \M_{s}^{l}
\end{psmallmatrix} $ 
is a $2d\times 2d$ positive definite scale matrix, $\nu ^{l}\geq 2d$ denotes
degrees of freedom, and
\vspace{-.2cm}
\begin{equation}
\begin{aligned}
\mathbf{C}^{l} &=&\mathbf{M}_{s}^{l}-{\mathbf{M}_{ts}^{l}}^{^{\prime }}{%
\left( \mathbf{M}_{t}^{l}\right) }^{-1}\mathbf{M}_{ts}^{l}, \\
\mathbf{F}^{l} &=&{\left( \mathbf{C}^{l}\right) }^{-1}{\mathbf{M}_{ts}^{l}}%
^{^{\prime }}{\left( \mathbf{M}_{t}^{l}\right) }^{-1}, \\
\mathbf{G}^{l} &=&{\mathbf{\Lambda }_{s}^{l}}^{\frac{1}{2}}\mathbf{F}^{l}%
\mathbf{\Lambda }_{t}^{l}{\mathbf{F}^{l}}^{^{\prime }}{\mathbf{\Lambda }%
_{s}^{l}}^{\frac{1}{2}}, \\
{(K^{l})}^{-1} &=&2^{d\nu ^{l}}\Gamma _{d}^{2}\left( \frac{\nu ^{l}}{2}%
\right) |\mathbf{M}^{l}|^{\frac{\nu ^{l}}{2}}.
\end{aligned}
\end{equation}
Using Corollary \ref{corollary1}, $\mathbf{\Lambda }_{t}^{l}$ and $\mathbf{%
\Lambda }_{s}^{l}$ have the following Wishart marginal distributions: 
\begin{equation}
\mathbf{\Lambda }_{z}^{l}\sim W_{d}(\mathbf{M}_{z}^{l},\nu ^{l}),~~~l\in
\{1,\cdots ,L\},~~~z\in \{s,t\}.  \label{marg_t}
\end{equation}

\vspace{-.3cm}
\section{Posteriors of Target Parameters}
\label{sec3}

\begin{figure*}[h!]
\centering
\begin{tikzpicture}[->,>=stealth',shorten >=.5pt,auto,node distance=2cm,
                thick,main node/.style={circle,draw}]

  \node[main node, cloud, draw,cloud puffs=10,cloud puff arc=120, aspect=1.5, inner ysep=.3em] (1) {$\mathcal{D}_t^l$};
  \node[main node] (2) [right of=1] {$\mub_t^l$};
  \node[main node] (3) [right of=2] {$\Lambdab_t^l$};
  \node[main node] (4) [right of=3] {$\Lambdab_s^l$};
  \node[main node] (5) [right of=4] {$\mub_s^l$};
  \node[main node, cloud, draw,cloud puffs=10,cloud puff arc=120, aspect=1.5, inner ysep=.3em] (6) [right of=5] {$\mathcal{D}_s^l$};
  \node[main node,rectangle,minimum height=1.8cm,minimum width=5.7cm,rounded corners=.3cm,dashed] (7) [right of=1, label=below:Target Domain] {};
  \node[main node,rectangle,minimum height=1.8cm,minimum width=5.7cm,rounded corners=.3cm,dashed] (8) [right of=4, label=below:Source Domain] {};

  \path
    (3) edge  node {} (2)
        edge [bend right] node {} (1)
    (2) edge  node {} (1)
    (4) edge  node {} (5)
        edge [bend left] node {} (6)
    (5) edge  node {} (6);
  \path[-]
  (3) edge  node {} (4);
\end{tikzpicture}
\vspace{-.1cm}
\caption{{\protect\footnotesize Dependency of the source and target domains
through their precision matrices for any class $l\in \{1,\cdots,L\}$.}}
\label{fig1}
\vspace{-.2cm}
\end{figure*}
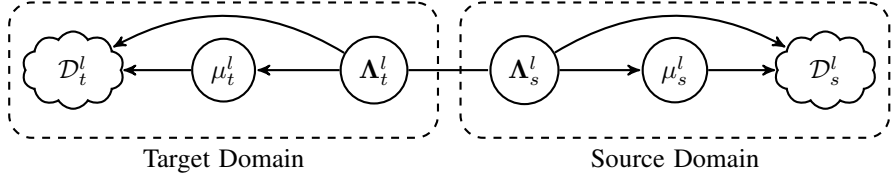

Having defined the prior distributions in the previous section, we aim to
derive the posterior distribution of the parameters of the target domain
upon observing the training source $\mathcal{D}_s$ and target $\mathcal{D}_t$
datasets. The likelihood of the datasets $\mathcal{D}_t$ and $\mathcal{D}_s$
is conditionally independent given the parameters of the target and source
domains. The dependence between the two domains is due to the dependence of
the prior distributions of the precision matrices, as shown in Fig \ref{fig1}%
. Within each domain, source or target, the likelihoods of the different
classes are also conditionally independent given the parameters of the
classes. As such, the joint likelihood of the datasets $\mathcal{D}_t$ and $%
\mathcal{D}_s$ can be written as \vspace{-.2cm} 
\begin{equation}  \label{likelihood}
\begin{aligned} p(\mathcal{D}_t,&\mathcal{D}_s
|\mathbf{\mu}_t,\mathbf{\mu}_s,\mathbf{\Lambda}_t,\mathbf{\Lambda}_s) =
p(\mathcal{D}_t|\mathbf{\mu}_t,\mathbf{\Lambda}_t)p(\mathcal{D}_s|\mathbf{%
\mu}_s,\mathbf{\Lambda}_s) \\
&=p(\mathcal{D}_t^1,\cdots,\mathcal{D}_t^L|\mathbf{\mu}_t^1,\cdots,\mathbf{%
\mu}_t^L,\mathbf{\Lambda}_t^1,\cdots,\mathbf{\Lambda}_t^L) \\ &~~~~\times
p(\mathcal{D}_s^1,\cdots,\mathcal{D}_s^L|\mathbf{\mu}_s^1,\cdots,\mathbf{%
\mu}_s^L,\mathbf{\Lambda}_s^1,\cdots,\mathbf{\Lambda}_s^L) \\
&=\prod_{l=1}^L p(\mathcal{D}_t^l|\mathbf{\mu}_t^l,\mathbf{\Lambda}_t^l)
\prod_{l=1}^L p(\mathcal{D}_s^l|\mathbf{\mu}_s^l,\mathbf{\Lambda}_s^l).
\end{aligned}
\end{equation}
The posterior of the parameters given $\mathcal{D}_t$ and $\mathcal{D}_s$
satisfies \vspace{-.2cm} 
\begin{equation}
\begin{aligned}
&p(\mathbf{\mu}_t,\mathbf{\mu}_s,\mathbf{\Lambda}_t,\mathbf{\Lambda}_s|%
\mathcal{D}_t,\mathcal{D}_s) \\ &\propto
p(\mathcal{D}_t,\mathcal{D}_s|\mathbf{\mu}_t,\mathbf{\mu}_s,\mathbf{%
\Lambda}_t,\mathbf{\Lambda}_s)
p(\mathbf{\mu}_t,\mathbf{\mu}_s,\mathbf{\Lambda}_t,\mathbf{\Lambda}_s) \\
&\propto \prod_{l=1}^L
p(\mathcal{D}_t^l|\mathbf{\mu}_t^l,\mathbf{\Lambda}_t^l) \prod_{l=1}^L
p(\mathcal{D}_s^l|\mathbf{\mu}_s^l,\mathbf{\Lambda}_s^l) \prod_{l=1}^L
p(\mathbf{\mu}_t^l,\mathbf{\mu}_s^l,\mathbf{\Lambda}_t^l,\mathbf{%
\Lambda}_s^l), \end{aligned}  \label{posterior_1}
\end{equation}
where we assume that the priors of the parameters in different classes are
independent, $p(\mathbf{\mu}_t,\mathbf{\mu}_s,\mathbf{\Lambda}_t,\mathbf{%
\Lambda}_s) = \prod_{l=1}^L p(\mathbf{\mu}_t^l,\mathbf{\mu}_s^l,\mathbf{%
\Lambda}_t^l,\mathbf{\Lambda}_s^l)$. From (\ref{p_mu}) and (\ref{posterior_1}%
), \vspace{-.2cm} 
\begin{equation}
\begin{aligned}
p(\mu_t,\mu_s,\mathbf{\Lambda}_t,\mathbf{\Lambda}_s|\mathcal{D}_t,%
\mathcal{D}_s) \propto \prod_{l=1}^L
p(\mathcal{D}_t^l|\mathbf{\mu}_t^l,\mathbf{\Lambda}_t^l) p(\mathcal{D}_s^l
|\mathbf{\mu}_s^l,\mathbf{\Lambda}_s^l) \\ \hspace{1cm}\times
p\left(\mathbf{\mu}_s^l | \mathbf{\Lambda}_s^l\right)
p\left(\mathbf{\mu}_t^l | \mathbf{\Lambda}_t^l\right)
p\left(\mathbf{\Lambda}_s^l, \mathbf{\Lambda}_t^l\right). \end{aligned}
\end{equation}
We can see that the posterior of the parameters is equal to the product of
the posteriors of the parameters of each class: \vspace{-.2cm} 
\begin{eqnarray}  \label{posterior2}
p(\mu_t,\mu_s,\mathbf{\Lambda}_t,\mathbf{\Lambda}_s|\mathcal{D}_t,\mathcal{D}%
_s) = \prod_{l=1}^L p(\mathbf{\mu}_t^l,\mathbf{\mu}_s^l,\mathbf{\Lambda}_t^l,%
\mathbf{\Lambda}_s^l|\mathcal{D}_t^l,\mathcal{D}_s^l) ,
\end{eqnarray}
where \vspace{-.2cm} 
\begin{eqnarray}  \label{posterior3}
p(\mathbf{\mu}_t^l,\mathbf{\mu}_s^l,\mathbf{\Lambda}_t^l,\mathbf{\Lambda}%
_s^l|\mathcal{D}_t^l,\mathcal{D}_s^l) \propto p(\mathcal{D}_t^l|\mathbf{\mu}%
_t^l,\mathbf{\Lambda}_t^l) p(\mathcal{D}_s^l|\mathbf{\mu}_s^l,\mathbf{\Lambda%
}_s^l)  \notag \\
\times p\left(\mathbf{\mu}_s^l | \mathbf{\Lambda}_s^l\right) p\left(\mathbf{%
\mu}_t^l | \mathbf{\Lambda}_t^l\right) p\left(\mathbf{\Lambda}_s^l, \mathbf{%
\Lambda}_t^l\right).
\end{eqnarray}
Since we are interested in the posterior of the parameters of the target
domain, we integrate out the parameters of the source domain in (\ref%
{posterior2}): 
\vspace{-.2cm}
\begin{equation}
\begin{aligned} \label{posterior4} p(\mathbf{\mu}_t,\mathbf{\Lambda}_t
&|\mathcal{D}_t,\mathcal{D}_s) = \int_{\mub_s,\Lambdab_s}
p(\mathbf{\mu}_t,\mathbf{\mu}_s,\mathbf{\Lambda}_t,\mathbf{\Lambda}_s|%
\mathcal{D}_t,\mathcal{D}_s)d\mathbf{\mu}_s d\mathbf{\Lambda}_s \\ &=
\prod_{l=1}^L \int_{\mub_s^l,\Lambdab_s^l}
p(\mathbf{\mu}_t^l,\mathbf{\mu}_s^l,\mathbf{\Lambda}_t^l,\mathbf{%
\Lambda}_s^l |\mathcal{D}_t^l,\mathcal{D}_s^l) d\mathbf{\mu}_s^l
d\mathbf{\Lambda}_s^l \\ &=\prod_{l=1}^L
p(\mathbf{\mu}_t^l,\mathbf{\Lambda}_t^l |\mathcal{D}_t^l,\mathcal{D}_s^l),
\nonumber \end{aligned}
\end{equation}
where 
\begin{equation}
\begin{aligned} &p(\mathbf{\mu}_t^l,\mathbf{\Lambda}_t^l
|\mathcal{D}_t^l,\mathcal{D}_s^l) \\ &= \int_{\mub_s^l,\Lambdab_s^l}
p(\mathbf{\mu}_t^l,\mathbf{\mu}_s^l,\mathbf{\Lambda}_t^l,\mathbf{%
\Lambda}_s^l|\mathcal{D}_t^l,\mathcal{D}_s^l) d\mathbf{\mu}_s^l
d\mathbf{\Lambda}_s^l \\ &\propto
p(\mathcal{D}_t^l|\mathbf{\mu}_t^l,\mathbf{\Lambda}_t^l)
p\left(\mathbf{\mu}_t^l | \mathbf{\Lambda}_t^l\right) \\ &\times
\int_{\mub_s^l,\Lambdab_s^l}
p(\mathcal{D}_s^l|\mathbf{\mu}_s^l,\mathbf{\Lambda}_s^l)
p\left(\mathbf{\mu}_s^l | \mathbf{\Lambda}_s^l\right)
p\left(\mathbf{\Lambda}_s^l, \mathbf{\Lambda}_t^l\right) d\mathbf{\mu}_s^l
d\mathbf{\Lambda}_s^l. \end{aligned}  \label{posterior5}
\end{equation}

\begin{theorem}
\label{thm:posterior}
Given the target $\mathcal{D}_t$ and source $\mathcal{D}_s$ data, the posterior distribution of target mean $\mu_t^l$ and target precision matrix $\Lambdab_t^l$ for the class $l\in \{1,\cdots,L\}$ has Gaussian-hypergeometric-function distribution 
\begin{equation}  
\label{prop4}
\begin{aligned} &p(\mathbf{\mu}_t^l,\mathbf{\Lambda}_t^l
|\mathcal{D}_t^l,\mathcal{D}_s^l) = \\ & A^l
\left|\mathbf{\Lambda}_t^l\right|^{\frac{1}{2}} \exp
\left(-\frac{\kappa_{t,n}^l}{2}\left(\mathbf{\mu}_t^l -
\mathbf{m}_{t,n}^l\right)^{'}\mathbf{\Lambda}_t^l \left(\mathbf{\mu}_t^l -
\mathbf{m}_{t,n}^l\right) \right) \\ &\times
\left|\mathbf{\Lambda}_{t}^l\right|^{\frac{\nu^l + n_t^l -d-1}{2}}
\mathrm{etr}\left(-\frac{1}{2}
{\left(\mathbf{T}_t^l\right)}^{-1}\mathbf{\Lambda}_{t}^l\right) \\ & \times
~_1F_1\left(\frac{\nu^l + n_s^l}{2}; \frac{\nu^l}{2}; \frac{1}{2}
\mathbf{F}^l \mathbf{\Lambda}_{t}^l {\mathbf{F}^l}^{'} \mathbf{T}_s^l
\right), \end{aligned}
\end{equation}
where $A^l$ is the constant of proportionality 
\begin{equation}  \label{A4}
\begin{aligned}  
&{\left(A^l\right)}^{-1} =\left(\frac{2\pi}{\kappa_{t,n}^l}\right)^{\frac{d}{2}}
2^{\frac{d\left(\nu^l+n_t^l \right)}{2}} \Gamma_d
\left(\frac{\nu^l+n_t^l}{2} \right) \left|\mathbf{T}_t^l\right|^{\frac{\nu^l
+ n_t^l}{2}} \\ 
& ~~~~~~ \times ~_2F_1\left(\frac{\nu^l + n_s^l}{2},
\frac{\nu^l + n_t^l}{2}; \frac{\nu^l}{2}; \mathbf{T}_s^l\mathbf{F}^l
\mathbf{T}_t^l {\mathbf{F}^l}^{'} \right), 
\end{aligned}
\end{equation}
and 
\begin{equation}  \label{const1}
\begin{aligned} &\kappa_{t,n}^l = \kappa_t^l + n_t^l,  \\ 
&\mathbf{m}_{t,n}^l =
\frac{\kappa_t^l \m_t^l + n_t^l \bar{\x}_t^l}{\kappa_t^l+n_t^l}, \\
 &{\left(\mathbf{T}_t^l\right)}^{-1} =
{\left(\mathbf{M}_{t}^l\right)}^{-1} + {\mathbf{F}^l}^{'}\mathbf{C}^l
\mathbf{F}^l + \mathbf{S}_t^l \\ & \hspace{2cm} + \frac{\kappa_t^l
n_t^l}{\kappa_t^l + n_t^l} (\mathbf{m}_t^l -\bar{\x}_t^l)(\mathbf{m}_t^l
-\bar{\x}_t^l)^{'}, \\ &{\left(\mathbf{T}_s^l\right)}^{-1} =
{\left(\mathbf{C}^l\right)}^{-1} + \mathbf{S}_s^l + \frac{\kappa_s^l
n_s^l}{\kappa_s^l + n_s^l} (\mathbf{m}_s^l -\bar{\x}_s^l)(\mathbf{m}_s^l
-\bar{\x}_s^l)^{'}, \end{aligned}
\end{equation}
with sample means and covariances for $z\in\{s,t\}$ as 
\begin{equation}  \label{const2}
\bar{\mathbf{x}}_z^l = \frac{1}{n_z^l} \sum_{i=1}^{n_z^l} \mathbf{x}%
_{z,i}^l, ~~~ \mathbf{S}_z^l = \sum_{i=1}^{n_z^l} \left(\mathbf{x}_{z,i}^l - 
\bar{\mathbf{x}}_z^l \right)\left(\mathbf{x}_{z,i}^l - \bar{\mathbf{x}}_z^l
\right)^{^{\prime }}.  \notag
\end{equation}
\end{theorem}

\begin{proof}
\label{proof_posterior}
See Appendix \ref{appendix:posterior}.
\end{proof}

\section{Effective Class-Conditional Densities}

\label{sec4} In classification, the feature-label distributions are written in terms of class-conditional densities and prior class probabilities, and the posterior probabilities of the classes upon observation of data are proportional to the product of class-conditional densities and prior class probabilities, according to the Bayes rule. This also holds in the Bayesian setting except we use effective class-conditional densities, as shown in \cite{Lori1,Lori2}. For optimal Bayesian classifier \cite{Lori1,Lori2}, using the posterior predictive densities of the classes, called ``effective
class-conditional densities", leads to the optimal choices for classifiers
in order to minimize the Bayesian error estimates of the classifiers.
Similarly, we can derive the effective class-conditional densities for
defining the OBTL classifier in the target domain, albeit with the posterior
of the target parameters derived from both the target and source datasets. 

Suppose that $\mathbf{x}$ denotes a $d\times 1$ new observed data point in
the target domain that we aim to optimally classify into one of the
classes $l\in \{1,\cdots ,L\}$. In the context of the optimal Bayesian
classifier, we need the effective class-conditional densities for the $L$
classes, defined as 
\begin{equation}
p(\mathbf{x}|l)=\int_{\mathbf{\mu }_{t}^{l},\mathbf{\Lambda }_{t}^{l}}p(%
\mathbf{x}|\mathbf{\mu }_{t}^{l},\mathbf{\Lambda }_{t}^{l})\pi ^{\star }(%
\mathbf{\mu }_{t}^{l},\mathbf{\Lambda }_{t}^{l})d\mathbf{\mu }_{t}^{l}d%
\mathbf{\Lambda }_{t}^{l},  \label{eff}
\end{equation}%
for $l\in \{1,\cdots ,L\}$, where $\pi ^{\star }(\mathbf{\mu }_{t}^{l},%
\mathbf{\Lambda }_{t}^{l})=p(\mathbf{\mu }_{t}^{l},\mathbf{\Lambda }_{t}^{l}|%
\mathcal{D}_{t}^{l},\mathcal{D}_{s}^{l})$ is the posterior of $(\mathbf{\mu }%
_{t}^{l},\mathbf{\Lambda }_{t}^{l})$ upon observation of $\mathcal{D}%
_{t}^{l} $ and $\mathcal{D}_{s}^{l}$.

\begin{theorem}
\label{thm-effective}
The effective class-conditional density, denoted by $p(\mathbf{x}|l)=O_{\mathrm{OBTL}}(\mathbf{x}| l)$, in the target domain is given by
\begin{equation}
\begin{aligned} &O_{\mathrm{OBTL}}(\mathbf{x}| l) = \pi^{-\frac{d}{2}}
\left(\frac{\kappa_{t,n}^l}{\kappa_\x^l} \right)^{\frac{d}{2}} \Gamma_d
\left(\frac{\nu^l+n_t^l + 1}{2} \right) \\ & \times \Gamma_d^{-1}
\left(\frac{\nu^l+n_t^l}{2} \right)
\left|\mathbf{T}_\x^l\right|^{\frac{\nu^l + n_t^l + 1}{2}}
\left|\mathbf{T}_t^l\right|^{-\frac{\nu^l + n_t^l}{2}} \\ & \times
~_2F_1\left(\frac{\nu^l + n_s^l}{2}, \frac{\nu^l + n_t^l + 1}{2};
\frac{\nu^l}{2}; \mathbf{T}_s^l\mathbf{F}^l \mathbf{T}_\x^l
{\mathbf{F}^l}^{'} \right) \\ & \times ~_2F_1^{-1}\left(\frac{\nu^l +
n_s^l}{2}, \frac{\nu^l + n_t^l}{2}; \frac{\nu^l}{2};
\mathbf{T}_s^l\mathbf{F}^l \mathbf{T}_t^l {\mathbf{F}^l}^{'} \right),
\end{aligned}  \label{eff4}
\end{equation}
where 
\begin{equation}
\begin{aligned} 
& \kappa_\x^l = \kappa_{t,n}^l + 1 = \kappa_t^l + n_t^l + 1, \\
& {\left(\mathbf{T}_\x^l\right)}^{-1} =
{\left(\mathbf{T}_t^l\right)}^{-1} + \frac{\kappa_{t,n}^l}{\kappa_{t,n}^l +
1} \left(\mathbf{m}_{t,n}^l-\mathbf{x} \right)
\left(\mathbf{m}_{t,n}^l-\mathbf{x} \right)^{'}.
\end{aligned}
\label{update_1}
\end{equation}
\end{theorem}

\begin{proof}
See Appendix \ref{appendix:effective}.
\end{proof}

\section{Optimal Bayesian Transfer Learning Classifier}

\label{sec5}

Let $c_{t}^{l}$ be the prior probability that the target sample $\mathbf{x}$
belongs to the class $l\in \{1,\cdots ,L\}$. Since $0<c_{t}^{l}<1$ and $%
\sum_{l=1}^{L}c_{t}^{l}=1$, a Dirichlet prior is assumed: 
\begin{equation}
(c_{t}^{1},\cdots ,c_{t}^{L})\sim \mathrm{Dir}(L,\mathbf{\xi }_{t}),
\end{equation}%
where $\mathbf{\xi }_{t}=(\xi _{t}^{1},\cdots ,\xi _{t}^{L})$ are the
concentration parameters, and $\xi _{t}^{l}>0$ for $l\in \{1,\cdots ,L\}$.
As the Dirichlet distribution is a conjugate prior for the categorical
distribution, upon observing $\mathbf{n}=(n_{t}^{1},\cdots ,n_{t}^{L})$ data
for class $l$ in the target domain, the posterior has a Dirichlet
distribution: 
\begin{equation}
\begin{aligned} \pi^{\star} = (c_t^1,\cdots,c_t^L| \mathbf{n}) &\sim
\mathrm{Dir}(L,\mathbf{\xi}_t+\mathbf{n}) \\ & =
\mathrm{Dir}(L,\xi_t^1+n_t^1, \cdots,\xi_t^L+n_t^L), \end{aligned}
\end{equation}%
with the posterior mean of $c_{t}^{l}$ as 
\begin{equation}
\mathrm{E}_{\pi ^{\star }}(c_{t}^{l})=\frac{\xi _{t}^{l}+n_{t}^{l}}{%
N_{t}+\xi _{t}^{0}},
\end{equation}%
where $N_{t}=\sum_{l=1}^{L}n_{t}^{l}$ and $\xi _{t}^{0}=\sum_{l=1}^{L}\xi
_{t}^{l}$. As such, the optimal Bayesian transfer learning (OBTL) classifier
for any new unlabeled sample $\mathbf{x}$ in the target domain is defined as 
\begin{equation}
\Psi _{\mathrm{OBTL}}(\mathbf{x})=\arg \!\max_{l\in \{1,\cdots ,L\}}\mathrm{E%
}_{\pi ^{\star }}(c_{t}^{l})O_{\mathrm{OBTL}}(\mathbf{x}|l),  \label{T-OBC}
\end{equation}%
which minimizes the expected error of the classifier in the target domain, that is, $\mathrm{E}_{\pi^{\star}}[\varepsilon(\Theta_t,\Psi _{\mathrm{OBTL}})] \leq \mathrm{E}_{\pi^{\star}}[\varepsilon(\Theta_t,\Psi)]$, where $\varepsilon(\Theta_t,\Psi)$ is the error of any arbitrary classifier $\Psi$ assuming the parameters $\Theta_t=\{c_t^l,\mu_t^l,\Lambdab_t^l\}_{l=1}^L$ of the feature-label distribution in the target domain, and the expectation is over the posterior $\pi^{\star}$ of $\Theta_t$ upon observation of data. If we do not have any prior knowledge for the selection of classes, we use the same concentration
parameter for all the classes: $\mathbf{\xi }_{t}=(\xi ,\cdots ,\xi )$.
Hence, if the number of samples in each class is the same, $n_{t}^{1}=\cdots
=n_{t}^{L}$, the first term $\mathrm{E}_{\pi ^{\star }}(c_{t}^{l})$ is the same for all
the classes and (\ref{T-OBC}) is reduced to: 
\begin{equation}
\Psi _{\mathrm{OBTL}}(\mathbf{x})=\arg \!\max_{l\in \{1,\cdots ,L\}}O_{%
\mathrm{OBTL}}(\mathbf{x}|l).  \label{OBTL_2}
\end{equation}

 We have derived the effective class-conditional densities in
closed forms (\ref{eff4}). However, deriving the OBTL classifier (\ref{T-OBC}%
) requires computing the Gauss hypergeometric function of matrix argument.
Computing the exact values of hypergeometirc functions of matrix argument
using the series of zonal polynomials, as in (\ref{Gauss}), is time-consuming
and is not scalable to high dimension. To facilitate computation, we propose
to use the Laplace approximation of this function, as in \cite%
{Laplace_approx}, which is computationally efficient and scalable. See Appendix \ref{appendix:Laplace} for the detailed description of the Laplace approximation of Gauss hypergeometric functions of matrix argument.

\vspace{-.2cm}

\section{OBC in Target Domain}

\label{sec6}

To see how the source data can help improve the performance, we compare the
OBTL classifier with the OBC based on the training data only from the target domain. Using
exactly the same modeling and parameters as the previous sections, the
priors for $\mathbf{\mu}_t^l$ and $\mathbf{\Lambda}_t^l$, from (\ref%
{mu_s_mu_t}) and (\ref{marg_t}), are given by \vspace{-.2cm} 
\begin{equation}
\vspace{-.2cm} \begin{aligned} \mathbf{\mu}_t^l | \mathbf{\Lambda}_t^l &\sim
\mathcal{N}\left(\mathbf{m}_t^l,\left(\kappa_t^l
\mathbf{\Lambda}_t^l\right)^{-1}\right), \\ \mathbf{\Lambda}_t^l &\sim
W_d(\mathbf{M}_t^l,\nu^l). \end{aligned}
\end{equation}
Using Lemma \ref{Lemma1} in Appendix \ref{appendix:posterior}, upon observing the dataset $\mathcal{D}_t^l$, the
posteriors of $\mathbf{\mu}_t^l$ and $\mathbf{\Lambda}_t^l$ will be \vspace{%
-.2cm} 
\begin{equation}
\begin{aligned} \mathbf{\mu}_t^l|\mathbf{\Lambda}_t^l , \mathcal{D}_t^l
&\sim \mathcal{N}\left(\mathbf{m}_{t,n}^l,
\left(\kappa_{t,n}^l\mathbf{\Lambda}_t^l\right)^{-1}\right), \\
\mathbf{\Lambda}_t^l | \mathcal{D}_t^l &\sim
W_d(\mathbf{M}_{t,n}^l,\nu_{t,n}^l), \end{aligned}
\end{equation}
where \vspace{-.2cm} 
\begin{equation}  \label{const11}
\begin{aligned} & \kappa_{t,n}^l = \kappa_t^l + n_t^l, ~~~ \nu_{t,n}^l =
\nu^l + n_t^l, ~~~ \mathbf{m}_{t,n}^l = \frac{\kappa_t^l \m_t^l + n_t^l
\bar{\x}_t^l}{\kappa_t^l+n_t^l}, \\ & {\left(\mathbf{M}_{t,n}^l\right)}^{-1}
= {\left(\mathbf{M}_{t}^l\right)}^{-1} + \mathbf{S}_t^l + \frac{\kappa_t^l
n_t^l}{\kappa_t^l + n_t^l} (\mathbf{m}_t^l -\bar{\x}_t^l)(\mathbf{m}_t^l
-\bar{\x}_t^l)^{'}, \end{aligned}
\end{equation}
with the corresponding sample mean and covariance: 
\begin{equation}
\vspace{-.2cm} \bar{\mathbf{x}}_t^l = \frac{1}{n_t^l} \sum_{i=1}^{n_t^l} 
\mathbf{x}_{t,i}^l, ~~~~ \mathbf{S}_t^l = \sum_{i=1}^{n_t^l} \left(\mathbf{x}%
_{t,i}^l - \bar{\mathbf{x}}_t^l \right)\left(\mathbf{x}_{t,i}^l - \bar{%
\mathbf{x}}_t^l \right)^{^{\prime }}.  \label{const22}
\end{equation}
By (\ref{eff}) and similar integral steps, the effective
class-conditional densities $p(\mathbf{x} | l) = O_{\mathrm{OBC}}(\mathbf{x}%
| l)$ for the OBC are derived as \cite{Lori1} 
\begin{equation}  \label{eff_target}
\begin{aligned} O_{\mathrm{OBC}}(\mathbf{x}| l) = \pi^{-\frac{d}{2}}
\left(\frac{\kappa_{t,n}^l}{\kappa_{t,n}^l + 1} \right)^{\frac{d}{2}}
\Gamma_d \left(\frac{\nu^l+n_t^l + 1}{2} \right) \\ \times ~ \Gamma_d^{-1}
\left(\frac{\nu^l+n_t^l}{2} \right)
\left|\mathbf{M}_\x^l\right|^{\frac{\nu^l + n_t^l + 1}{2}}
\left|\mathbf{M}_{t,n}^l\right|^{-\frac{\nu^l + n_t^l}{2}}, \end{aligned}
\end{equation}
where 
\begin{equation}  \label{update_2}
{\left(\mathbf{M}_{\mathbf{x}}^l\right)}^{-1} = {\left(\mathbf{M}%
_{t,n}^l\right)}^{-1} + \frac{\kappa_{t,n}^l}{\kappa_{t,n}^l + 1} (\mathbf{m}%
_{t,n}^l -\mathbf{x})(\mathbf{m}_{t,n}^l - \mathbf{x})^{^{\prime }}.
\end{equation}
The multi-class OBC \cite{dalton2015optimal}, under a zero-one loss
function, can be defined as 
\begin{equation}
\Psi_{\mathrm{OBC}}(\mathbf{x}) = \arg\!\max_{l\in\{1,\cdots,L\}} \mathrm{E}%
_{\pi^{\star}}(c_t^l) O_{\mathrm{OBC}}(\mathbf{x}| l).  \label{OBC}
\end{equation}
Similar to the OBTL, in the case of equal prior probabilities for the
classes, 
\begin{equation}
\Psi_{\mathrm{OBC}}(\mathbf{x}) = \arg\!\max_{l\in\{1,\cdots,L\}} O_{\mathrm{%
OBC}}(\mathbf{x}| l).  \label{OBC_2}
\end{equation}
For binary classification, the definition of the OBC in (\ref{OBC}) is
equivalent to the definition in \cite{Lori1}, where it is defined to be the
binary classifier possessing the minimum Bayesian mean square error estimate 
\cite{Lori-MMSE} relative to the posterior distribution.

\begin{theorem}
\label{theorem4}  If $\mathbf{M}_{ts}^l=\mathbf{0}$ for all $l\in
\{1,\cdots,L\}$, then 
\begin{equation}
\Psi_{\mathrm{OBTL}}(\mathbf{x}) = \Psi_{\mathrm{OBC}}(\mathbf{x}),
\end{equation}
meaning that if there is no interaction between the source and target
domains in all the classes a priori, then the OBTL classifier turns to the
OBC classifier in the target domain.
\end{theorem}

\begin{proof}
If $\mathbf{M}_{ts}^l=\mathbf{0}$ for all $l\in \{1,\cdots,L\}$, then $%
\mathbf{F}^l = \mathbf{0}$. Since $_2F_1(a,b;c;\mathbf{0})=1$ for any values
of $a$, $b$, and $c$, the Gauss hypergeometric functions will disappear in (%
\ref{eff4}). From (\ref{const1}) and (\ref{const11}), $\mathbf{T}_t^l=%
\mathbf{M}_{t,n}^l$. From (\ref{update_1}) and (\ref{update_2}), $\mathbf{T}%
_{\mathbf{x}}^l=\mathbf{M}_{\mathbf{x}}^l$. As a result, $O_{\mathrm{OBTL}}(%
\mathbf{x}| l)=O_{\mathrm{OBC}}(\mathbf{x}| l)$, and consequently, $\Psi_{%
\mathrm{OBTL}}(\mathbf{x})=\Psi_{\mathrm{OBC}}(\mathbf{x})$.
\end{proof}

\vspace{-.5cm}

\section{Experiments}

\label{sec7}

\subsection{Synthetic datasets}

We have considered a simulation setup and evaluated the OBTL classifiers by
the average classification error with different joint prior densities
modeling the relatedness of the source and target domains. The setup is as
follows. Unless mentioned, the feature dimension is $d=10$, the number of
classes in each domain is $L=2$, the number of source training data per
class is $n_{s}=n_{s}^{l}=200$, the number of target training data per class
is $n_{t}=n_{t}^{l}=10$, $\nu =\nu ^{l}=25$, $\kappa _{t}=\kappa
_{t}^{l}=100 $, $\kappa _{s}=\kappa _{s}^{l}=100$, for both the classes $%
l=1,2$, $\mathbf{m}_{t}^{1}=\mathbf{0}_{d}$, $\mathbf{m}_{t}^{2}=0.05\times 
\mathbf{1}_{d}$, $\mathbf{m}_{s}^{1}=\mathbf{m}_{t}^{1}+\mathbf{1}_{d}$, and 
$\mathbf{m}_{s}^{2}=\mathbf{m}_{t}^{2}+\mathbf{1}_{d}$, where $\mathbf{0}_{d}
$ and $\mathbf{1}_{d}$ are $d\times 1$ all-zero and all-one vectors,
respectively. For the scale matrices, we choose $\mathbf{M}_{t}^{l}=k_{t}%
\mathbf{I}_{d}$, $\mathbf{M}_{s}^{l}=k_{s}\mathbf{I}_{d}$, and $\mathbf{M}%
_{ts}^{l}=k_{ts}\mathbf{I}_{d}$ for two classes $l=1,2$, where $\mathbf{I}%
_{d}$ is the $d\times d$ identity matrix. Note that choosing an identity
matrix for $\mathbf{M}_{ts}^{l}$ makes sense when the order of the features
in the two domains is the same. We have the constraint that the scale matrix 
$\mathbf{M}^{l}=%
\begin{psmallmatrix} \mathbf{M}_{t}^l & \mathbf{M}_{ts}^l \\
{\mathbf{M}_{ts}^l}^{'} & \mathbf{M}_{s}^l\end{psmallmatrix}$ should be
positive definite for any class $l$. It is easy to check the following
corresponding constraints on $k_{t}$, $k_{s}$, and $k_{ts}$: $k_{t}>0$, $%
k_{s}>0$, and $|k_{ts}|<\sqrt{k_{t}k_{s}}$. We define $k_{ts}=\alpha \sqrt{%
k_{t}k_{s}}$, where $|\alpha |<1$. In this particular example, the value of $%
|\alpha |$ shows the amount of relatedness between the source and target
domains. If $|\alpha |=0$, the two domains are not related and if $|\alpha |$
is close to one, we have greater relatedness. We set $k_{t}=k_{s}=1$ and
plot the average classification error curves for different values of $%
|\alpha |$. All the simulations assume equal prior probabilities for the
classes, so we use (\ref{OBTL_2}) and (\ref{OBC_2}) for the OBTL classifier
and OBC, respectively.

We evaluate the prediction performance according to the common evaluation
procedure of Bayesian learning by average classification errors. To sample
from the prior (\ref{p_mu}) we first sample from a Wishart distribution $%
W_{2d}(\mathbf{M}^{l},\nu ^{l})$ to get a sample for $\mathbf{\Lambda }^{l}=%
\begin{psmallmatrix} \Lambdab_{t}^l & \Lambdab_{ts}^l \\
{\Lambdab_{ts}^l}^{'} & \Lambdab_{s}^l\end{psmallmatrix}$, for each class $%
l=1,2$, and then pick $(\mathbf{\Lambda }_{t}^{l},\mathbf{\Lambda }_{s}^{l})$%
, which is a joint sample from $p(\mathbf{\Lambda }_{t}^{l},\mathbf{\Lambda }%
_{s}^{l})$ in (\ref{joint}). Then given $\mathbf{\Lambda }_{t}^{l}$ and $%
\mathbf{\Lambda }_{s}^{l}$, we sample from (\ref{mu_s_mu_t}) to get samples
of $\mathbf{\mu }_{t}^{l}$ and $\mathbf{\mu }_{s}^{l}$ for $l=1,2$. Once we
have $\mathbf{\mu }_{t}^{l}$, $\mathbf{\mu }_{s}^{l}$, $\mathbf{\Lambda }%
_{t}^{l}$, and $\mathbf{\Lambda }_{s}^{l}$, we generate $100$ different
training and test sets from (\ref{x_s_x_t}). Training sets contain samples
from both the target and source domains, but the test set contains only
samples from the target domain. As the numbers of source and target training
data per class are $n_{s}$ and $n_{t}$, there are $Ln_{s}$ and $Ln_{t}$
source and target training data in total, respectively. We assume the size
of the test set per class is $1000$ in the simulations, so $2000$ in total.
For each training and test set, we use the OBTL classifier and its target-only version,
OBC, and calculate the error. Then we average all the errors for $100$
different training and test sets. We further repeat this whole process $1000$
times for different realizations of $\mathbf{\Lambda }_{t}^{l}$ and $\mathbf{%
\Lambda }_{s}^{l}$, $\mathbf{\mu }_{t}^{l}$, and $\mathbf{\mu }_{s}^{l}$ for 
$l=1,2$, and finally average all the errors and return the average
classification error. Note that in all figures, the hyperparameters used in
the OBTL classifier are the same as the ones used for simulating data,
except for the figures showing the sensitivity of the performance with
respect to different hyperparameters, in which case we assume that true
values of the hyperparameters used for simulating data are unknown.

\begin{figure}[t!]
\centering
\begin{subfigure}[t]{0.45\textwidth}
        \centering
        \includegraphics[width=\textwidth]{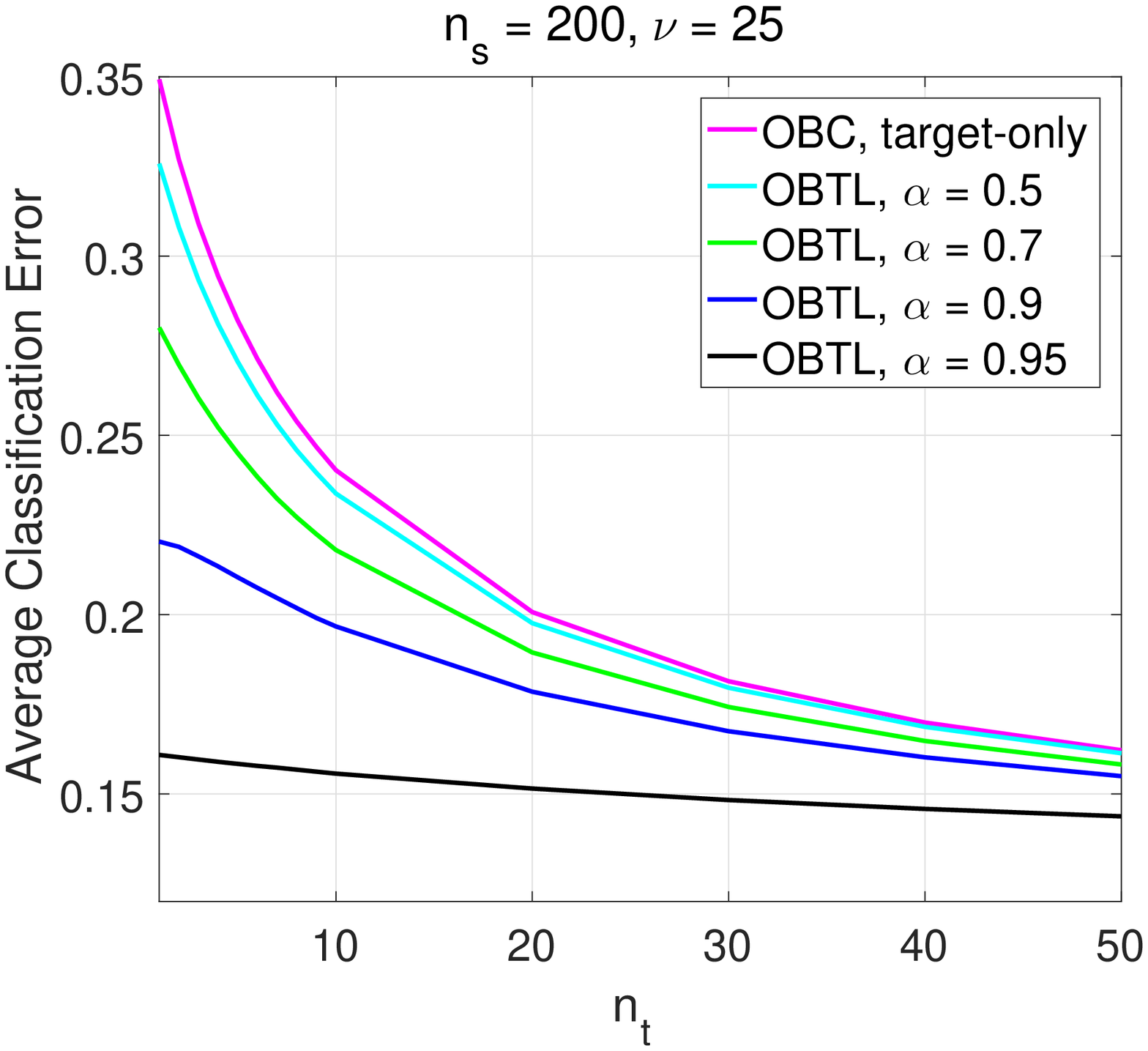}
        \caption{}
    \end{subfigure}
    
\begin{subfigure}[t]{0.45\textwidth}
        \centering
        \includegraphics[width=\textwidth]{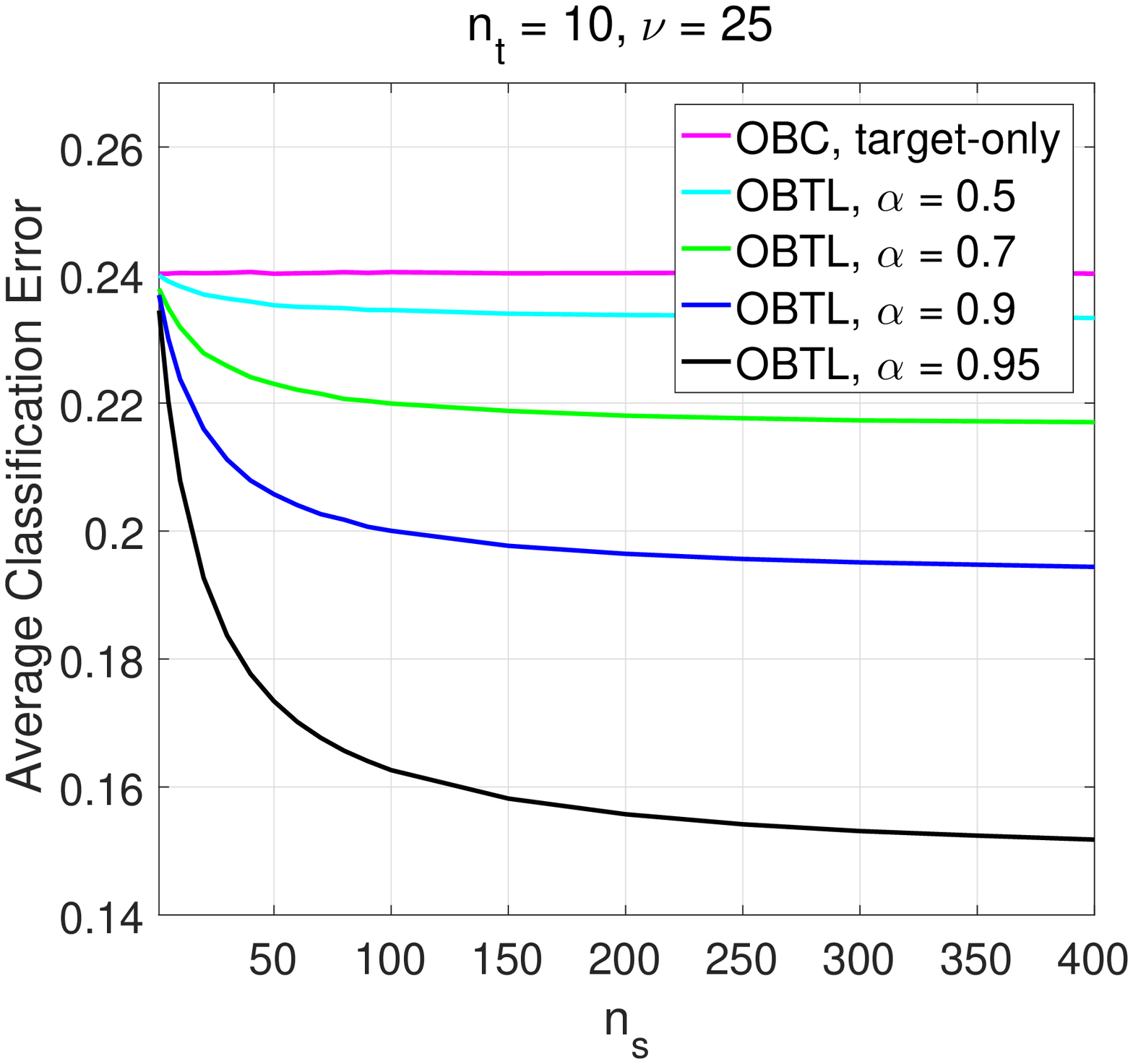}
        \caption{}
    \end{subfigure}
\caption{{\protect\footnotesize (a) Average classification error versus the
number of target training data per class, $n_t$, (b) Average classification
error versus the number of source training data per class, $n_s$.}}
\label{fig2}
\vspace{-.5cm}
\end{figure}

To examine how the source data improves the classifier in target domain, we
compare the performance of the OBTL classifier with the OBC designed in the
target domain alone. The average classification error versus $n_{t}$ is
depicted in Fig. \ref{fig2}a for the OBC and OBTL with different values of $%
\alpha $. When $\alpha $ is close to one, the performance of the OBTL
classifier is much better than that of the OBC, this due to the greater
relatedness between the two domains and appropriate use of the source data.
This performance improvement is especially noticeable when $n_{t}$ is small,
which reflects the real-world scenario. In Fig. \ref{fig2}a, we also observe
that the errors of the OBTL classifier and OBC are converging to a similar
value when $n_{t}$ gets very large, meaning that the source data are
redundant when there is a large amount of target data. When $\alpha $ is
larger, the error curves converge faster to the optimal error, which is the
average Bayes error of the target classifier. The corresponding Bayes error
averaged over $1000$ randomly generated distributions is equal to $0.122$ in this simulation setup. Recall that when $\alpha =0$,
the OBTL classifier reduces to the OBC. In this particular example, the sign
of $\alpha $ does not matter in the performance of the OBTL, which can be
verified by (\ref{eff4}). Hence, we can use $|\alpha |$ in all the cases.

Figure \ref{fig2}b depicts average classification error versus $n_{s}$ for
the OBC and OBTL with different values of $\alpha $. The error of the OBC is
constant for all $n_{s}$ as it does not employ the source data. The error of
the OBTL classifier equals that of the OBC when $n_{s}=0$ and starts to
decrease as $n_{s}$ increases. In Fig. \ref{fig2}b when $\alpha $ is larger,
the amount of improvement is greater since the two domains are more related. Another important 
point in Fig. \ref{fig2}b is that having very large source data when the two domains are highly related can compensate
the lack of target data and lead to a target classification error as small as the Bayes error in the target domain. 

Figure \ref{fig:box_plot} illustrates the box plots of the simulated classification errors corresponding to the $1000$ distributions randomly drawn from the prior distributions for both the OBC and OBTL with $\alpha=0.9$, which show the variability for different numbers $n_t$ of target data per class.

\begin{figure}[t!]
\centering
 \includegraphics[width=.45\textwidth]{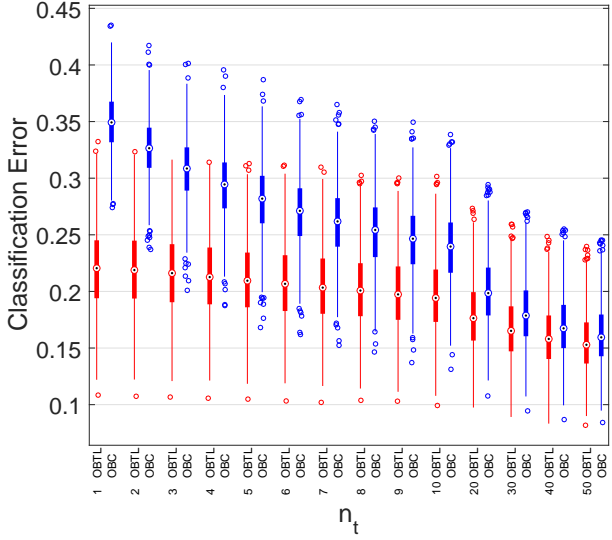}
\caption{{\protect\footnotesize Box plots of $1000$ simulated classification errors for different $n_t$. Blue denotes the OBC and red denotes the OBTL with $\alpha=0.9$.}}
\label{fig:box_plot}
\vspace{-.5cm}
\end{figure}

We investigate the sensitivity of the OBTL with respect to the
hyperparameters. Fig. \ref{fig3} represents the average classification error
of the OBTL with respect to $|\alpha|$, where we assume that we do not know
the true value $\alpha_{true}$ of the amount of relatedness between source
and target domains. In Figs. \ref{fig3}a-\ref{fig3}d we plot the error
curves when $\alpha_{true}=0.3,0.5,0.7,0.9$, respectively. We observe
several important trends in these figures. First of all, the performance
gain of the OBTL towards the OBC depends heavily on the relatedness (value
of $\alpha_{true}$) of source and target and the value of $\alpha$ used in
the classifier. Generally speaking, there exists an $\alpha_{max}$ in $(0,1)$
such that for $|\alpha| < \alpha_{max}$, the OBTL has a performance gain
towards the OBC, where the maximum gain is achieved at $|\alpha| = \alpha_{true}$ (it might not be exactly at $\alpha_{true}$ due to the Laplace approximation of the Gauss hypergeometric function). Second, the performance gain is higher when the two domains are highly
related (Fig. \ref{fig3}d). Third, when the two domains are very related,
for example, $\alpha_{true}=0.9$ in Fig. \ref{fig3}d, $\alpha_{max}=1$,
meaning that for any $|\alpha|$, the OBTL has performance gain towards the
target-only OBC. However, when the source and target domains are not related
much, like Figs. \ref{fig3}a and \ref{fig3}b, $\alpha_{max}<1$, and choosing $%
|\alpha|$ greater than $\alpha_{max}$ leads to performance loss compared to
the OBC. This means that exaggeration in the amount of relatedness between
the two domains can hurt the transfer learning classifier when the two
domains are not actually related, which refers to the concept of negative
transfer.

\begin{figure}[t!]
\centering
\begin{subfigure}[t]{0.23\textwidth}
        \centering
        \includegraphics[width=\textwidth]{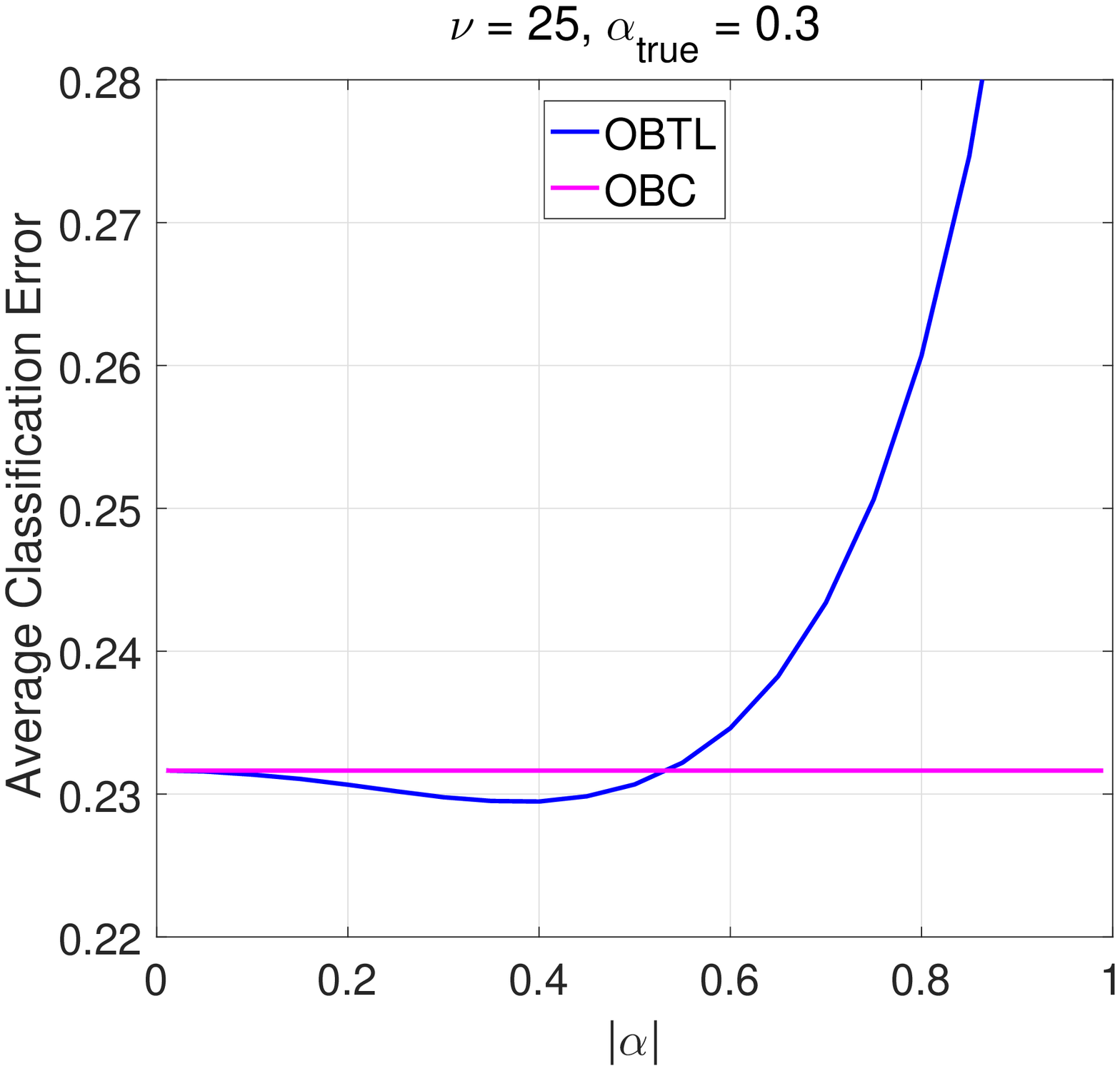}
        \caption{}
    \end{subfigure}
~ 
\begin{subfigure}[t]{0.23\textwidth}
        \centering
        \includegraphics[width=\textwidth]{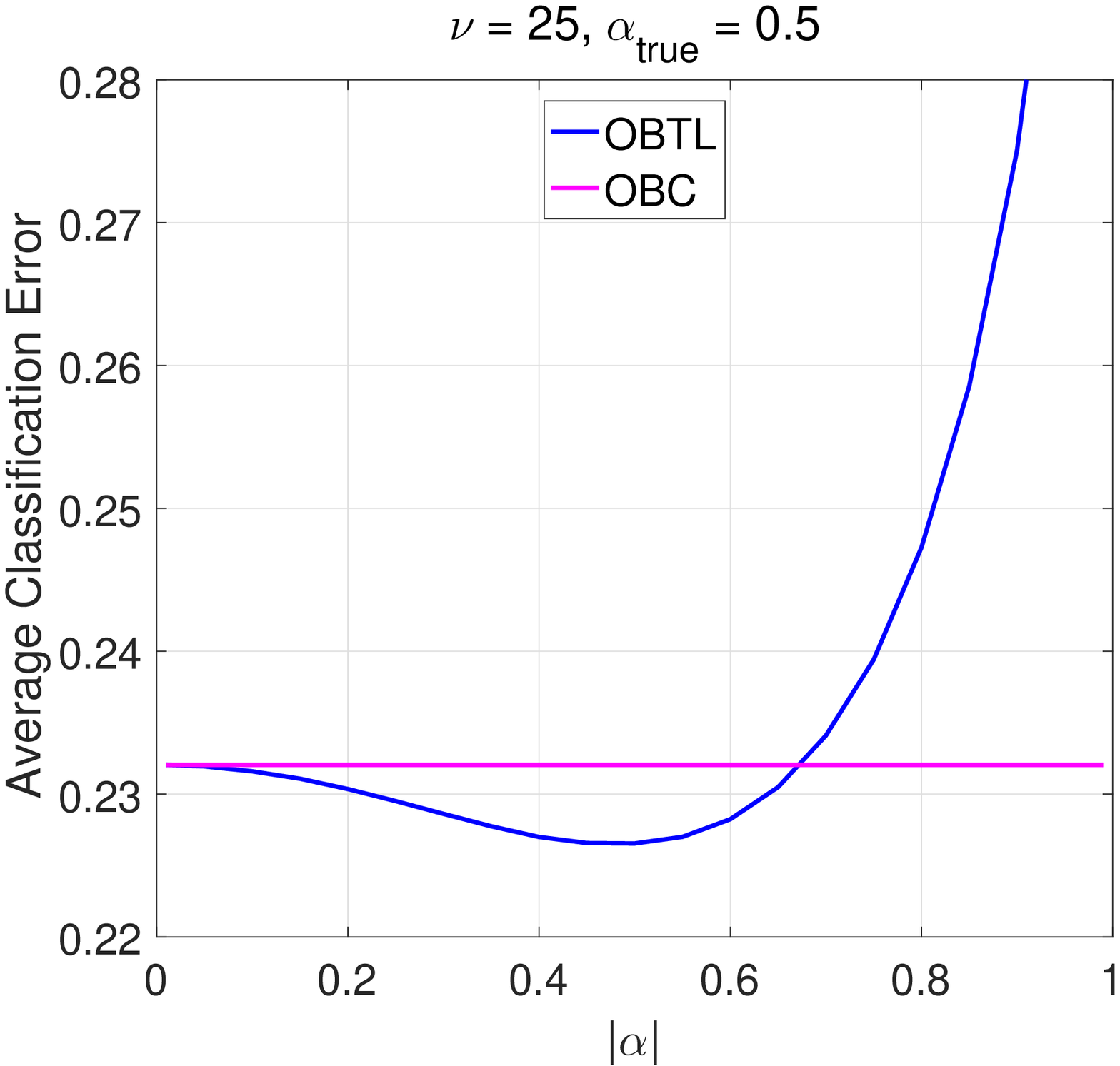}
        \caption{}
    \end{subfigure}
\par
\begin{subfigure}[t]{0.23\textwidth}
        \centering
        \includegraphics[width=\textwidth]{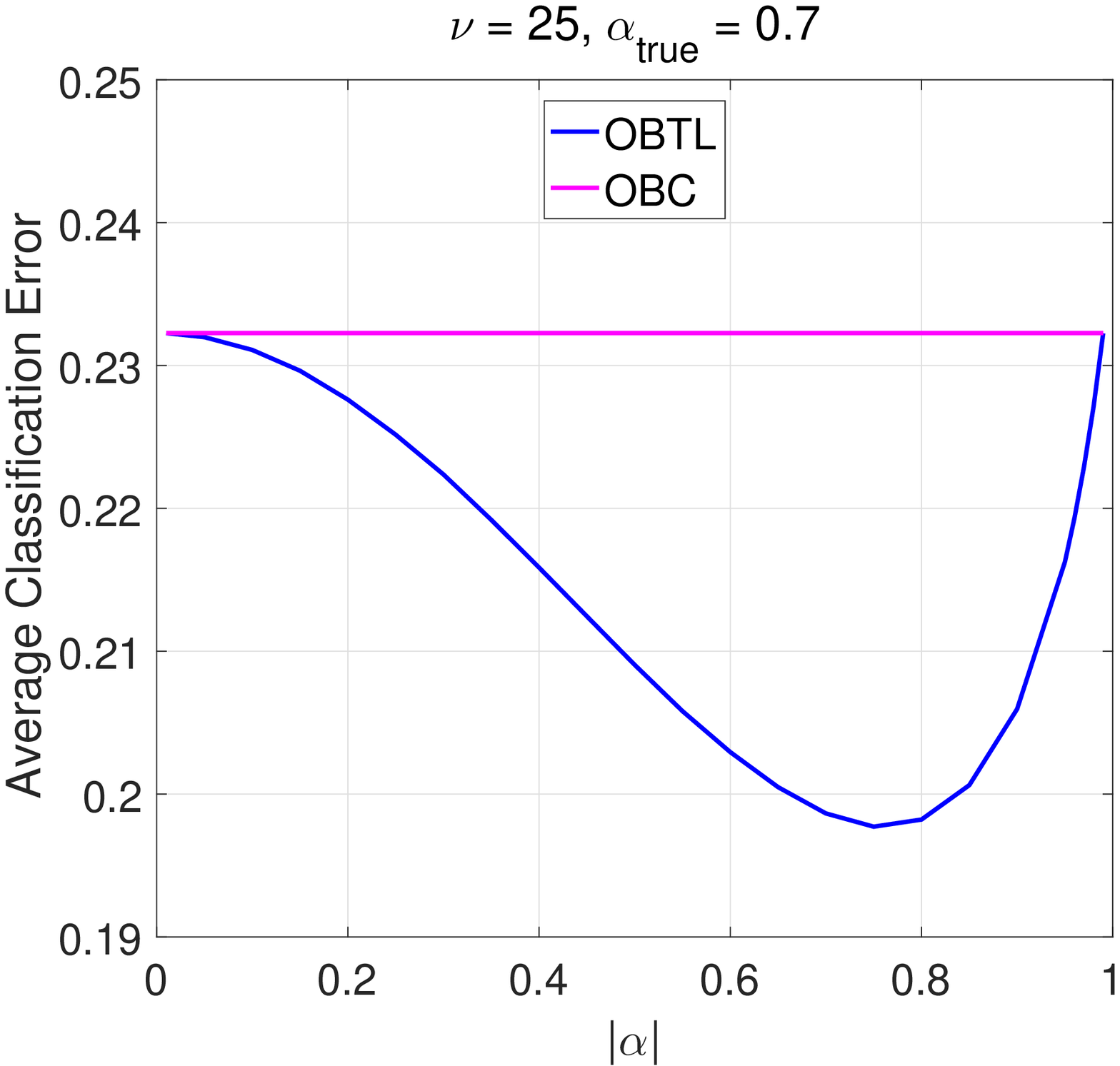}
        \caption{}
    \end{subfigure}     
~ 
\begin{subfigure}[t]{0.23\textwidth}
        \centering
        \includegraphics[width=\textwidth]{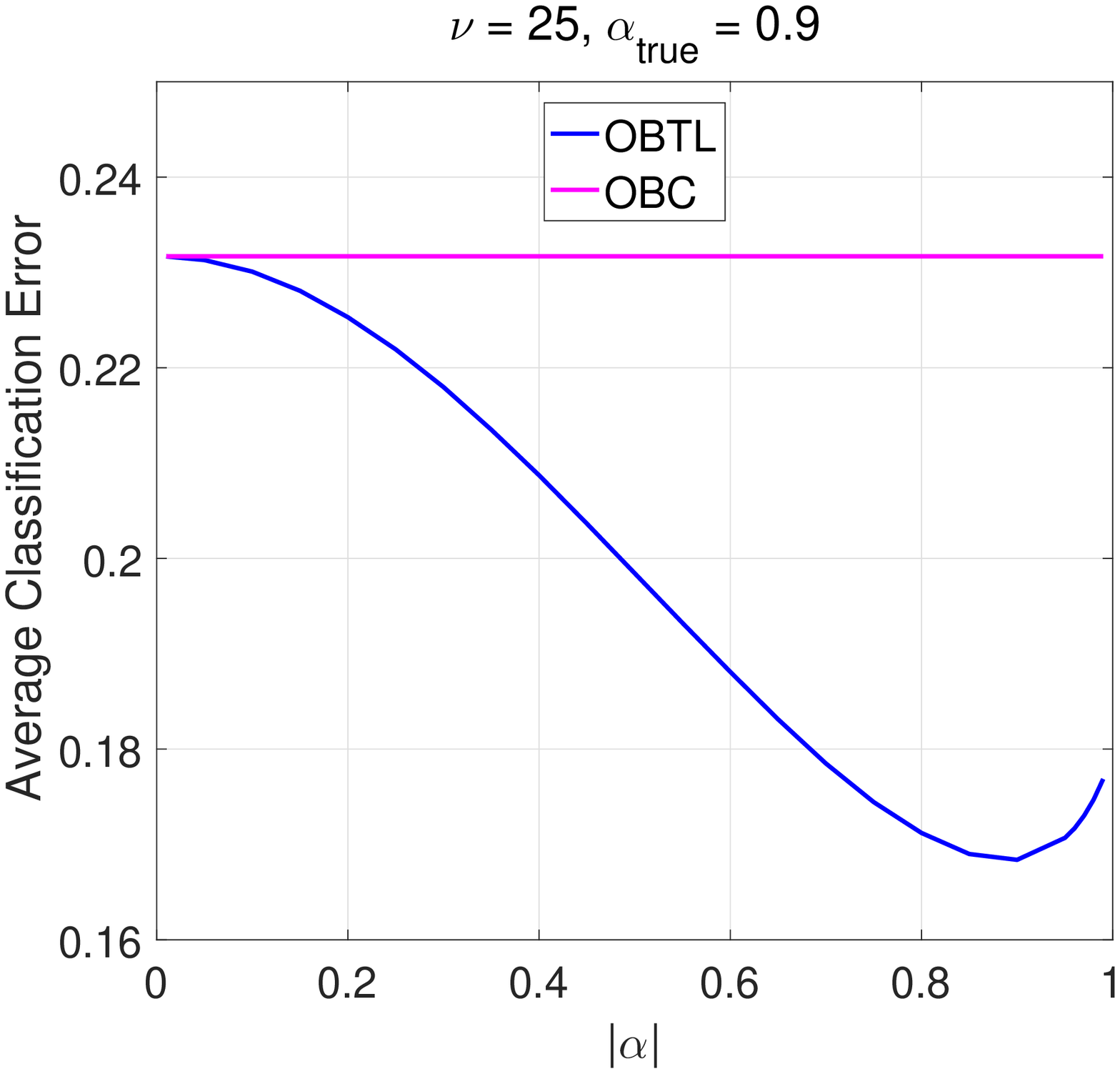}
        \caption{}
    \end{subfigure}

\caption{{\protect\footnotesize Average classification error vs $|\protect%
\alpha|$}}
\label{fig3}
\vspace{-.7cm}
\end{figure}

Figure \ref{fig4} shows the errors versus $\nu $, assuming unknown true
value $\nu _{true}$, for different values of $\alpha $ ($0.5$ and $0.9$) and 
$\nu _{true}$ ($25$ and $50$). The salient point here is that the
performance of the OBTL classifier is not so sensitive to $\nu $ if it is
chosen in its allowable range, that is, $\nu \geq 2d$. In Fig. \ref{fig4},
the error of the OBTL does not change much for $\nu \geq 2d=20$. As a
result, we can choose any arbitrary $\nu \geq 2d$ in real datasets without
worrying about critical performance deterioration.

\begin{figure}[t!]
\centering
\begin{subfigure}[t]{0.23\textwidth}
        \centering
        \includegraphics[width=\textwidth]{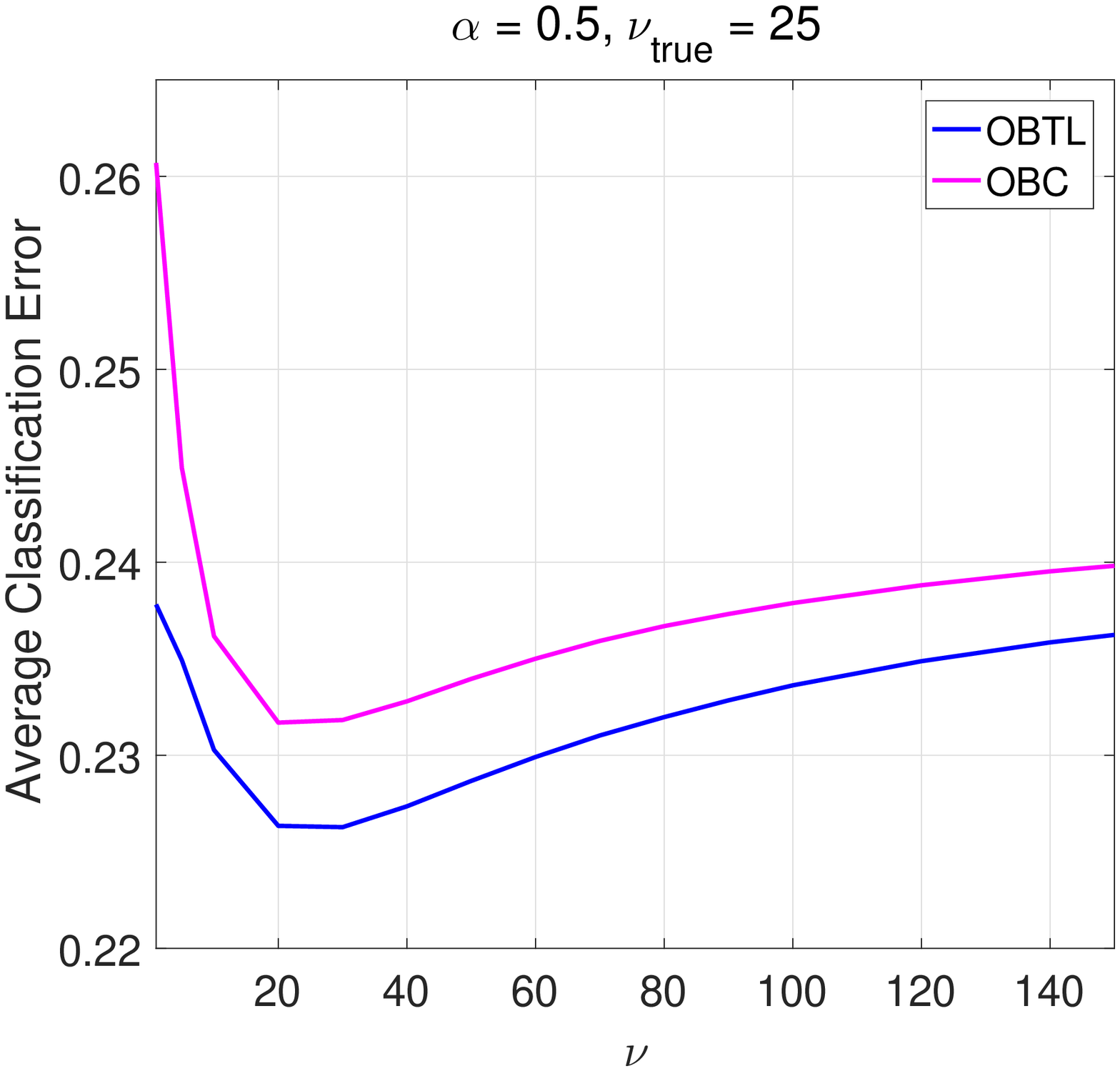}
        \caption{}
    \end{subfigure}
~ 
\begin{subfigure}[t]{0.23\textwidth}
        \centering
        \includegraphics[width=\textwidth]{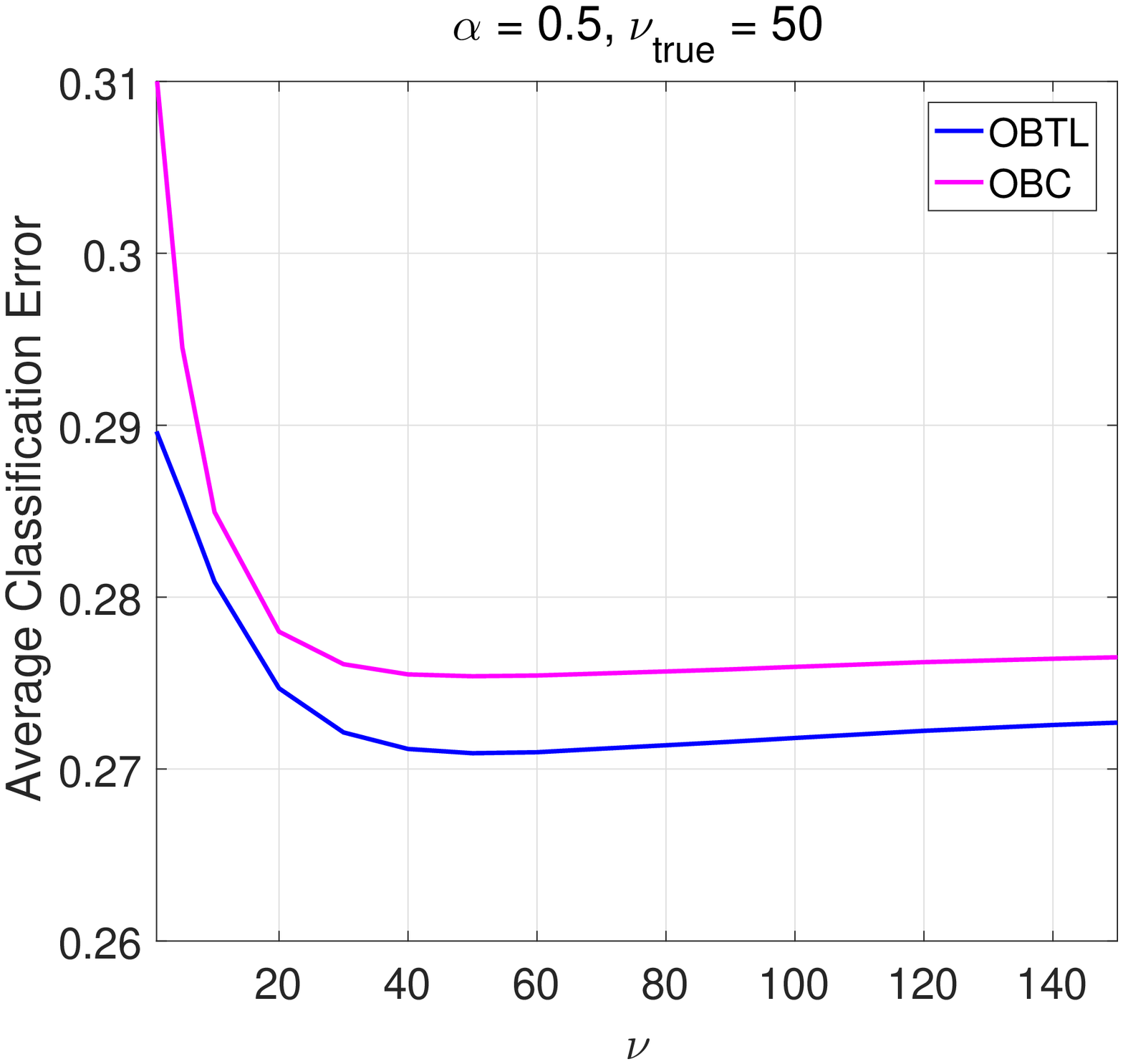}
        \caption{}
    \end{subfigure}
\par
\begin{subfigure}[t]{0.23\textwidth}
        \centering
        \includegraphics[width=\textwidth]{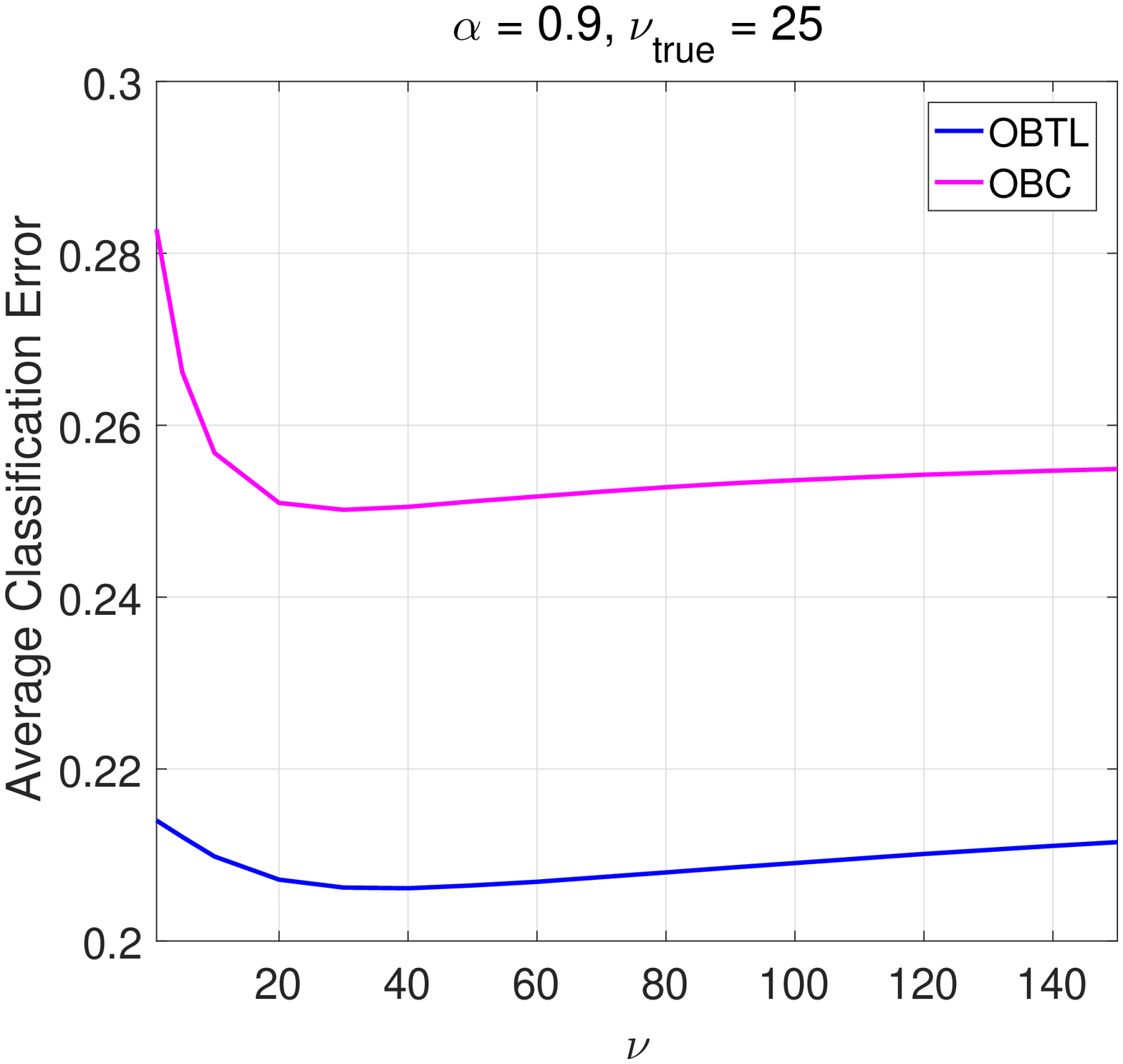}
        \caption{}
    \end{subfigure}     
~ 
\begin{subfigure}[t]{0.23\textwidth}
        \centering
        \includegraphics[width=\textwidth]{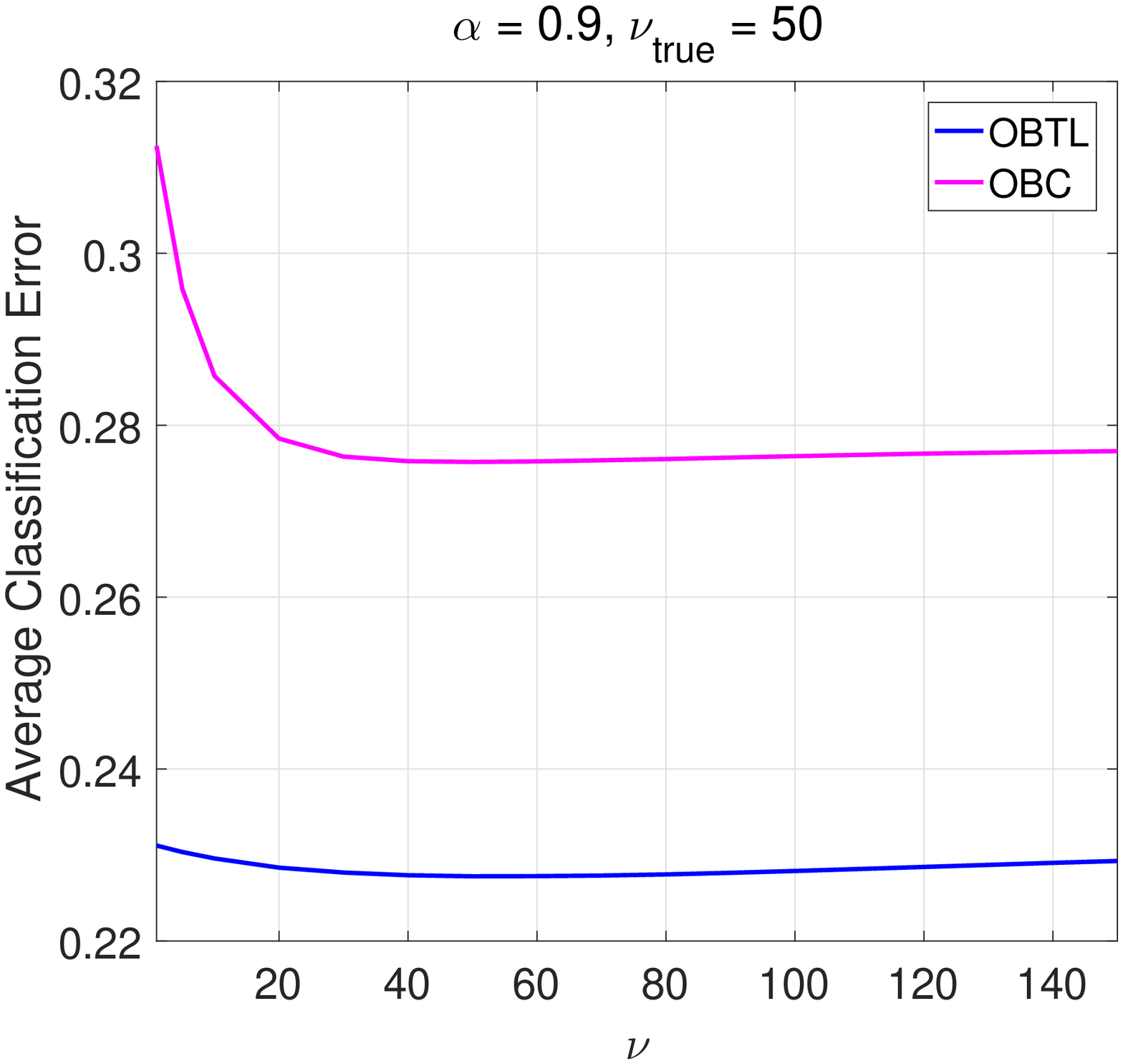}
        \caption{}
    \end{subfigure}
\caption{{\protect\footnotesize Average classification error vs $\protect\nu$%
}}
\label{fig4}
\vspace{-.5cm}
\end{figure}

Figure \ref{fig5} depicts average classification error versus $\kappa _{t}$
for two different values of $\alpha $ ($0.5$ and $0.9$), where the true
value of $\kappa _{t}$ is $\kappa _{true}=50$. Similar to $\nu $, if $\kappa
_{t}$ is greater than a value ($20$ in Fig. \ref{fig5}), the performance
does not change much. According to (\ref{const1}), it is better to choose $%
\kappa _{t}^{l}$ and $\kappa _{s}^{l}$ to be proportional to $n_{t}$ and $%
n_{s}$, respectively, since the values of updated means $\mathbf{m}%
_{t,n}^{l} $ and $\mathbf{m}_{s,n}^{l}$ are weighted averages of our prior
knowledge about means, $\mathbf{m}_{t}^{l}$ and $\mathbf{m}_{s}^{l}$, and
the sample means $\bar{\mathbf{x}}_{t}^{l}$ and $\bar{\mathbf{x}}_{s}^{l}$.
Assuming that $\kappa _{t}=\beta _{t}n_{t}$ and $\kappa _{s}=\beta _{s}n_{s}$%
, for some $\beta _{t},\beta _{s}>0$, if we have higher confidence on our priors
on means, we pick higher $\beta _{t}$ and $\beta _{s}$ (as in Fig. \ref%
{fig5}); but for the untrustworthy priors, we choose lower values for $\beta
_{t}$ and $\beta _{s}$.

Sensitivity results in Figs. \ref{fig3}, \ref{fig4}, and \ref{fig5} reveal
that in our simulation setup the performance improvement of the OBTL depends
on the value of $\alpha$ and true relatedness ($\alpha_{true}$ in this
example) between the two domains and is not affected that much by the
choices of other hyperparameters like $\nu$, $\kappa_t$, and $\kappa_s$. We
could have a reasonable range of $\alpha$ to get improved performance but
the correct estimates of \textit{relatedness} or \textit{transferability}
are critical, which is an important future research direction (see
Conclusions in Section \ref{sec8}).

\begin{figure}[t!]
\centering
\begin{subfigure}[t]{0.23\textwidth}
        \centering
        \includegraphics[width=\textwidth]{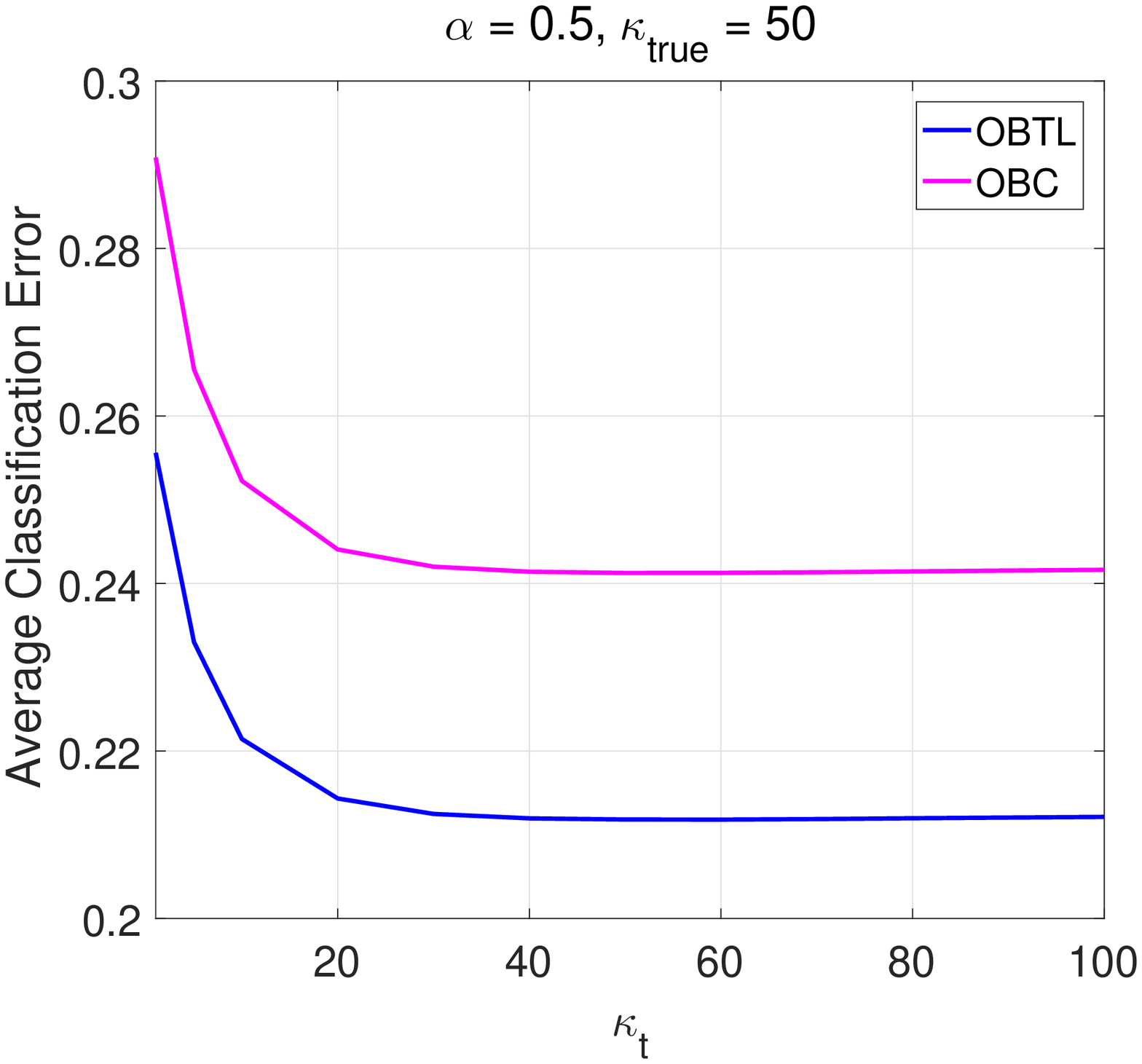}
        \caption{}
    \end{subfigure}
~ 
\begin{subfigure}[t]{0.23\textwidth}
        \centering
        \includegraphics[width=\textwidth]{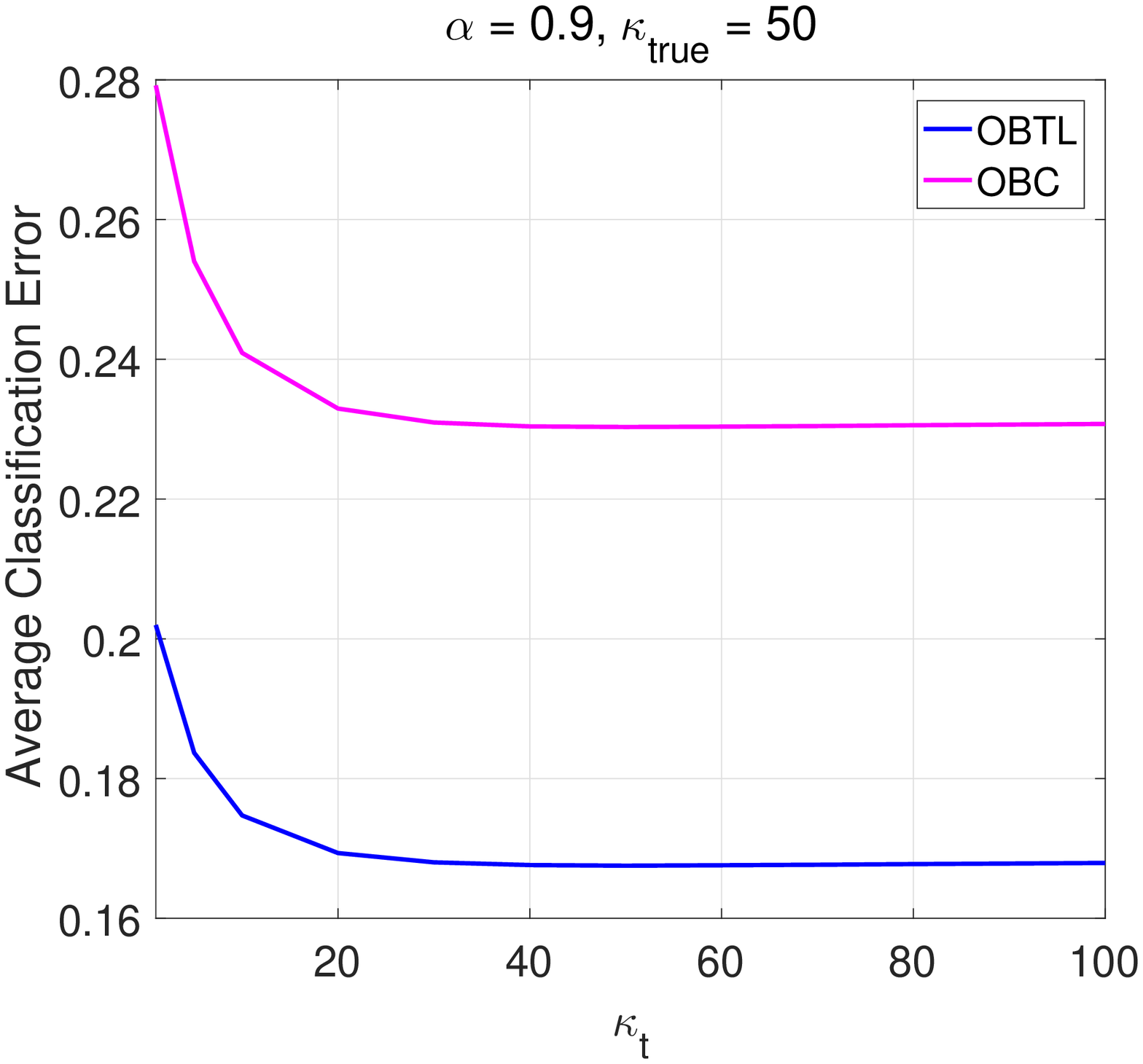}
        \caption{}
    \end{subfigure}

\caption{{\protect\footnotesize Average classification error vs $\protect%
\kappa_t$}}
\label{fig5}
\vspace{-.5cm}
\end{figure}

\subsection{Real-world benchmark datasets}

We test the OBTL classifier on \textit{Office} \cite{office} and \textit{%
Caltech256} \cite{caltech} image datasets, which have been adopted to help
benchmark different transfer learning algorithms in the literature. We have
used exactly the same evaluation setup and data splits of MMDT (Max-Margin
Domain Transform) \cite{hoffman2013}.

\begin{table*}[h!]
\caption{{\protect\footnotesize Semi-supervised accuracy for different
source and target domains in the \textit{Office+Caltech256} dataset using
SURF features. Domain names are denoted as a: \textit{amazon}, w: \textit{%
webcam}, d: \textit{dslr}, c: \textit{Caltech256}. The numbers in red show
the best accuracy and the numbers in blue show the second best accuracy in
each column. The results of the first six methods have been adopted from 
\protect\cite{ILS2017}. Similar to \protect\cite{ILS2017}, we have also used
the evaluation setup of \protect\cite{hoffman2013} for the OBTL.}}
\label{table-1}\centering
{\small \addtolength{\tabcolsep}{-3pt} %\begin{subtable}[t]{0.1\linewidth}
%\centering
\begin{tabular}{|c|c|c|c|c|c|c|c|c|c|c|c|c||c|}
\hline
$~$ & a $\rightarrow$ w & a $\rightarrow$ d & a $\rightarrow$ c & w $%
\rightarrow$ a & w $\rightarrow$ d & w $\rightarrow$ c & d $\rightarrow$ a & 
d $\rightarrow$ w & d $\rightarrow$ c & c $\rightarrow$ a & c $\rightarrow$ w
& c $\rightarrow$ d & Mean \\ \hline\hline
1-NN-t & 34.5 & 33.6 & 19.7 & 29.5 & 35.9 & 18.9 & 27.1 & 33.4 & 18.6 & 29.2
& 33.5 & 34.1 & 29.0 \\ \hline
SVM-t & 63.7 & 57.2 & 32.2 & 46.0 & 56.5 & 29.7 & 45.3 & 62.1 & 32.0 & 45.1
& 60.2 & 56.3 & 48.9 \\ \hline
HFA \cite{HFA2012} & 57.4 & 55.1 & 31.0 & \textbf{\textcolor{red}{56.5}} & 
56.5 & 29.0 & 42.9 & 60.5 & 30.9 & 43.8 & 58.1 & 55.6 & 48.1 \\ \hline
MMDT \cite{hoffman2013} & 64.6 & 56. 7 & 36.4 & 47.7 & 67.0 & 32.2 & 46.9 & 
74.1 & 34.1 & 49.4 & 63.8 & 56.5 & 52.5 \\ \hline
CDLS \cite{CDLS2016} & \textbf{\textcolor{blue}{68.7}} & \textbf{%
\textcolor{red}{60.4}} & 35.3 & 51.8 & 60.7 & 33.5 & 50.7 & 68.5 & 34.9 & 
50.9 & \textbf{\textcolor{blue}{66.3}} & \textbf{\textcolor{blue}{59.8}} & 
53.5 \\ \hline
ILS (1-NN) \cite{ILS2017} & 59.7 & 49.8 & \textbf{\textcolor{red}{43.6}} & 
54.3 & \textbf{\textcolor{blue}{70.8}} & \textbf{\textcolor{red}{38.6}} & 
\textbf{\textcolor{red}{55.0}} & \textbf{\textcolor{blue}{80.1}} & \textbf{%
\textcolor{red}{41.0}} & \textbf{\textcolor{red}{55.1}} & 62.9 & 56.2 & 
\textbf{\textcolor{blue}{55.6}} \\ \hline
\textbf{OBTL} & \textbf{\textcolor{red}{72.4}} & \textbf{%
\textcolor{blue}{60.2}} & \textbf{\textcolor{blue}{41.5}} & \textbf{%
\textcolor{blue}{55.0}} & \textbf{\textcolor{red}{75.0}} & \textbf{%
\textcolor{blue}{37.4}} & \textbf{\textcolor{blue}{54.4}} & \textbf{%
\textcolor{red}{83.2}} & \textbf{\textcolor{blue}{40.3}} & \textbf{%
\textcolor{blue}{54.8}} & \textbf{\textcolor{red}{71.1}} & \textbf{%
\textcolor{red}{61.5}} & \textbf{\textcolor{red}{58.9}} \\ \hline
\end{tabular}
%\end{subtable}
}
\end{table*}

\begin{table*}[!t]
\caption{{\protect\footnotesize The values of hyperparameter $\alpha$ of the OBTL
used in each experiment. $n_t$ and $n_s$ are based on the data splits
provided by \protect\cite{hoffman2013}.}}
\label{table-2}\centering
{\small \addtolength{\tabcolsep}{-3pt} %\begin{subtable}[t]{0.1\linewidth}
%\centering
\begin{tabular}{|c|c|c|c|c|c|c|c|c|c|c|c|c|}
\hline
$~$ & a $\rightarrow$ w & a $\rightarrow$ d & a $\rightarrow$ c & w $%
\rightarrow$ a & w $\rightarrow$ d & w $\rightarrow$ c & d $\rightarrow$ a & 
d $\rightarrow$ w & d $\rightarrow$ c & c $\rightarrow$ a & c $\rightarrow$ w
& c $\rightarrow$ d \\ \hline\hline
$n_t$ & 3 & 3 & 3 & 3 & 3 & 3 & 3 & 3 & 3 & 3 & 3 & 3 \\ \hline
$n_s$ & 20 & 20 & 20 & 8 & 8 & 8 & 8 & 8 & 8 & 8 & 8 & 8 \\ \hline
$\alpha$ & 0.6 & 0.75 & 0.99 & 0.9 & 0.99 & 0.99 & 0.9 & 0.99 & 0.99 & 0.85 & 
0.5 & 0.75 \\ \hline
\end{tabular}
%\end{subtable}
}
\vspace{-.2cm}
\end{table*}

\noindent $\bullet$ \textbf{Office dataset:} This dataset has images in
three different domains: \textit{amazon}, \textit{webcam}, and \textit{dslr}%
. The dataset contains 31 classes including the office stuff like backpack,
chair, keyboard, etc. The three domains \textit{amazon}, \textit{webcam},
and \textit{dslr} contain images from Amazon's website, a webcam, and a
digital single-lens reflex (dslr) camera, respectively, with different
lighting and backgrounds. SURF \cite{surf} image features are used in all
the domains, which are of dimension 800.

\noindent $\bullet$ \textbf{Office + Caltech256 dataset:} This dataset has $%
L=10$ common classes of both \textit{Office} and \textit{Caltech256}
datasets with the same feature dimension $d=800$. According to the data
splits of \cite{hoffman2013}, the numbers of training data per class in the
source domain are $n_s=20$ for \textit{amazon} and $n_s=8$ for the other
three domains, and in the target domain $n_t=3$ for all the four domains.
For this four-domain dataset, 20 random train-test splits have been created
by \cite{hoffman2013}. We run the OBTL classifier on that 20 provided
train-test splits and report the average accuracy. Note that the test data
are solely from the target domains. Authors of MMDT \cite{hoffman2013}
reduce the dimension to $d=20$ using PCA. We follow the same procedure for
the OBTL classifier.

Following the comparison framework of \cite{ILS2017}, which used the same
evaluation setup of \cite{hoffman2013}, we compare the OBTL's performance in
terms of accuracy (10-class) in Table 1 with two target-only classifiers and
four state-of-the-art semi-supervised transfer learning algorithms
(including \cite{ILS2017} itself). The evaluation setup is exactly the same
for the OBTL and all the other six methods. As a result, we use the results
of \cite{ILS2017} for the first six methods in Table 1 and compare them with
the OBTL classifier. The six methods are as follows.

\noindent $\bullet$ \textbf{1-NN-t and SVM-t:} The Nearest Neighbor (1-NN)
and linear SVM classifiers designed using only the target data.

\noindent $\bullet$ \textbf{HFA \cite{HFA2012}:} This Heterogeneous Feature
Augmentation (HFA) method learns a common latent space between source and
target domains using the max-margin approach and designs a classifier in
that common space.

\noindent $\bullet$ \textbf{MMDT \cite{hoffman2013}:} This Max-Margin Domain
Transform (MMDT) method learns a transformation between the source and
target domains and employs the weighted SVM for classification.

\noindent $\bullet$ \textbf{CDLS \cite{CDLS2016}:} This Cross-Domain
Landmark Selection (CDLS) is a semi-supervised heterogeneous domain
adaptation method, which derives a domain-invariant feature space for
improved classification performance.

\noindent $\bullet$ \textbf{ILS (1-NN) \cite{ILS2017}:} This is a recent
method that learns an Invariant Latent Space (ILS) to reduce the discrepancy
between the source and target domains and uses Riemannian optimization
techniques to match statistical properties between samples projected into
the latent space from different domains.

In Table \ref{table-1}, we have calculated the accuracy of the OBTL classifier in 12
distinct experiments, where the source-target pairs are different (source $%
\rightarrow$ target) in each experiment. We have marked the best accuracy in
each column with red and the second best accuracy with blue. We see that 
the OBTL classifier has either the best or second best accuracy in all the 12 experiments.
 We have written the mean accuracy of each method in the last
column, which has been averaged over all the 12 different experiments. The
OBTL classifier has the best mean accuracy and the ILS \cite{ILS2017} has
the second best accuracy among all the methods. We have assumed equal prior probabilities for
all the classes and used (\ref{OBTL_2}) for the OBTL classifier.

\noindent $\bullet$ \textbf{Hyperparameters of the OBTL:} We assume the same
values of hyperparameters for all the 10 classes in each domain, so we can
drop the superscript $l$ denoting the class label. We set $\nu=10d=200$ for
all the experiments. We choose $\alpha$ separately in each experiment since
the relatedness between distinct pairs of domains are different. For $%
\mathbf{m}_t$ and $\mathbf{m}_s$, we pool all the target and source data in
all the 10 classes, respectively, and use the sample means of the datasets.  We fix $\beta_t=\beta_s=1$ (meaning that $\kappa_t=n_t$ and $\kappa_s=n_s$) and $k_t=k_s=1/\nu=1/200=0.005$. The mean of the Wishart precision matrix $\Lambdab_z$, for $z \in \{s,t\}$, with scale matrix $\M_z$ and $\nu$ degrees of freedom is $\nu\M_z$. Consequently, $E(\mathbf{\Lambda}_t)=E(\mathbf{\Lambda}_s)=I_d$, which is a reasonable choice, since the provided datasets of \cite{hoffman2013} have been normally standardized. Therefore, the only hyperparameter and the most important one is $\alpha$ ($\in (0,1)$), which shows the relatedness between the two domains. Figs. \ref{fig-real}a and \ref{fig-real}b demonstrate that the accuracy is robust for $k_t\in (0.005, 0.02)$ and $k_s\in (0.005, 0.02)$, respectively. Figs. \ref{fig-real}a and \ref{fig-real}b are corresponding to two experiments: $a \rightarrow w, \alpha = 0.6$ and $w \rightarrow d, \alpha = 0.99$. Figs. \ref{fig-real}c and \ref{fig-real}d show interesting results. We have already seen similar behavior in the synthetic data as well. In the case of $a \rightarrow w$, accuracy grows smoothly by increasing $\alpha$, reaches the maximum at $\alpha=0.6$, and decreases afterwards. This verifies the fact that the source domain $a$ cannot help the target domain $w$ that much. On the contrary, accuracy increases monotonically in Fig. \ref{fig-real}d, in the case of $w \rightarrow d$, and the difference between accuracy for $\alpha=0.01$ and $\alpha=0.99$ is huge. This  confirms that the source domain $w$ is very related to the target domain $d$ and helps it a lot. Interestingly, this coincides with the findings from the literature that the two domains $w$ and $d$ are highly related. We choose the values of $\alpha$ in each experiment which give the best accuracy. They are shown in Table \ref{table-2}. The values of $\alpha$ in Table \ref{table-2} also reveal the amount of relatedness between any pairs of source and target domains. For example, both $w \rightarrow d$ and $d \rightarrow w$ have high relatedness with $\alpha=0.99$, which has already been verified by other papers as well \cite{gong2012geodesic}. 

%Our empirical Bayes method for estimation of the hyperparameters of the OBTL classifier will appear in upcoming paper due to the lack of space in this paper. 

\begin{figure}[t!]
\centering
\begin{subfigure}[t]{0.23\textwidth}
        \centering
        \includegraphics[width=\textwidth]{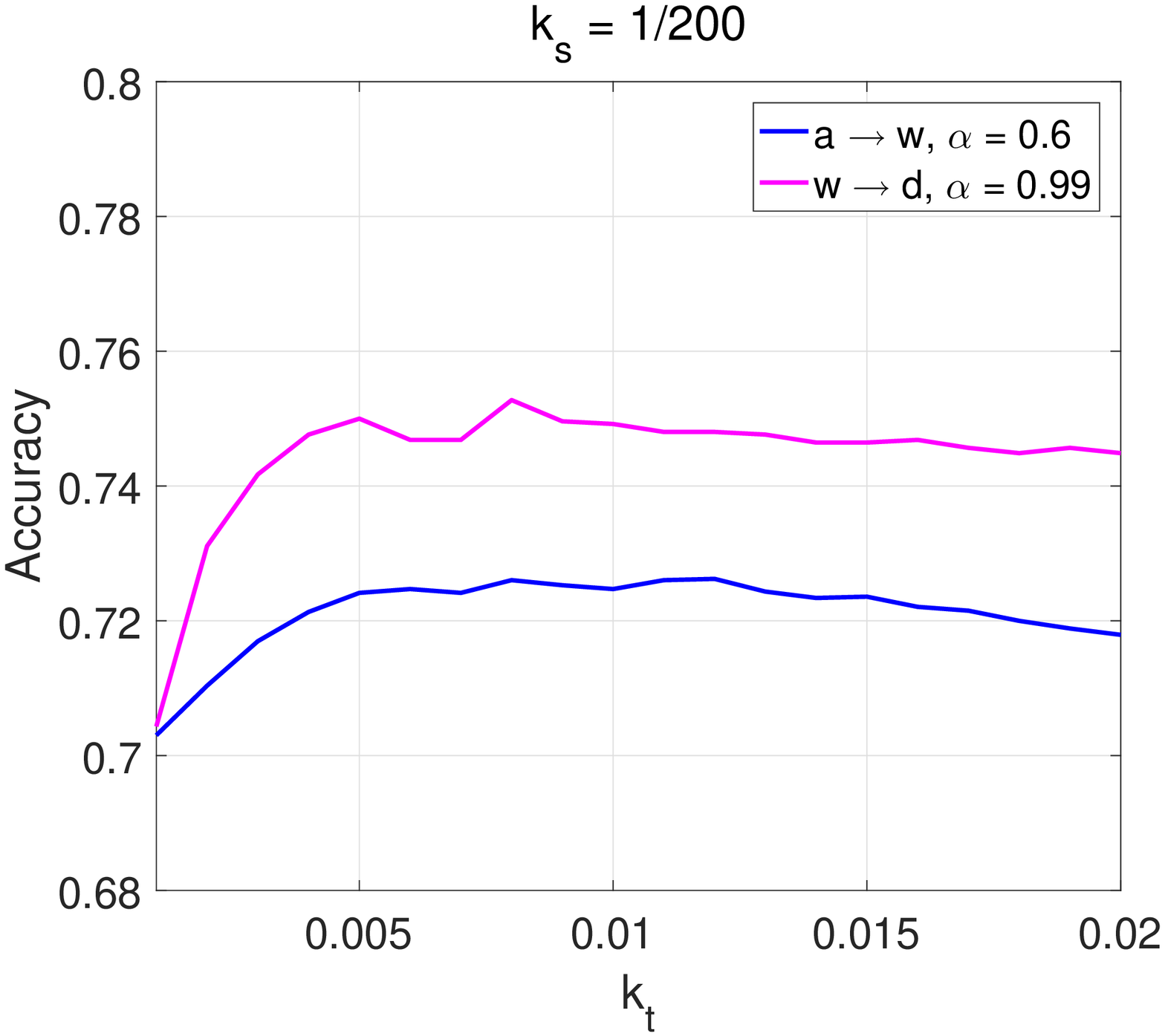}
        \caption{}
    \end{subfigure}
~ 
\begin{subfigure}[t]{0.23\textwidth}
        \centering
        \includegraphics[width=\textwidth]{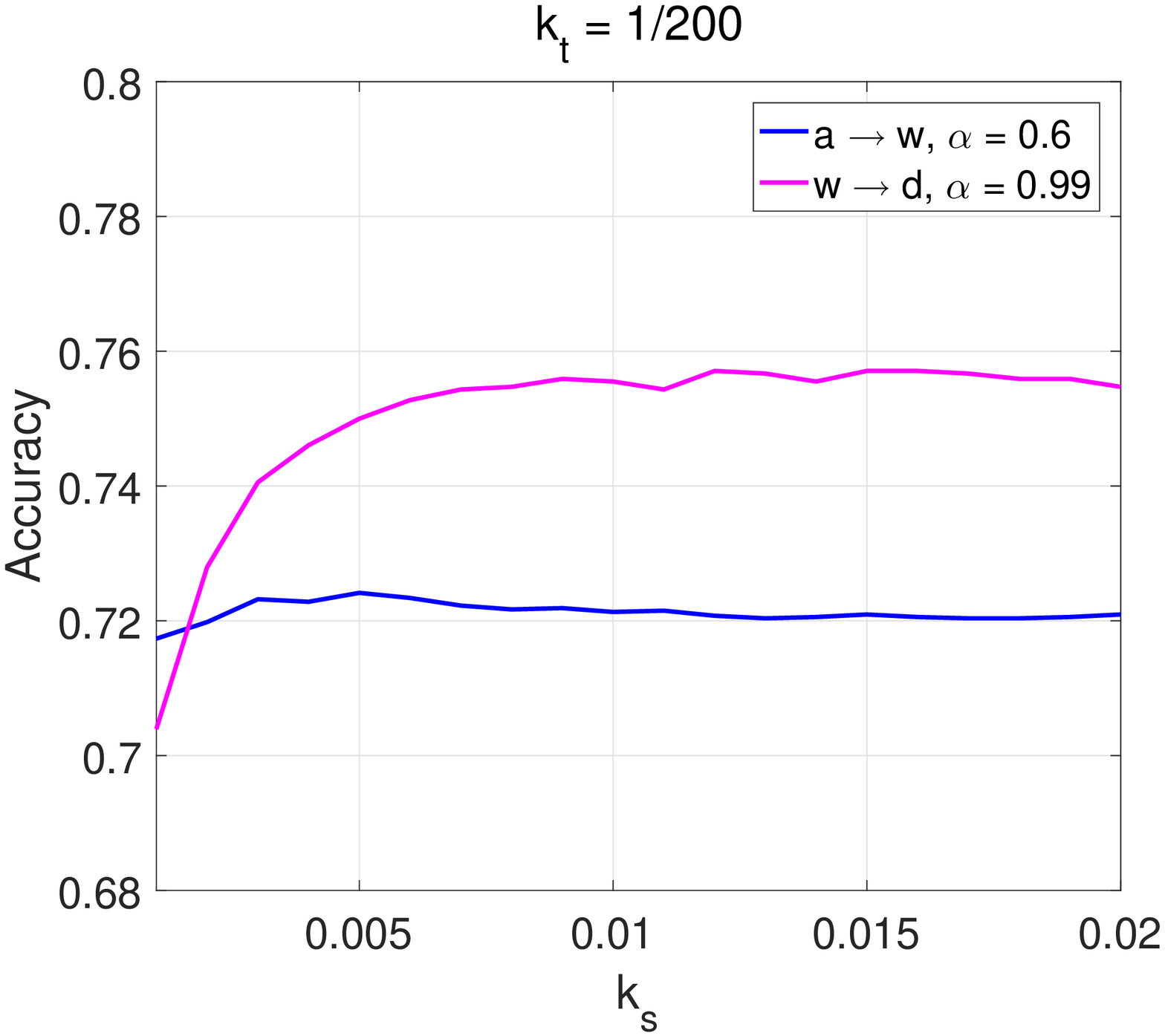}
        \caption{}
    \end{subfigure}
    \begin{subfigure}[t]{0.23\textwidth}
        \centering
        \includegraphics[width=\textwidth]{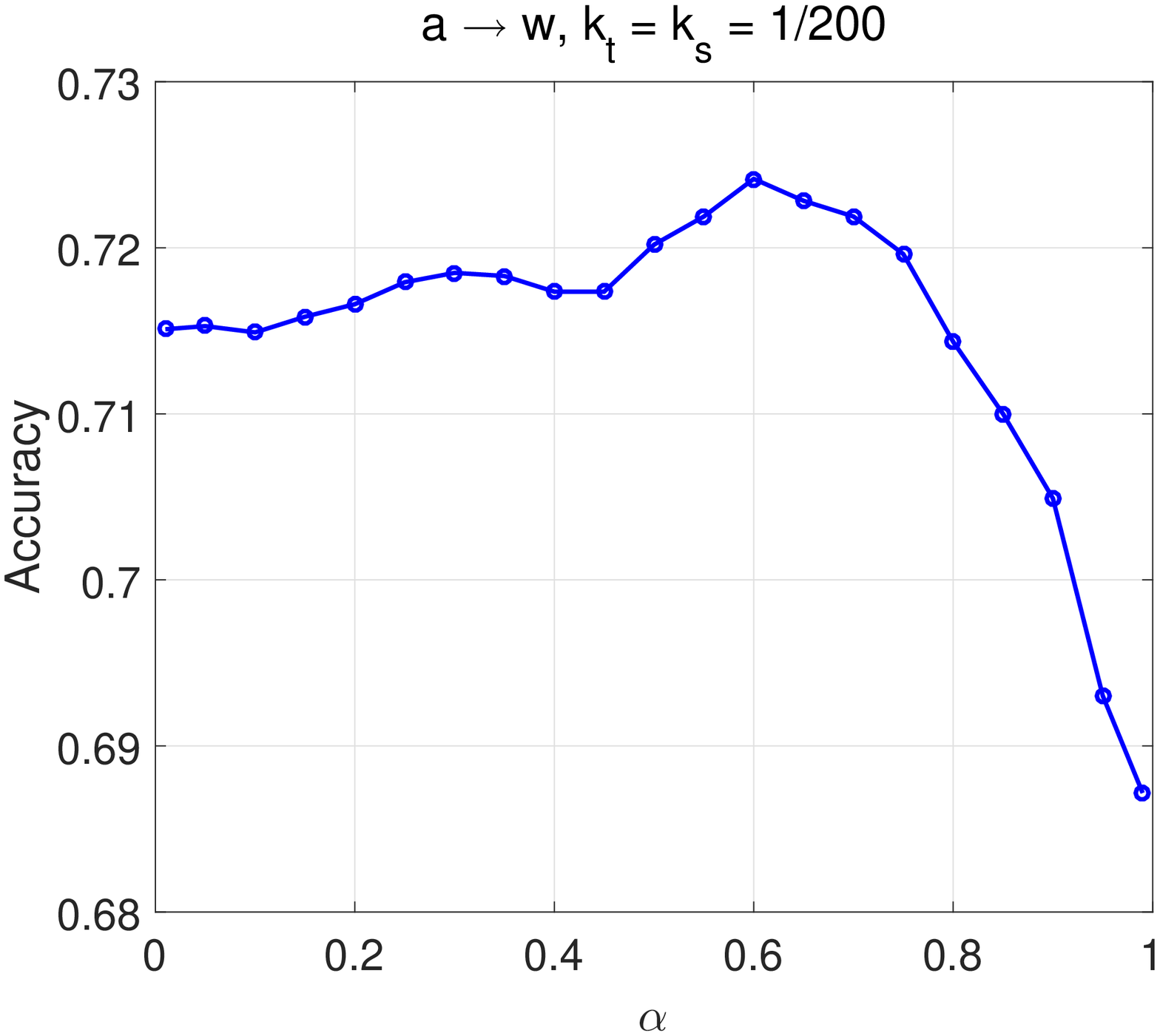}
        \caption{}
    \end{subfigure}
~ 
\begin{subfigure}[t]{0.23\textwidth}
        \centering
        \includegraphics[width=\textwidth]{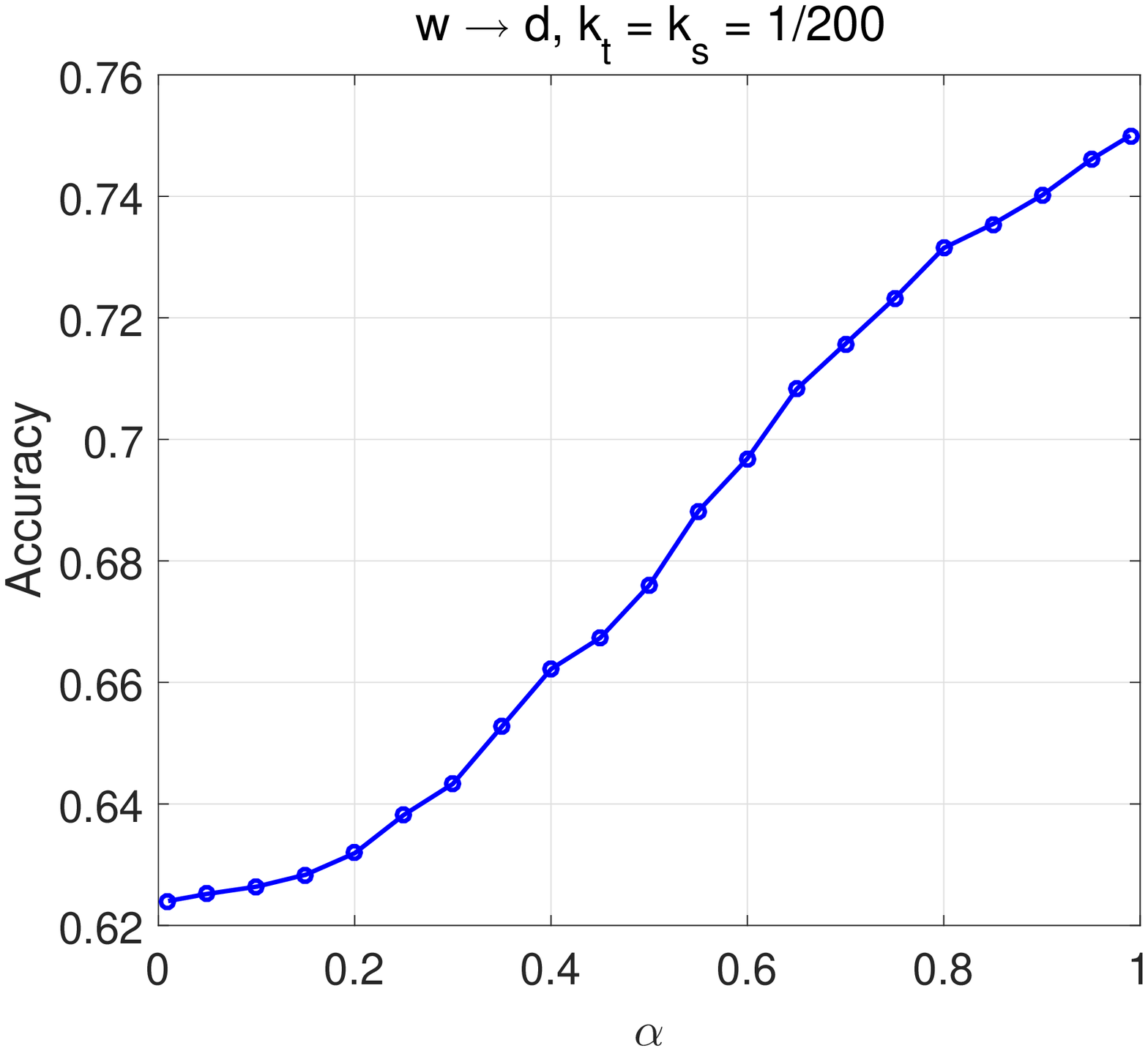}
        \caption{}
    \end{subfigure}
\caption{{\protect\footnotesize  Accuracy in the \textit{Office+Caltech256} dataset versus: (a) $k_t$ when $k_s=1/200$ and for two experiments $a \rightarrow w, \alpha = 0.6$ and $w \rightarrow d, \alpha = 0.99$, (b) $k_s$ when $k_t=1/200$ and for two experiments $a \rightarrow w, \alpha = 0.6$ and $w \rightarrow d, \alpha = 0.99$, (c) $\alpha$ when $k_t=k_s=1/200$ and for the experiment $a \rightarrow w$, (d) $\alpha$ when $k_t=k_s=1/200$ and for the experiment $w \rightarrow d$. }}
\label{fig-real}
\vspace{-.5cm}
\end{figure}

%In other words, if $\mathcal{D}_t$ and $\mathcal{D}_s$ are the source and target training labeled datasets, respectively, and $\mathcal{D}_t^u$ is the target test unlabeled dataset, we choose $\m_t=\mathrm{mean}(\mathcal{D}_t \cup \mathcal{D}_t^u)$ and $\m_s=\mathrm{mean}(\mathcal{D}_s)$.

\section{Conclusions And Future Work}
\label{sec8} 

We have constructed a Bayesian transfer learning framework to
tackle the supervised transfer learning problem. The
proposed Optimal Bayesian Transfer Learning (OBTL) classifier can deal with
the lack of labeled data in the target domain and is optimal in this new
Bayesian framework since it minimizes the expected classification error. We
have obtained the closed-form posterior distribution of the target
parameters and accordingly the closed-form effective class-conditional
densities in the target domain to define the OBTL classifier. As the OBTL's
objective function consists of hypergeometric functions of matrix argument,
we use the Laplace approximations of those functions to derive a
computationally efficient and scalable OBTL classifier, while preserving its
superior performance. We have compared the performance of the OBTL with its
target-only version, OBC, to see how transferring from source to target
domain can help. We have tested the OBTL classifier with real-world
benchmark image datasets and demonstrated its excellent performance compared
to other state-of-the-art domain adaption methods.

This paper considers a Gaussian model, in which we can derive closed-form
solutions, as the case with the OBC. Since many practical problems cannot be
approximated by a Gaussian model, an important aspect of OBC development has
been the utilization of MCMC methods \cite{knight2014mcmc,
knight2015detecting}. In a forthcoming paper, we extend the OBTL setting to count data with a Negative Binomial model, in which the inference of parameters is done by MCMC. We will also apply the OBTL in dynamical systems and time series scenarios \cite{Alireza_TCBB,Alireza_TCBB2,Alireza_ICASSP,Alireza_BMC}.

We have only considered two domains in this paper, assuming there is only
one source domain. Having seen the good performance of the OBTL classifier
in two domains, in future work, we are going to apply it to the multi-source
transfer learning problems, where we can benefit from the knowledge of
different related sources in order to further improve the target classifier.

As in the case of the OBC, a basic engineering aspect of the OBTL is prior
construction. This has been studied under different conditions in the
context of the OBC: using the data from unused features to infer a prior
distribution \cite{dalton2011application}, deriving the prior distribution
from models of the data-generating technology \cite{knight2014mcmc},
 and applying constraints based on prior
knowledge to map the prior knowledge into a prior distribution via
optimization \cite{Esfahani_TCBB_1,Esfahani_TCBB_2,Shahin_BMC}. The methods in  \cite{Esfahani_TCBB_1,Esfahani_TCBB_2} are very general and have been placed into a formal mathematical
structure in \cite{Shahin_BMC}, where the prior results from an optimization
involving the Kullback-Leibler (KL) divergence constrained by conditional
probability statements characterizing physical knowledge, such as genetic
pathways in genomic medicine. A key focus of our future work will be to
extend this general framework to the OBTL, which will require a formulation
that incorporates knowledge relating the source and target domains. It
should be emphasized that with optimal Bayesian classification, as well as
with optimal Bayesian filtering \cite{Lori-IBRF,Qian-OBF,Roozbeh_IBR_Kalman}%
, the prior distribution is not on the operator to be designed (classifier
or filter) but on the underlying scientific model (feature-label
distribution, covariance matrix, or observation model) for which the
operator is optimized. It is for this reason that uncertainty in the
scientific model can be mapped into a prior distribution based on physical
laws.

\appendices

\section{Theorems for Zonal Polynomials and Generalized Hypergeometric Functions of Matrix Argument}
\label{appendix:hypergeometric}

\begin{theorem}
\label{theorem5} \cite{muirhead}: Let $\mathbf{Z}$ be a complex symmetric
matrix whose real part is positive-definite, and let $\mathbf{X}$ be an
arbitrary complex symmetric matrix. Then 
\begin{equation}
\begin{aligned} \int_{\mathbf{R}>0} \mathrm{etr}(-\mathbf{Z}\mathbf{R})
|\mathbf{R}|^{\alpha-\frac{d+1}{2}} C_\kappa(\mathbf{R}\mathbf{X})
d\mathbf{R} \\ = \Gamma_d(\alpha)(\alpha)_\kappa |\mathbf{Z}|^{-\alpha}
C_\kappa(\mathbf{X}\mathbf{Z}^{-1}), \end{aligned}
\end{equation}
the integration being over the space of positive-definite $d\times d$
matrices, and valid for all complex numbers $\alpha$ satisfying $\mathrm{Re}%
(\alpha)>\frac{d-1}{2}$. $\Gamma_d(\alpha)$ is the multivariate gamma
function defined in (\ref{Gamma_multi}).
\end{theorem}

\begin{theorem}
\label{theorem6}  \cite{Appell}: The zonal polynomials are invariant under
orthogonal transformation. That is, for a $d \times d$ symmetric matrix $%
\mathbf{X}$, 
\begin{equation}
C_\kappa(\mathbf{X}) = C_\kappa(\mathbf{H}\mathbf{X}\mathbf{H}^{^{\prime }}),
\end{equation}
where $\mathbf{H}$ is an orthogonal matrix of order $d$. If $\mathbf{R}$ is
a symmetric positive-definite matrix of order $d$, then 
\begin{equation}
C_\kappa(\mathbf{RX})=C_\kappa(\mathbf{R}^{1/2}\mathbf{X} \mathbf{R}^{1/2}).
\end{equation}
\end{theorem}

As a result, if $\mathbf{R}$ is a symmetric positive-definite matrix, the
hypergeometric function has the following property: 
\begin{equation}
\begin{aligned} ~_pF_q(a_1,\cdots,a_p;b_1,\cdots,b_q;\mathbf{RX})
\hspace{3cm} \\
=~_pF_q(a_1,\cdots,a_p;b_1,\cdots,b_q;\mathbf{R}^{1/2}\mathbf{X}
\mathbf{R}^{1/2}). \end{aligned}
\end{equation}

\begin{theorem}
\label{theorem7}  \cite{gupta2016properties}: If $\mathbf{Z}>0$ and $\mathrm{%
Re}(\alpha)>\frac{d-1}{2}$, and $\mathbf{X}$ is a $d \times d$ symmetric
matrix, we have 
\begin{equation}
\begin{aligned} &\int_{\mathbf{R}>0} \mathrm{etr}(-\mathbf{ZR})
|\mathbf{R}|^{\alpha-\frac{d+1}{2}} \\ &\hspace{2cm} \times
~_pF_q(a_1,\cdots,a_p;b_1,\cdots,b_q;\mathbf{RX}) d \mathbf{R} \\ &=
\int_{\mathbf{R}>0} \mathrm{etr}(-\mathbf{ZR})
|\mathbf{R}|^{\alpha-\frac{d+1}{2}} \\ &\hspace{1.5cm} \times
~_pF_q(a_1,\cdots,a_p;b_1,\cdots,b_q;\mathbf{R}^{1/2}\mathbf{X}
\mathbf{R}^{1/2}) d \mathbf{R} \\ &= \Gamma_d(\alpha) |\mathbf{Z}|^{-\alpha}
~_{p+1}F_q(a_1,\cdots,a_p,\alpha;b_1,\cdots,b_q;\mathbf{X}\mathbf{Z}^{-1}).
\end{aligned}  \notag
\end{equation}
\end{theorem}

\section{Proof of Theorem \ref{thm:posterior}}
\label{appendix:posterior}

We require the following lemma.
\begin{lemma}
\label{Lemma1} \cite{muirhead} If $\mathcal{D}=\{\mathbf{x}_1,\cdots,\mathbf{x}_n\}$ where $%
\mathbf{x}_i$ is a $d \times 1$ vector and $\mathbf{x}_i \sim \mathcal{N}(%
\mathbf{\mu},(\mathbf{\Lambda})^{-1})$, for $i=1,\cdots,n$, and $(\mathbf{\mu%
},\mathbf{\Lambda})$ has a Gaussian-Wishart prior, such that, $\mathbf{\mu} |%
\mathbf{\Lambda} \sim \mathcal{N}(\mathbf{m},(\kappa \mathbf{\Lambda})^{-1})$
and $\mathbf{\Lambda} \sim W_d(\mathbf{M},\nu)$, then the posterior of $(%
\mathbf{\mu},\mathbf{\Lambda})$ upon observing $\mathcal{D}$ is also a
Gaussian-Wishart distribution: 
\begin{equation}
\begin{aligned} \label{lemma1} \mathbf{\mu} |\mathbf{\Lambda} , \mathcal{D}
&\sim \mathcal{N}(\mathbf{m}_n, (\kappa_n\mathbf{\Lambda})^{-1}), \\
\mathbf{\Lambda} | \mathcal{D} &\sim W_d(\mathbf{M}_n,\nu_n), \end{aligned}
\end{equation}
where 
\begin{equation}
\begin{aligned}
& \kappa_n = \kappa + n, \\
& \nu_n = \nu + n, \\
& \mathbf{m}_n = \frac{\kappa \m + n \bar{\x}}{\kappa + n}, \\ 
& \mathbf{M}_n^{-1} = \mathbf{M}^{-1} + \mathbf{S} + \frac{\kappa n}{\kappa + n} (\mathbf{m}
-\bar{\x})(\mathbf{m} -\bar{\x})^{'},
 \end{aligned}
\end{equation}
depending on the sample mean and covariance matrix 
\begin{equation}
\begin{aligned}
& \bar{\mathbf{x}} = \frac{1}{n} \sum_{i=1}^n \mathbf{x}_i, \\
& \mathbf{S} =
\sum_{i=1}^n (\mathbf{x}_i - \bar{\mathbf{x}}) (\mathbf{x}_i - \bar{\mathbf{x%
}}) ^{^{\prime }}.
\end{aligned}
\end{equation}
\end{lemma}
We now provide the proof. From (\ref{x_s_x_t}), for each domain $z\in\{s,t\}$, 
\begin{equation}
p(\mathcal{D}_z^l|\mathbf{\mu}_z^l,\mathbf{\Lambda}_z^l) = (2\pi)^{-\frac{d
n_z^l}{2}} ~ \left|\mathbf{\Lambda}_z^l\right|^{\frac{n_z^l}{2}} \exp \left(-%
\frac{1}{2} \mathbf{Q}_z^l \right),  \label{D}
\end{equation}
where $\mathbf{Q}_z^l = \sum_{i=1}^{n_z^l}\left(\mathbf{x}_{z,i}^l-\mathbf{%
\mu}_z^l \right)^{^{\prime }} \mathbf{\Lambda}_z^l \left(\mathbf{x}_{z,i}^l-%
\mathbf{\mu}_z^l \right) $. Moreover, from (\ref{mu_s_mu_t}), for each
domain $z\in\{s,t\}$,
\begin{eqnarray}
p\left(\mathbf{\mu}_z^l | \mathbf{\Lambda}_z^l\right) = (2\pi)^{-\frac{d}{2}%
} \left(\kappa_z^l\right)^{\frac{d}{2}} \left|\mathbf{\Lambda}_z^l\right|^{%
\frac{1}{2}} \hspace{2cm}  \notag \\
\times\exp \left(-\frac{\kappa_z^l}{2}\left(\mathbf{\mu}_z^l - \mathbf{m}%
_z^l\right)^{^{\prime }}\mathbf{\Lambda}_z^l \left(\mathbf{\mu}_z^l - 
\mathbf{m}_z^l\right) \right).  \label{mu}
\end{eqnarray}
From (\ref{joint}), (\ref{posterior5}), (\ref{D}), and (\ref{mu}),
\begin{equation}  \label{post_t}
\begin{aligned} &p(\mathbf{\mu}_t^l,\mathbf{\Lambda}_t^l
|\mathcal{D}_t^l,\mathcal{D}_s^l) \propto
\left|\mathbf{\Lambda}_t^l\right|^{\frac{n_t^l}{2}} \exp \left(-\frac{1}{2}
\mathbf{Q}_t^l \right) \left|\mathbf{\Lambda}_t^l\right|^{\frac{1}{2}} \\
&\times \exp \left(-\frac{\kappa_t^l}{2}\left(\mathbf{\mu}_t^l -
\mathbf{m}_t^l\right)^{'}\mathbf{\Lambda}_t^l \left(\mathbf{\mu}_t^l -
\mathbf{m}_t^l\right) \right) \\ &\times
\left|\mathbf{\Lambda}_{t}^l\right|^{\frac{\nu^l-d-1}{2}}
\mathrm{etr}\left(-\frac{1}{2} \left({\left(\mathbf{M}_{t}^l\right)}^{-1} +
{\mathbf{F}^l}^{'}\mathbf{C}^l
\mathbf{F}^l\right)\mathbf{\Lambda}_{t}^l\right) \\ &\times
\int_{\mub_s^l,\Lambdab_s^l} \left\lbrace
\left|\mathbf{\Lambda}_s^l\right|^{\frac{n_s^l}{2}} \exp \left(-\frac{1}{2}
\mathbf{Q}_s^l \right) \left|\mathbf{\Lambda}_s^l\right|^{\frac{1}{2}}
\right.\\ &\times \exp \left(-\frac{\kappa_s^l}{2}\left(\mathbf{\mu}_s^l -
\mathbf{m}_s^l\right)^{'}\mathbf{\Lambda}_s^l \left(\mu_s^l -
\mathbf{m}_s^l\right) \right) \\ &\times
\left|\mathbf{\Lambda}_{s}^l\right|^{\frac{\nu^l-d-1}{2}}
\mathrm{etr}\left(-\frac{1}{2} {\left(\mathbf{C}^l\right)}^{-1}
\mathbf{\Lambda}_{s}^l\right) \\ &\left. \times ~_0F_1\left(\frac{\nu^l}{2};
\frac{1}{4} {\mathbf{\Lambda}_{s}^l}^{\frac{1}{2}} \mathbf{F}^l
\mathbf{\Lambda}_{t}^l
{\mathbf{F}^l}^{'}{\mathbf{\Lambda}_{s}^l}^{\frac{1}{2}} \right)
\right\rbrace d\mathbf{\mu}_s^l d\mathbf{\Lambda}_s^l. \end{aligned}
\end{equation}
Using Lemma \ref{Lemma1} we can simplify (\ref{post_t}) as 
\begin{equation}  \label{prop2}
\begin{aligned} &p(\mathbf{\mu}_t^l,\mathbf{\Lambda}_t^l
|\mathcal{D}_t^l,\mathcal{D}_s^l) \\ &\propto
\left|\mathbf{\Lambda}_t^l\right|^{\frac{1}{2}} \exp
\left(-\frac{\kappa_{t,n}^l}{2}\left(\mathbf{\mu}_t^l -
\mathbf{m}_{t,n}^l\right)^{'}\mathbf{\Lambda}_t^l \left(\mathbf{\mu}_t^l -
\mathbf{m}_{t,n}^l\right) \right) \\ &\times
\left|\mathbf{\Lambda}_{t}^l\right|^{\frac{\nu^l + n_t^l -d-1}{2}}
\mathrm{etr}\left(-\frac{1}{2}
{\left(\mathbf{T}_t^l\right)}^{-1}\mathbf{\Lambda}_{t}^l\right) \\ &
\int_{\mub_s^l,\Lambdab_s^l} \left\lbrace
\left|\mathbf{\Lambda}_s^l\right|^{\frac{1}{2}} \exp
\left(-\frac{\kappa_{s,n}^l}{2}\left(\mathbf{\mu}_s^l -
\mathbf{m}_{s,n}^l\right)^{'}\mathbf{\Lambda}_s^l \left(\mathbf{\mu}_s^l -
\mathbf{m}_{s,n}^l\right) \right) \right. \\ & \times
\left|\mathbf{\Lambda}_{s}^l\right|^{\frac{\nu^l + n_s^l -d-1}{2}}
\mathrm{etr}\left(-\frac{1}{2}
{\left(\mathbf{T}_s^l\right)}^{-1}\mathbf{\Lambda}_{s}^l\right) \\ &\left.
\times ~_0F_1\left(\frac{\nu^l}{2}; \frac{1}{4}
{\mathbf{\Lambda}_{s}^l}^{\frac{1}{2}} \mathbf{F}^l \mathbf{\Lambda}_{t}^l
{\mathbf{F}^l}^{'}{\mathbf{\Lambda}_{s}^l}^{\frac{1}{2}} \right)
\right\rbrace d\mathbf{\mu}_s^l d\mathbf{\Lambda}_s^l, \end{aligned}
\end{equation}
where 
\begin{equation}  \label{app:const1}
\begin{aligned} &\kappa_{t,n}^l = \kappa_t^l + n_t^l, \hspace{1.5cm}
\kappa_{s,n}^l = \kappa_s^l + n_s^l, \\ &\mathbf{m}_{t,n}^l =
\frac{\kappa_t^l \m_t^l + n_t^l \bar{\x}_t^l}{\kappa_t^l+n_t^l}, ~~~~~
\mathbf{m}_{s,n}^l = \frac{\kappa_s^l \m_s^l + n_s^l
\bar{\x}_s^l}{\kappa_s^l+n_s^l}, \\ &{\left(\mathbf{T}_t^l\right)}^{-1} =
{\left(\mathbf{M}_{t}^l\right)}^{-1} + {\mathbf{F}^l}^{'}\mathbf{C}^l
\mathbf{F}^l + \mathbf{S}_t^l \\ & \hspace{2cm} + \frac{\kappa_t^l
n_t^l}{\kappa_t^l + n_t^l} (\mathbf{m}_t^l -\bar{\x}_t^l)(\mathbf{m}_t^l
-\bar{\x}_t^l)^{'}, \\ &{\left(\mathbf{T}_s^l\right)}^{-1} =
{\left(\mathbf{C}^l\right)}^{-1} + \mathbf{S}_s^l + \frac{\kappa_s^l
n_s^l}{\kappa_s^l + n_s^l} (\mathbf{m}_s^l -\bar{\x}_s^l)(\mathbf{m}_s^l
-\bar{\x}_s^l)^{'}, \end{aligned}
\end{equation}
with sample means and covariances for $z\in\{s,t\}$ as 
\begin{equation}  \label{const2}
\bar{\mathbf{x}}_z^l = \frac{1}{n_z^l} \sum_{i=1}^{n_z^l} \mathbf{x}%
_{z,i}^l, ~~~ \mathbf{S}_z^l = \sum_{i=1}^{n_z^l} \left(\mathbf{x}_{z,i}^l - 
\bar{\mathbf{x}}_z^l \right)\left(\mathbf{x}_{z,i}^l - \bar{\mathbf{x}}_z^l
\right)^{^{\prime }}.  \notag
\end{equation}
Using the equation 
\begin{equation}  \label{int_mu}
\int_{\mathbf{x}} \exp \left(-\frac{1}{2}(\mathbf{x}-\mathbf{\mu})^{^{\prime
}}\mathbf{\Lambda} (\mathbf{x}-\mathbf{\mu}) \right) d\mathbf{x} = (2\pi)^{%
\frac{d}{2}} |\mathbf{\Lambda}|^{-\frac{1}{2}},
\end{equation}
and integrating out $\mathbf{\mu}_s^l$ in (\ref{prop2}) yields 
\begin{equation}  \label{prop3}
\begin{aligned} &p(\mathbf{\mu}_t^l,\mathbf{\Lambda}_t^l
|\mathcal{D}_t^l,\mathcal{D}_s^l) \\ &\propto
\left|\mathbf{\Lambda}_t^l\right|^{\frac{1}{2}} \exp
\left(-\frac{\kappa_{t,n}^l}{2}\left(\mathbf{\mu}_t^l -
\mathbf{m}_{t,n}^l\right)^{'}\mathbf{\Lambda}_t^l \left(\mathbf{\mu}_t^l -
\mathbf{m}_{t,n}^l\right) \right) \\ &\times
\left|\mathbf{\Lambda}_{t}^l\right|^{\frac{\nu^l + n_t^l -d-1}{2}}
\mathrm{etr}\left(-\frac{1}{2}
{\left(\mathbf{T}_t^l\right)}^{-1}\mathbf{\Lambda}_{t}^l\right) \\ &\times
\int_{\Lambdab_s^l} \left\lbrace
\left|\mathbf{\Lambda}_{s}^l\right|^{\frac{\nu^l + n_s^l -d-1}{2}}
\mathrm{etr}\left(-\frac{1}{2}
{\left(\mathbf{T}_s^l\right)}^{-1}\mathbf{\Lambda}_{s}^l\right) \right. \\ &
\hspace{1.5cm} \times \left. _0F_1\left(\frac{\nu^l}{2}; \frac{1}{4}
{\mathbf{\Lambda}_{s}^l}^{\frac{1}{2}} \mathbf{F}^l \mathbf{\Lambda}_{t}^l
{\mathbf{F}^l}^{'}{\mathbf{\Lambda}_{s}^l}^{\frac{1}{2}} \right)
\right\rbrace d\mathbf{\Lambda}_s^l. \end{aligned}
\end{equation}
The integral, $I$, in (\ref{prop3}) can be done using Theorem \ref{theorem7}
as 
\begin{equation}
\begin{aligned} & I = \Gamma_d\left(\frac{\nu^l + n_s^l}{2}\right) \\
&\times \left| 2\mathbf{T}_s^l\right|^{\frac{\nu^l + n_s^l}{2}}
~_1F_1\left(\frac{\nu^l + n_s^l}{2}; \frac{\nu^l}{2}; \frac{1}{2}
\mathbf{F}^l \mathbf{\Lambda}_{t}^l {\mathbf{F}^l}^{'} \mathbf{T}_s^l
\right), \end{aligned}
\end{equation}
where $_1F_1(a;b;\mathbf{X})$ is the Confluent hypergeometric function with
the matrix argument $\mathbf{X}$. As a result, (\ref{prop3}) becomes 
\begin{equation}  \label{app:prop4}
\begin{aligned} &p(\mathbf{\mu}_t^l,\mathbf{\Lambda}_t^l
|\mathcal{D}_t^l,\mathcal{D}_s^l) = \\ & A^l
\left|\mathbf{\Lambda}_t^l\right|^{\frac{1}{2}} \exp
\left(-\frac{\kappa_{t,n}^l}{2}\left(\mathbf{\mu}_t^l -
\mathbf{m}_{t,n}^l\right)^{'}\mathbf{\Lambda}_t^l \left(\mathbf{\mu}_t^l -
\mathbf{m}_{t,n}^l\right) \right) \\ &\times
\left|\mathbf{\Lambda}_{t}^l\right|^{\frac{\nu^l + n_t^l -d-1}{2}}
\mathrm{etr}\left(-\frac{1}{2}
{\left(\mathbf{T}_t^l\right)}^{-1}\mathbf{\Lambda}_{t}^l\right) \\ & \times
~_1F_1\left(\frac{\nu^l + n_s^l}{2}; \frac{\nu^l}{2}; \frac{1}{2}
\mathbf{F}^l \mathbf{\Lambda}_{t}^l {\mathbf{F}^l}^{'} \mathbf{T}_s^l
\right), \end{aligned}
\end{equation}
where the constant of proportionality, $A^l$, makes the integration of the
posterior $p(\mathbf{\mu}_t^l,\mathbf{\Lambda}_t^l |\mathcal{D}_t^l,\mathcal{%
D}_s^l)$ with respect to $\mathbf{\mu}_t^l$ and $\mathbf{\Lambda}_t^l$ equal
to one. Hence, 
\begin{equation}  \label{A1}
\begin{aligned} 
&{\left(A^l\right)}^{-1} = \int_{\mathbf{\Lambda}_t^l} \left|\mathbf{\Lambda}%
_{t}^l\right|^{\frac{\nu^l + n_t^l -d-1}{2}} \mathrm{etr}\left(-\frac{1}{2} {%
\left(\mathbf{T}_t^l\right)}^{-1}\mathbf{\Lambda}_{t}^l\right) \left|\mathbf{%
\Lambda}_t^l\right|^{\frac{1}{2}}  \\
& \times \int_{\mub_t^l} \exp
\left(-\frac{\kappa_{t,n}^l}{2}\left(\mathbf{\mu}_t^l -
\mathbf{m}_{t,n}^l\right)^{'}\mathbf{\Lambda}_t^l \left(\mathbf{\mu}_t^l -
\mathbf{m}_{t,n}^l\right) \right) d\mathbf{\mu}_t^l \\ &\times
_1F_1\left(\frac{\nu^l + n_s^l}{2}; \frac{\nu^l}{2}; \frac{1}{2}
\mathbf{F}^l \mathbf{\Lambda}_{t}^l {\mathbf{F}^l}^{'} \mathbf{T}_s^l
\right)d\mathbf{\Lambda}_t^l. \end{aligned}
\end{equation}
Using (\ref{int_mu}), the inner integral equals to $(2\pi)^{\frac{d}{2}}
|\kappa_{t,n}^l \mathbf{\Lambda}_t^l|^{-\frac{1}{2}}=\left(\frac{2\pi}{%
\kappa_{t,n}^l}\right)^{\frac{d}{2}} |\mathbf{\Lambda}_t^l|^{-\frac{1}{2}}$.
Hence, 
\begin{equation}  \label{A2}
\begin{aligned} {\left(A^l\right)}^{-1} =
\left(\frac{2\pi}{\kappa_{t,n}^l}\right)^{\frac{d}{2}} \int_{\Lambdab_t^l}
\left|\mathbf{\Lambda}_{t}^l\right|^{\frac{\nu^l + n_t^l -d-1}{2}}
\mathrm{etr}\left(-\frac{1}{2}
{\left(\mathbf{T}_t^l\right)}^{-1}\mathbf{\Lambda}_{t}^l\right) \\ \times
~_1F_1\left(\frac{\nu^l + n_s^l}{2}; \frac{\nu^l}{2}; \frac{1}{2}
\mathbf{F}^l \mathbf{\Lambda}_{t}^l {\mathbf{F}^l}^{'} \mathbf{T}_s^l
\right)d\mathbf{\Lambda}_t^l. \hspace{1cm} \end{aligned}
\end{equation}
With the variable change $\Omega = \mathbf{F}^l \mathbf{\Lambda}_{t}^l {%
\mathbf{F}^l}^{^{\prime }}$, we have $d\Omega=|\mathbf{F}^l|^{d+1} d\mathbf{%
\Lambda}_t^l$ and $\mathbf{\Lambda}_t^l = {\left(\mathbf{F}^l\right)}^{-1}
\Omega \left({\mathbf{F}^{l}}^{^{\prime }}\right)^{-1}$. Since $\mathrm{tr}(%
\mathbf{ABCD})=\mathrm{tr}(\mathbf{BCDA})=\mathrm{tr}(\mathbf{CDAB})=\mathrm{%
tr}(\mathbf{DABC})$ and $|\mathbf{ABC}|=|\mathbf{A}||\mathbf{B}||\mathbf{C}|$%
, $A^l$ can be derived as 
\begin{equation}  \label{app:A4}
\begin{aligned} &{\left(A^l\right)}^{-1} =
\left(\frac{2\pi}{\kappa_{t,n}^l}\right)^{\frac{d}{2}}
|\mathbf{F}^l|^{-\left(\nu^l + n_t^l \right)} \int_{\Omega} \left\lbrace
|\Omega |^{\frac{\nu^l + n_t^l -d-1}{2}} \right. \\ & \hspace{.5cm}
\times\mathrm{etr}\left(-\frac{1}{2}
{\left({\mathbf{F}^{l}}^{'}\right)}^{-1} {\left(\mathbf{T}_t^l\right)}^{-1}
{\mathbf{F}^l}^{-1} \Omega \right) \\ &\left.\hspace{.5cm} \times
~_1F_1\left(\frac{\nu^l + n_s^l}{2}; \frac{\nu^l}{2}; \frac{1}{2} \Omega
\mathbf{T}_s^l \right) \right\rbrace d\Omega \\ &=
\left(\frac{2\pi}{\kappa_{t,n}^l}\right)^{\frac{d}{2}}
2^{\frac{d\left(\nu^l+n_t^l \right)}{2}} \Gamma_d
\left(\frac{\nu^l+n_t^l}{2} \right) \left|\mathbf{T}_t^l\right|^{\frac{\nu^l
+ n_t^l}{2}} \\ & ~~~~~~ \times ~_2F_1\left(\frac{\nu^l + n_s^l}{2},
\frac{\nu^l + n_t^l}{2}; \frac{\nu^l}{2}; \mathbf{T}_s^l\mathbf{F}^l
\mathbf{T}_t^l {\mathbf{F}^l}^{'} \right), \end{aligned}
\end{equation}
where the second equality follows from Theorem \ref{theorem7}, and $%
_2F_1(a,b;c;\mathbf{X})$ is the Gauss hypergeometric function with the
matrix argument $\mathbf{X}$. As such, we have derived the closed-form
posterior distribution of the target parameters $(\mathbf{\mu}_t^l,\mathbf{%
\Lambda}_t^l)$ in (\ref{prop4}), where ${A^l}$ is given by (\ref{A4}).

\section{Proof of Theorem \ref{thm-effective}}
\label{appendix:effective}

The likelihood $p(\mathbf{x}|\mathbf{%
\mu }_{t}^{l},\mathbf{\Lambda }_{t}^{l})$ and posterior $p(\mathbf{\mu }%
_{t}^{l},\mathbf{\Lambda }_{t}^{l}|\mathcal{D}_{t}^{l},\mathcal{D}_{s}^{l})$
are given in (\ref{x_s_x_t}) and (\ref{prop4}), respectively. Hence, 
\begin{equation}
\begin{aligned} & p(\mathbf{x} | l) = (2\pi)^{-\frac{d}{2}} A^l
\int_{\mub_t^l,\Lambdab_t^l} \left\lbrace
|\mathbf{\Lambda}_t^l|^{\frac{1}{2}} \right. \\ & \times
\exp\left(-\frac{1}{2} \left(\mathbf{x}-\mathbf{\mu}_t^l
\right)^{'}\mathbf{\Lambda}_t^l \left(\mathbf{x}-\mathbf{\mu}_t^l \right)
\right) \\ & \times \left|\mathbf{\Lambda}_t^l\right|^{\frac{1}{2}} \exp
\left(-\frac{\kappa_{t,n}^l}{2}\left(\mathbf{\mu}_t^l -
\mathbf{m}_{t,n}^l\right)^{'}\mathbf{\Lambda}_t^l \left(\mathbf{\mu}_t^l -
\mathbf{m}_{t,n}^l\right) \right) \\ & \times
\left|\mathbf{\Lambda}_{t}^l\right|^{\frac{\nu^l + n_t^l -d-1}{2}}
\mathrm{etr}\left(-\frac{1}{2}
{\left(\mathbf{T}_t^l\right)}^{-1}\mathbf{\Lambda}_{t}^l\right) \\ & \left.
\times ~_1F_1\left(\frac{\nu^l + n_s^l}{2}; \frac{\nu^l}{2}; \frac{1}{2}
\mathbf{F}^l \mathbf{\Lambda}_{t}^l {\mathbf{F}^l}^{'} \mathbf{T}_s^l
\right) \right\rbrace d\mathbf{\mu}_t^l d\mathbf{\Lambda}_t^l. \end{aligned}
\label{eff1}
\end{equation}%
Similarly, we can simplify (\ref{eff1}) as 
\begin{equation}
\begin{aligned} & p(\mathbf{x}| l) = (2\pi)^{-\frac{d}{2}} A^l
\int_{\mub_t^l,\Lambdab_t^l} \left\lbrace
|\mathbf{\Lambda}_t^l|^{\frac{1}{2}} \right. \\ & \times
\exp\left(-\frac{\kappa_\x^l}{2} \left(\mathbf{\mu}_t^l-\mathbf{m}_\x^l
\right)^{'}\mathbf{\Lambda}_t^l \left(\mathbf{\mu}_t^l-\mathbf{m}_\x^l
\right) \right) \\ & \times \left|\mathbf{\Lambda}_{t}^l\right|^{\frac{\nu^l
+ n_t^l +1 -d-1}{2}} \mathrm{etr}\left(-\frac{1}{2}
{\left(\mathbf{T}_\x^l\right)}^{-1}\mathbf{\Lambda}_{t}^l\right) \\ & \left.
\times ~_{1}F_{1}\left( \frac{\nu ^{l}+n_{s}^{l}}{2};\frac{\nu
^{l}}{2};\frac{1}{2}\mathbf{F}^{l}\mathbf{\Lambda
}_{t}^{l}{\mathbf{F}^{l}}^{^{\prime }}\mathbf{T}_{s}^{l}\right) \right\}
d\mathbf{\mu }_{t}^{l}d\mathbf{\Lambda }_{t}^{l}, \end{aligned}  \label{eff2}
\end{equation}
where 
\begin{equation}
\begin{aligned} & \kappa_\x^l = \kappa_{t,n}^l + 1 = \kappa_t^l + n_t^l + 1,
~~~~~ \mathbf{m}_\x^l = \frac{\kappa_{t,n}^l \m_{t,n}^l +
\x}{\kappa_{t,n}+1}, \\ & {\left(\mathbf{T}_\x^l\right)}^{-1} =
{\left(\mathbf{T}_t^l\right)}^{-1} + \frac{\kappa_{t,n}^l}{\kappa_{t,n}^l +
1} \left(\mathbf{m}_{t,n}^l-\mathbf{x} \right)
\left(\mathbf{m}_{t,n}^l-\mathbf{x} \right)^{'}. \end{aligned}
\label{app:update_1}
\end{equation}%
The integration in (\ref{eff2}) is similar to the one in (\ref{A1}). As a
result, using (\ref{A4}), 
\begin{equation}
\begin{aligned} & p(\mathbf{x}| l) = (2\pi)^{-\frac{d}{2}} A^l
\left(\frac{2\pi}{\kappa_\x^l}\right)^{\frac{d}{2}}
2^{\frac{d\left(\nu^l+n_t^l + 1\right)}{2}} \Gamma_d \left(\frac{\nu^l+n_t^l
+ 1}{2} \right) \\ & \left|\mathbf{T}_\x^l\right|^{\frac{\nu^l + n_t^l +
1}{2}} ~_2F_1\left(\frac{\nu^l + n_s^l}{2}, \frac{\nu^l + n_t^l + 1}{2};
\frac{\nu^l}{2}; \mathbf{T}_s^l\mathbf{F}^l \mathbf{T}_\x^l
{\mathbf{F}^l}^{'} \right). \end{aligned}  \label{eff3}
\end{equation}%
By replacing the value of $A^{l}$, we have the effective class-conditional
density. We denote $O_{\mathrm{OBTL}}(\mathbf{x}|l)=p(\mathbf{x}|l)$, since
it is the objective function for the OBTL classifier. As such, 
\begin{equation}
\begin{aligned} &O_{\mathrm{OBTL}}(\mathbf{x}| l) = \pi^{-\frac{d}{2}}
\left(\frac{\kappa_{t,n}^l}{\kappa_\x^l} \right)^{\frac{d}{2}} \Gamma_d
\left(\frac{\nu^l+n_t^l + 1}{2} \right) \\ & \times \Gamma_d^{-1}
\left(\frac{\nu^l+n_t^l}{2} \right)
\left|\mathbf{T}_\x^l\right|^{\frac{\nu^l + n_t^l + 1}{2}}
\left|\mathbf{T}_t^l\right|^{-\frac{\nu^l + n_t^l}{2}} \\ & \times
~_2F_1\left(\frac{\nu^l + n_s^l}{2}, \frac{\nu^l + n_t^l + 1}{2};
\frac{\nu^l}{2}; \mathbf{T}_s^l\mathbf{F}^l \mathbf{T}_\x^l
{\mathbf{F}^l}^{'} \right) \\ & \times ~_2F_1^{-1}\left(\frac{\nu^l +
n_s^l}{2}, \frac{\nu^l + n_t^l}{2}; \frac{\nu^l}{2};
\mathbf{T}_s^l\mathbf{F}^l \mathbf{T}_t^l {\mathbf{F}^l}^{'} \right).
\end{aligned}  \label{app:eff4}
\end{equation}

\section{Laplace Approximation of the Gauss Hypergeometric Function of Matrix Argument}
\label{appendix:Laplace}
The Gauss hypergeomeric function has the following integral representation: 
\begin{equation}
\begin{aligned} &~_2F_1(a,b;c;\mathbf{X})= B_d^{-1}(a,c-a) \\ & \times
\int_{0_d<\Y<\I_d} |\mathbf{Y}|^{a-\frac{d+1}{2}} |\mathbf{I}_d
-\mathbf{Y}|^{c-a-\frac{d+1}{2}} |\mathbf{I}_d - \mathbf{X}\mathbf{Y}|^{-b}
d\mathbf{Y}, \end{aligned}  \label{int_rep}
\end{equation}%
which is valid under the following conditions: $\mathbf{X}\in \mathbf{C}%
^{d\times d}$ is symmetric and satisfies $\mathrm{Re}(\mathbf{X})<\mathbf{I}%
_{d}$, $\mathrm{Re}(a)>\frac{d-1}{2}$, and $\mathrm{Re}(c-a)>\frac{d-1}{2}$. 
$B_{d}(\alpha ,\beta )$ is the multivariate beta function 
\begin{equation}
B_{d}(\alpha ,\beta )=\frac{\Gamma _{d}(\alpha )\Gamma _{d}(\beta )}{\Gamma
_{d}(\alpha +\beta )},
\end{equation}%
where $\Gamma _{d}(\alpha )$ is the multivariate gamma function defined in (%
\ref{Gamma_multi}). The Laplace approximation is one common solution to
approximate the integral 
\begin{equation}
I=\int_{y\in D}h(y)\exp (-\lambda g(y))dy,  \label{laplace}
\end{equation}%
where $D\subseteq \mathbf{R}^{d}$ is an open set and $\lambda $ is a real
parameter. If $g(\lambda )$ has a unique minimum over $D$ at point $\hat{y}%
\in D$, then the Laplace approximation to $I$ is given by 
\begin{equation}
\tilde{I}=(2\pi )^{\frac{d}{2}}\lambda ^{-\frac{d}{2}}|g^{^{\prime \prime }}(%
\hat{y})|^{-\frac{1}{2}}h(\hat{y})\exp (-\lambda g(\hat{y})),
\label{laplace2}
\end{equation}%
where $g^{^{\prime \prime }}(y)=\frac{\partial ^{2}g(y)}{\partial y\partial
y^{T}}$ is the Hessian of $g(y)$. The hypergeometric function $%
~_{2}F_{1}(a,b;c;\mathbf{X})$ depends only on the eigenvalues of the
symmetric matrix $\mathbf{X}$. Hence, without loss of generality, it is
assumed that $\mathbf{X}=\mathrm{diag}\{x_{1},\cdots ,x_{d}\}$. The
following $g$ and $h$ functions are used for (\ref{int_rep}): 
\begin{equation}
\begin{aligned} &g(\mathbf{Y}) = -a\log |\mathbf{Y}| - (c-a) \log
|\mathbf{I}_d - \mathbf{Y}| + \log |\mathbf{I}_d-\mathbf{X}\mathbf{Y}|, \\
&h(\mathbf{Y}) = B_d^{-1}(a,c-a) |\mathbf{Y}|^{-\frac{d+1}{2}}
|\mathbf{I}_d-\mathbf{Y}|^{-\frac{d+1}{2}}. \end{aligned}  \label{gh}
\end{equation}%
Using (\ref{laplace2}) and (\ref{gh}), the Laplace approximation to $%
~_{2}F_{1}(a,b;c;\mathbf{X})$ is given by \cite{Laplace_approx} 
\begin{equation}
\begin{aligned} &~_2\tilde{F}_1(a,b;c;\mathbf{X}) = \frac{2^{\frac{d}{2}}
\pi^{\frac{d(d+1)}{4}}}{B_d(a,c-a)} J_{2,1}^{-\frac{1}{2}} \\ &
\hspace{1cm}\times\prod_{i=1}^d\{\hat{y}_i^a
(1-\hat{y}_i)^{c-a}(1-x_i\hat{y}_i)^{-b}\}, \end{aligned}
\label{laplace_approx}
\end{equation}%
where $\hat{y}_{i}$ is defined as 
\begin{equation}
\hat{y}_{i}=\frac{2a}{\sqrt{\tau ^{2}-4ax_{i}(c-b)}-\tau },
\end{equation}%
with $\tau =x_{i}(b-a)-c$, and 
\begin{equation}
J_{2,1}=\prod_{i=1}^{d}\prod_{j=i}^{d}\{a(1-\hat{y}_{i})(1-\hat{y}_{j})+(c-a)%
\hat{y}_{i}\hat{y}_{j}-bL_{i}L_{j}\},
\end{equation}%
with 
\begin{equation}
L_{i}=\frac{x_{i}\hat{y}_{i}(1-\hat{y}_{i})}{1-x_{i}\hat{y}_{i}}.
\end{equation}%
The value of $_{2}F_{1}(a,b;c;\mathbf{X})$ at $\mathbf{X}=\mathbf{0}$ is 1,
that is, $~_{2}F_{1}(a,b;c;\mathbf{0})=1$. As a result, the Laplace
approximation in (\ref{laplace_approx}) is calibrated at $\mathbf{X}=\mathbf{%
0}$ to give the calibrated Laplace approximation \cite{Laplace_approx}: 
\begin{equation}
\begin{aligned} & ~_2\hat{F}_1(a,b;c;\mathbf{X}) =
\frac{~_2\tilde{F}_1(a,b;c;\X) }{~_2\tilde{F}_1(a,b;c;\mathbf{0}) } =
c^{cd-\frac{d(d+1)}{4}} R_{2,1}^{-\frac{1}{2}} \\ &\hspace{1cm} \times
\prod_{i=1}^d \left\lbrace\left(\frac{\hat{y}_i}{a}\right)^a
\left(\frac{1-\hat{y}_i}{c-a}\right)^{c-a}
(1-x_i\hat{y}_i)^{-b}\right\rbrace, \end{aligned}  \label{lablace_calib}
\end{equation}%
where 
\begin{equation}
\begin{aligned} &R_{2,1} = \prod_{i=1}^d \prod_{j=i}^d \left\lbrace
\frac{\hat{y}_i \hat{y}_j}{a} + \frac{(1-\hat{y}_i)(1-\hat{y}_j)}{c-a}
\right. \\ & \hspace{2cm} \left. - \frac{b x_ix_j \hat{y}_i \hat{y}_j
(1-\hat{y}_i)(1-\hat{y}_j)}{(1-x_i\hat{y}_i)(1-x_j\hat{y}_j)a(c-a)}\right%
\rbrace. \end{aligned}
\end{equation}

According to \cite{Laplace_approx}, the relative error of the approximation remains uniformly bounded:
\begin{equation}
\sup |\log ~_2\hat{F}_1(a,b;c;\X) - \log ~_2F_1(a,b;c;\X)| < \infty,
\end{equation}
supremum being over $c\geq c_0 > \frac{d-1}{2}$, $a,b\in R$, and $ 0_d\leq \X <(1-\epsilon)I_d$  for any $\epsilon \in (0,1)$. Authors provide in \cite{Laplace_approx} some numerical examples to show how well this approximation works. We also follow the same way and show two plots in Fig. \ref{fig_laplace}, which demonstrate a very good numerical accuracy for several different setups. As mentioned, the hypergeometric function $~_2F_1(a,b;c;\X)$ of matrix argument is only a function of the eigenvalues of $\X$. So, we fix $\X=\tau I_d$ and draw the exact and approximate values of $~_2F_1(a,b;c;\tau I_d)$ versus $\tau$ (note $0<\tau<1$ for convergence as mentioned in the definition of $~_2F_1(a,b;c;\X)$ in (\ref{Gauss})) in Fig. \ref{fig_laplace}a for $d=5$, $a=3$, $b=4$, and $c=6$. Fig. \ref{fig_laplace}b shows the exact and approximate values of $~_2F_1(a,b;c;\tau I_d)$ versus $c$ for $d=10$, $a=30$, $b=50$, and $\tau=0.01$. The authors stated in \cite{Laplace_approx} that when the integral representation is not valid, that is, when $c-a < \frac{d-1}{2}$, this Laplace approximation still gives good accuracy. We also see that approximation in Fig. \ref{fig_laplace}b is accurate for all range of $c$, even though the integral representation is not valid for $c<a+\frac{d-1}{2} = 34.5$. We also note that this approximation is more accurate in the smaller function values.

\begin{figure}[t!]
\centering
\begin{subfigure}[t]{0.23\textwidth}
        \centering
        \includegraphics[width=\textwidth]{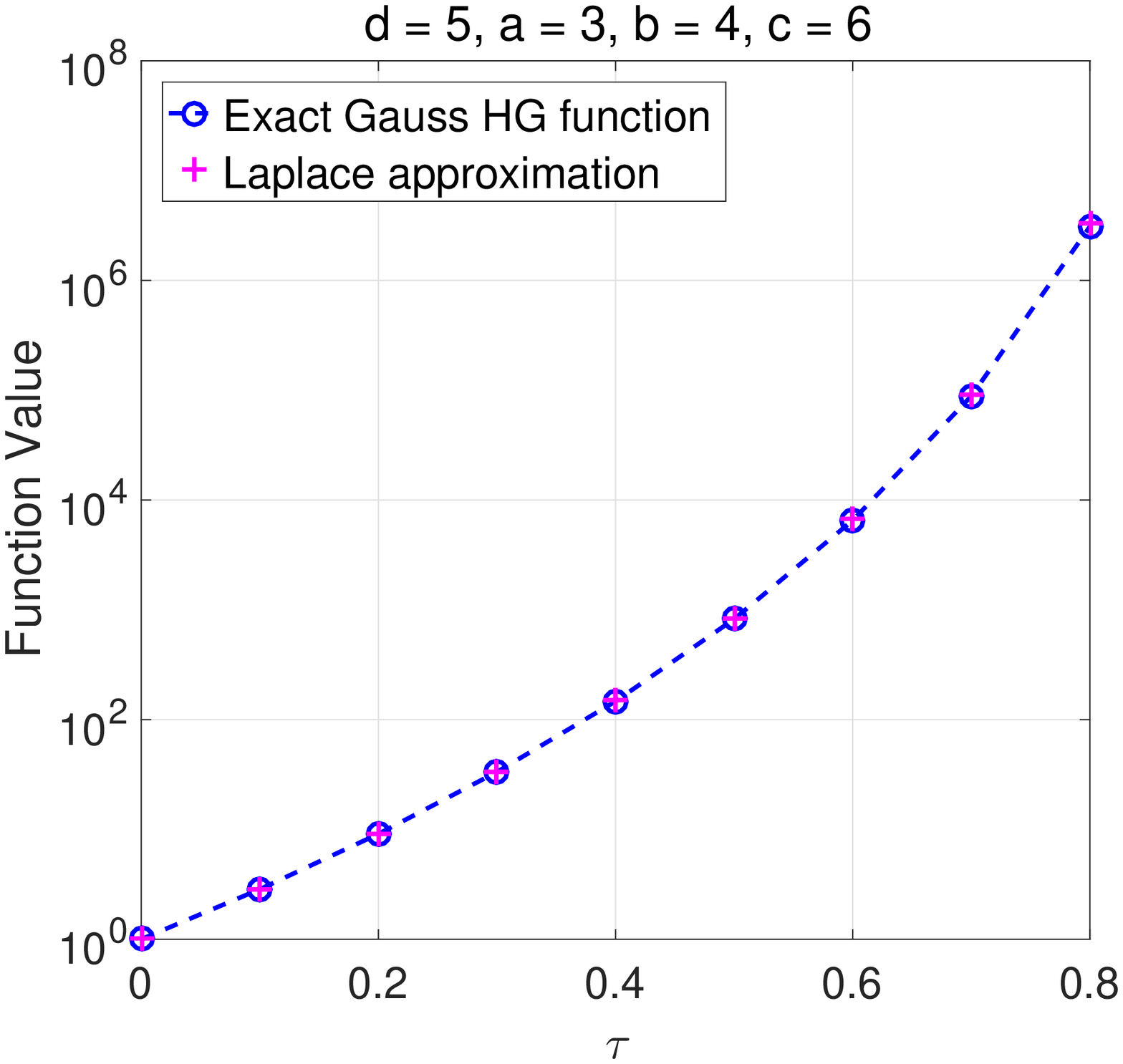}
        \caption{}
    \end{subfigure}
~ 
\begin{subfigure}[t]{0.23\textwidth}
        \centering
        \includegraphics[width=\textwidth]{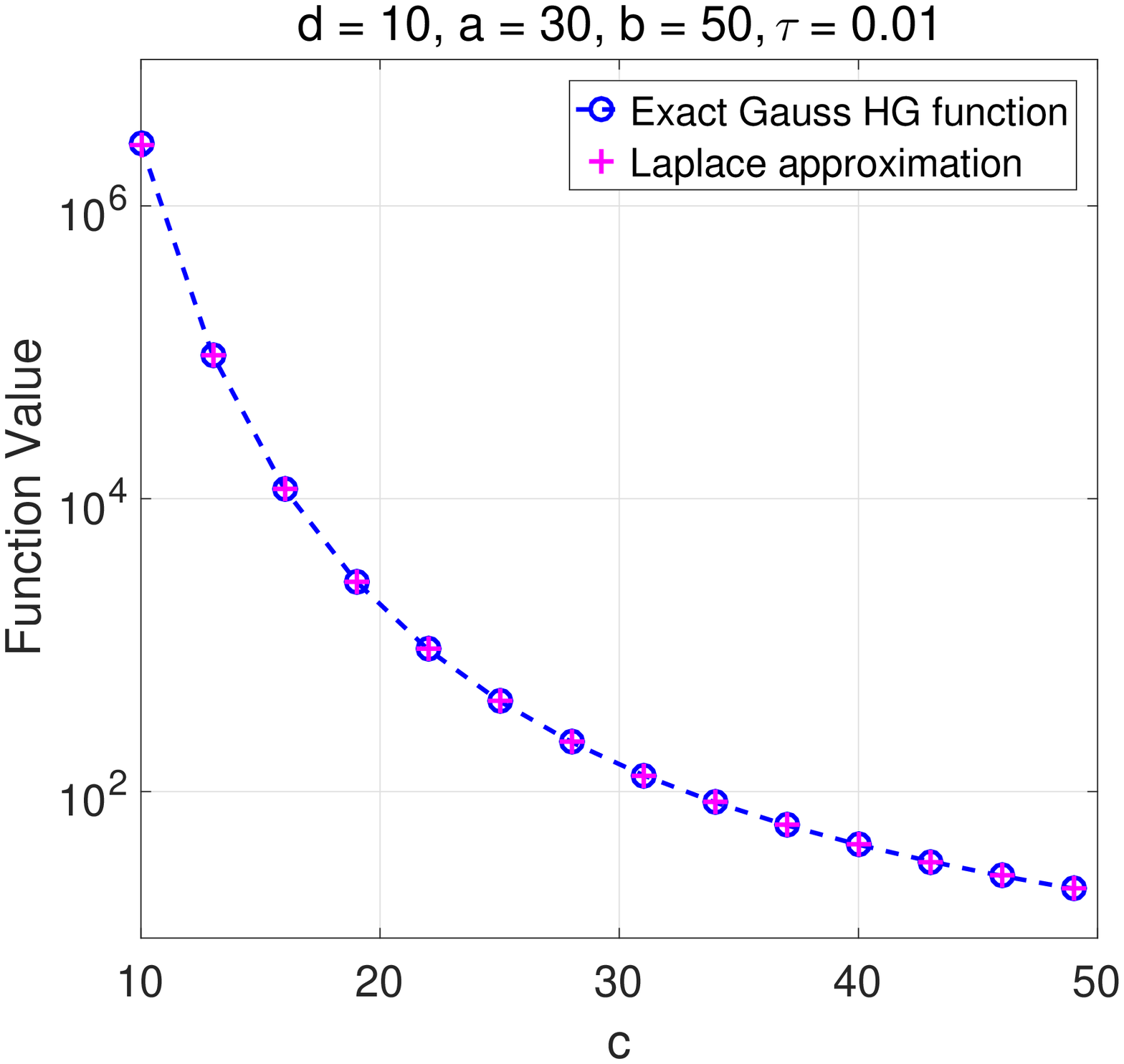}
        \caption{}
    \end{subfigure}
\caption{{\protect\footnotesize Exact values of function $~_2F_1(a,b;c;\tau I_d)$ and its corresponding Laplace approximation $~_2\hat{F}_1(a,b;c;\tau I_d)$ versus: (a) $\tau$, for $d=5$, $a=3$, $b=4$, and $c=6$, (b) $c$, for $d=10$, $a=30$, $b=50$, and $\tau=0.01$.}}
\label{fig_laplace}
\end{figure}

\section*{Acknowledgment}

This work was funded in part by Award CCF-1553281 from the National Science
Foundation.

\vspace{-.1cm} 
\bibliographystyle{IEEEtran}
\bibliography{IEEEabrv,ref}

\begin{thebibliography}{10}
\providecommand{\url}[1]{#1}
\csname url@samestyle\endcsname
\providecommand{\newblock}{\relax}
\providecommand{\bibinfo}[2]{#2}
\providecommand{\BIBentrySTDinterwordspacing}{\spaceskip=0pt\relax}
\providecommand{\BIBentryALTinterwordstretchfactor}{4}
\providecommand{\BIBentryALTinterwordspacing}{\spaceskip=\fontdimen2\font plus
\BIBentryALTinterwordstretchfactor\fontdimen3\font minus
  \fontdimen4\font\relax}
\providecommand{\BIBforeignlanguage}[2]{{%
\expandafter\ifx\csname l@#1\endcsname\relax
\typeout{** WARNING: IEEEtran.bst: No hyphenation pattern has been}%
\typeout{** loaded for the language `#1'. Using the pattern for}%
\typeout{** the default language instead.}%
\else
\language=\csname l@#1\endcsname
\fi
#2}}
\providecommand{\BIBdecl}{\relax}
\BIBdecl

\bibitem{survey2010}
S.~J. Pan and Q.~Yang, ``A survey on transfer learning,'' \emph{IEEE
  Transactions on knowledge and data engineering}, vol.~22, no.~10, pp.
  1345--1359, 2010.

\bibitem{survey2017deep}
H.~Venkateswara, S.~Chakraborty, and S.~Panchanathan, ``Deep-learning systems
  for domain adaptation in computer vision: Learning transferable feature
  representations,'' \emph{IEEE Signal Processing Magazine}, vol.~34, no.~6,
  pp. 117--129, 2017.

\bibitem{survey2015}
V.~M. Patel, R.~Gopalan, R.~Li, and R.~Chellappa, ``Visual domain adaptation: A
  survey of recent advances,'' \emph{IEEE signal processing magazine}, vol.~32,
  no.~3, pp. 53--69, 2015.

\bibitem{survey2016}
K.~Weiss, T.~M. Khoshgoftaar, and D.~Wang, ``A survey of transfer learning,''
  \emph{Journal of Big Data}, vol.~3, no.~1, p.~9, 2016.

\bibitem{survey2017}
G.~Csurka, ``Domain adaptation for visual applications: A comprehensive
  survey,'' \emph{arXiv preprint arXiv:1702.05374}, 2017.

\bibitem{zou2015transfer}
N.~Zou, Y.~Zhu, J.~Zhu, M.~Baydogan, W.~Wang, and J.~Li, ``A transfer learning
  approach for predictive modeling of degenerate biological systems,''
  \emph{Technometrics}, vol.~57, no.~3, pp. 362--373, 2015.

\bibitem{ganchev2011transfer}
P.~Ganchev, D.~Malehorn, W.~L. Bigbee, and V.~Gopalakrishnan, ``Transfer
  learning of classification rules for biomarker discovery and verification
  from molecular profiling studies,'' \emph{Journal of biomedical informatics},
  vol.~44, pp. S17--S23, 2011.

\bibitem{gong2012geodesic}
B.~Gong, Y.~Shi, F.~Sha, and K.~Grauman, ``Geodesic flow kernel for
  unsupervised domain adaptation,'' in \emph{Computer Vision and Pattern
  Recognition (CVPR), 2012 IEEE Conference on}.\hskip 1em plus 0.5em minus
  0.4em\relax IEEE, 2012, pp. 2066--2073.

\bibitem{HFA2012}
L.~Duan, D.~Xu, and I.~Tsang, ``Learning with augmented features for
  heterogeneous domain adaptation,'' \emph{ICML}, 2012.

\bibitem{hoffman2013}
J.~Hoffman, E.~Rodner, T.~Darrell, J.~Donahue, and K.~Saenko, ``Efficient
  learning of domain-invariant image representations,'' in \emph{International
  Conference on Learning Representations (ICLR)}, 2013.

\bibitem{hoffman2014}
J.~Hoffman, E.~Rodner, J.~Donahue, B.~Kulis, and K.~Saenko, ``Asymmetric and
  category invariant feature transformations for domain adaptation,''
  \emph{International journal of computer vision}, vol. 109, no. 1-2, pp.
  28--41, 2014.

\bibitem{CDLS2016}
Y.-H. Hubert~Tsai, Y.-R. Yeh, and Y.-C. Frank~Wang, ``Learning cross-domain
  landmarks for heterogeneous domain adaptation,'' in \emph{Proceedings of the
  IEEE Conference on Computer Vision and Pattern Recognition}, 2016, pp.
  5081--5090.

\bibitem{MMD}
K.~M. Borgwardt, A.~Gretton, M.~J. Rasch, H.-P. Kriegel, B.~Sch{\"o}lkopf, and
  A.~J. Smola, ``Integrating structured biological data by kernel maximum mean
  discrepancy,'' \emph{Bioinformatics}, vol.~22, no.~14, pp. e49--e57, 2006.

\bibitem{dai2007boosting}
W.~Dai, Q.~Yang, G.-R. Xue, and Y.~Yu, ``Boosting for transfer learning,'' in
  \emph{Proceedings of the 24th international conference on Machine
  learning}.\hskip 1em plus 0.5em minus 0.4em\relax ACM, 2007, pp. 193--200.

\bibitem{duan2009domain}
L.~Duan, I.~W. Tsang, D.~Xu, and S.~J. Maybank, ``Domain transfer svm for video
  concept detection,'' in \emph{Computer Vision and Pattern Recognition, 2009.
  CVPR 2009. IEEE Conference on}.\hskip 1em plus 0.5em minus 0.4em\relax IEEE,
  2009, pp. 1375--1381.

\bibitem{bruzzone2010domain}
L.~Bruzzone and M.~Marconcini, ``Domain adaptation problems: A dasvm
  classification technique and a circular validation strategy,'' \emph{IEEE
  transactions on pattern analysis and machine intelligence}, vol.~32, no.~5,
  pp. 770--787, 2010.

\bibitem{ILS2017}
S.~Herath, M.~Harandi, and F.~Porikli, ``Learning an invariant hilbert space
  for domain adaptation,'' in \emph{2017 IEEE Conference on Computer Vision and
  Pattern Recognition (CVPR)}, July 2017, pp. 3956--3965.

\bibitem{OT}
N.~Courty, R.~Flamary, D.~Tuia, and A.~Rakotomamonjy, ``Optimal transport for
  domain adaptation,'' \emph{IEEE Transactions on Pattern Analysis and Machine
  Intelligence}, vol.~39, no.~9, pp. 1853--1865, Sept 2017.

\bibitem{long2015learning}
M.~Long, Y.~Cao, J.~Wang, and M.~Jordan, ``Learning transferable features with
  deep adaptation networks,'' in \emph{International Conference on Machine
  Learning}, 2015, pp. 97--105.

\bibitem{long2016unsupervised}
M.~Long, H.~Zhu, J.~Wang, and M.~I. Jordan, ``Unsupervised domain adaptation
  with residual transfer networks,'' in \emph{Advances in Neural Information
  Processing Systems}, 2016, pp. 136--144.

\bibitem{ganin2016domain}
Y.~Ganin, E.~Ustinova, H.~Ajakan, P.~Germain, H.~Larochelle, F.~Laviolette,
  M.~Marchand, and V.~Lempitsky, ``Domain-adversarial training of neural
  networks,'' \emph{Journal of Machine Learning Research}, vol.~17, no.~59, pp.
  1--35, 2016.

\bibitem{liu2016coupled}
M.-Y. Liu and O.~Tuzel, ``Coupled generative adversarial networks,'' in
  \emph{Advances in neural information processing systems}, 2016, pp. 469--477.

\bibitem{Lori1}
L.~A. Dalton and E.~R. Dougherty, ``Optimal classifiers with minimum expected
  error within a {B}ayesian framework—{P}art {I}: Discrete and gaussian
  models,'' \emph{Pattern Recognition}, vol.~46, no.~5, pp. 1288 -- 1300, 2013.

\bibitem{Lori2}
------, ``Optimal classifiers with minimum expected error within a {B}ayesian
  framework — {P}art {II}: Properties and performance analysis,''
  \emph{Pattern Recognition}, vol.~46, no.~5, pp. 1301 -- 1314, 2013.

\bibitem{muirhead}
R.~J. Muirhead, \emph{Aspects of multivariate statistical theory}.\hskip 1em
  plus 0.5em minus 0.4em\relax John Wiley \& Sons, 2009.

\bibitem{joint_wishart}
K.~Halvorsen, V.~Ayala, and E.~Fierro, ``On the marginal distribution of the
  diagonal blocks in a blocked {W}ishart random matrix,'' \emph{International
  Journal of Analysis}, vol. 2016, pp. 1--5, 2016.

\bibitem{nagar2017properties}
D.~K. Nagar and J.~C. Mosquera-Ben{\i}tez, ``Properties of matrix variate
  hypergeometric function distribution,'' \emph{Applied Mathematical Sciences},
  vol.~11, no.~14, pp. 677--692, 2017.

\bibitem{constantine1963}
A.~G. Constantine, ``Some non-central distribution problems in multivariate
  analysis,'' \emph{Ann. Math. Statist.}, vol.~34, no.~4, pp. 1270--1285, 12
  1963.

\bibitem{Laplace_approx}
R.~W. Butler and A.~T.~A. Wood, ``Laplace approximations for hypergeometric
  functions with matrix argument,'' \emph{The Annals of Statistics}, vol.~30,
  no.~4, pp. 1155--1177, 2002.

\bibitem{dalton2015optimal}
L.~A. Dalton and M.~R. Yousefi, ``On optimal {B}ayesian classification and risk
  estimation under multiple classes,'' \emph{EURASIP Journal on Bioinformatics
  and Systems Biology}, vol. 2015, no.~1, p.~8, 2015.

\bibitem{Lori-MMSE}
L.~A. Dalton and E.~R. Dougherty, ``{B}ayesian minimum mean-square error
  estimation for classification error-{P}art {I}: Definition and the {B}ayesian
  {MMSE} error estimator for discrete classification,'' \emph{IEEE Transactions
  on Signal Processing}, vol.~59, no.~1, pp. 115--129, Jan 2011.

\bibitem{office}
K.~Saenko, B.~Kulis, M.~Fritz, and T.~Darrell, ``Adapting visual category
  models to new domains,'' in \emph{Proceedings of the 11th European Conference
  on Computer Vision: Part IV}, ser. ECCV'10.\hskip 1em plus 0.5em minus
  0.4em\relax Berlin, Heidelberg: Springer-Verlag, 2010, pp. 213--226.

\bibitem{caltech}
G.~Griffin, A.~Holub, and P.~Perona, ``Caltech-256 object category dataset,''
  \emph{Technical Report 7694, California Institute of Technology}, 2007.

\bibitem{surf}
H.~Bay, T.~Tuytelaars, and L.~Van~Gool, ``{SURF}: Speeded up robust features,''
  \emph{Computer vision--ECCV 2006}, pp. 404--417, 2006.

\bibitem{knight2014mcmc}
J.~M. Knight, I.~Ivanov, and E.~R. Dougherty, ``{MCMC} implementation of the
  optimal {B}ayesian classifier for non-{G}aussian models: {M}odel-based
  {RNA}-seq classification,'' \emph{BMC bioinformatics}, vol.~15, no.~1, p.
  401, 2014.

\bibitem{knight2015detecting}
J.~M. Knight, I.~Ivanov, K.~Triff, R.~S. Chapkin, and E.~R. Dougherty,
  ``Detecting multivariate gene interactions in {RNA}-seq data using optimal
  {B}ayesian classification,'' \emph{IEEE/ACM transactions on computational
  biology and bioinformatics}, 2015.

\bibitem{Alireza_TCBB}
A.~Karbalayghareh, U.~Braga-Neto, J.~Hua, and E.~R. Dougherty, ``Classification
  of state trajectories in gene regulatory networks,'' \emph{IEEE/ACM
  Transactions on Computational Biology and Bioinformatics}, vol.~15, no.~1,
  pp. 68--82, Jan 2018.

\bibitem{Alireza_TCBB2}
A.~Karbalayghareh, U.~Braga-Neto, and E.~R. Dougherty, ``Classification of
  single-cell gene expression trajectories from incomplete and noisy data,''
  \emph{IEEE/ACM Transactions on Computational Biology and Bioinformatics},
  vol.~PP, no.~99, pp. 1--1, 2018.

\bibitem{Alireza_ICASSP}
------, ``Classification of gaussian trajectories with missing data in boolean
  gene regulatory networks,'' in \emph{2017 IEEE International Conference on
  Acoustics, Speech and Signal Processing (ICASSP)}, March 2017, pp.
  1078--1082.

\bibitem{Alireza_BMC}
------, ``Intrinsically bayesian robust classifier for single-cell gene
  expression trajectories in gene regulatory networks,'' \emph{BMC Systems
  Biology}, vol.~12, no.~3, p.~23, Mar 2018.

\bibitem{dalton2011application}
L.~A. Dalton and E.~R. Dougherty, ``Application of the {B}ayesian {MMSE}
  estimator for classification error to gene expression microarray data,''
  \emph{Bioinformatics}, vol.~27, no.~13, pp. 1822--1831, 2011.

\bibitem{Esfahani_TCBB_1}
M.~S. Esfahani and E.~R. Dougherty, ``Incorporation of biological pathway
  knowledge in the construction of priors for optimal {B}ayesian
  classification,'' \emph{IEEE/ACM Transactions on Computational Biology and
  Bioinformatics}, vol.~11, no.~1, pp. 202--218, Jan 2014.

\bibitem{Esfahani_TCBB_2}
------, ``An optimization-based framework for the transformation of incomplete
  biological knowledge into a probabilistic structure and its application to
  the utilization of gene/protein signaling pathways in discrete phenotype
  classification,'' \emph{IEEE/ACM Transactions on Computational Biology and
  Bioinformatics}, vol.~12, no.~6, pp. 1304--1321, Nov 2015.

\bibitem{Shahin_BMC}
S.~Boluki, M.~S. Esfahani, X.~Qian, and E.~R. Dougherty, ``Incorporating
  biological prior knowledge for bayesian learning via maximal knowledge-driven
  information priors,'' \emph{BMC Bioinformatics}, vol.~18, no.~14, p. 552, Dec
  2017.

\bibitem{Lori-IBRF}
L.~A. Dalton and E.~R. Dougherty, ``Intrinsically optimal {B}ayesian robust
  filtering,'' \emph{IEEE Transactions on Signal Processing}, vol.~62, no.~3,
  pp. 657--670, Feb 2014.

\bibitem{Qian-OBF}
X.~Qian and E.~R. Dougherty, ``{B}ayesian regression with network prior:
  Optimal {B}ayesian filtering perspective,'' \emph{IEEE Transactions on Signal
  Processing}, vol.~64, no.~23, pp. 6243--6253, Dec 2016.

\bibitem{Roozbeh_IBR_Kalman}
R.~Dehghannasiri, M.~S. Esfahani, and E.~R. Dougherty, ``Intrinsically
  {B}ayesian robust {K}alman filter: An innovation process approach,''
  \emph{IEEE Transactions on Signal Processing}, vol.~65, no.~10, pp.
  2531--2546, May 2017.

\bibitem{Appell}
D.~K. Nagar and S.~Nadarajah, ``Appell's hypergeometric functions of matrix
  arguments,'' \emph{Integral Transforms and Special Functions}, vol.~28,
  no.~2, pp. 91--112, 2017.

\bibitem{gupta2016properties}
A.~K. Gupta, D.~K. Nagar, and L.~E. S{\'a}nchez, ``Properties of matrix variate
  confluent hypergeometric function distribution,'' \emph{Journal of
  Probability and Statistics}, vol. 2016, 2016.

\end{thebibliography}

\end{document}